\documentclass[review]{elsarticle}

\usepackage{graphicx}

\usepackage{color}
\usepackage{comment}
\usepackage{amsmath,amssymb,amsthm}
\usepackage{lineno,hyperref}
\usepackage[caption=false]{subfig}
\usepackage{epstopdf}
\usepackage{layouts}
\usepackage{todonotes}
\usepackage{tabularx}
\usepackage[export]{adjustbox}
\graphicspath{{./}{./img/}{./img/dome_vs_unw_persp/}} 
\modulolinenumbers[5]

\newtheorem{theorem}{Theorem}
\newtheorem{lemma}{Lemma}

\makeatletter
\def\input@path{{./}}
\makeatother

\renewcommand{\d}[1]{\mbox{\boldmath$#1$}}

\newcommand{\trans}[0]{^{\sf T}}

\newcommand{\q}[1]{{{\bf #1}}}

 \newcommand{\m}[1]{{\mbox{{\fontencoding{T1}\sffamily{\itshape #1}}}}}

\newcommand{\mq}[1]{{\q #1}}

\renewcommand{\vec}[0]{\mbox{vec}}

\newcommand{\mC}{\mathchoice {\setbox0=\hbox{$\displaystyle\rm
C$}\hbox{\hbox to0pt{\kern0.4\wd0\vrule height0.9\ht0\hss}\box0}}
{\setbox0=\hbox{$\textstyle\rm C$}\hbox{\hbox
to0pt{\kern0.4\wd0\vrule height0.9\ht0\hss}\box0}}
{\setbox0=\hbox{$\scriptstyle\rm C$}\hbox{\hbox
to0pt{\kern0.4\wd0\vrule height0.9\ht0\hss}\box0}}
{\setbox0=\hbox{$\scriptscriptstyle\rm C$}\hbox{\hbox
to0pt{\kern0.4\wd0\vrule height0.9\ht0\hss}\box0}}}

\newcommand{\mG}{\mathchoice {\setbox0=\hbox{$\displaystyle\rm
G$}\hbox{\hbox to0pt{\kern0.4\wd0\vrule height0.9\ht0\hss}\box0}}
{\setbox0=\hbox{$\textstyle\rm G$}\hbox{\hbox
to0pt{\kern0.4\wd0\vrule height0.9\ht0\hss}\box0}}
{\setbox0=\hbox{$\scriptstyle\rm G$}\hbox{\hbox
to0pt{\kern0.4\wd0\vrule height0.9\ht0\hss}\box0}}
{\setbox0=\hbox{$\scriptscriptstyle\rm G$}\hbox{\hbox
to0pt{\kern0.4\wd0\vrule height0.9\ht0\hss}\box0}}}

\newcommand{\mJ}{\mathchoice {\setbox0=\hbox{$\displaystyle\rm
J$}\hbox{\hbox to0pt{\kern0.4\wd0\vrule height0.9\ht0\hss}\box0}}
{\setbox0=\hbox{$\textstyle\rm J$}\hbox{\hbox
to0pt{\kern0.4\wd0\vrule height0.9\ht0\hss}\box0}}
{\setbox0=\hbox{$\scriptstyle\rm J$}\hbox{\hbox
to0pt{\kern0.4\wd0\vrule height0.9\ht0\hss}\box0}}
{\setbox0=\hbox{$\scriptscriptstyle\rm J$}\hbox{\hbox
to0pt{\kern0.4\wd0\vrule height0.9\ht0\hss}\box0}}}

\newcommand{\mO}{\mathchoice {\setbox0=\hbox{$\displaystyle\rm
O$}\hbox{\hbox to0pt{\kern0.4\wd0\vrule height0.9\ht0\hss}\box0}}
{\setbox0=\hbox{$\textstyle\rm O$}\hbox{\hbox
to0pt{\kern0.4\wd0\vrule height0.9\ht0\hss}\box0}}
{\setbox0=\hbox{$\scriptstyle\rm O$}\hbox{\hbox
to0pt{\kern0.4\wd0\vrule height0.9\ht0\hss}\box0}}
{\setbox0=\hbox{$\scriptscriptstyle\rm O$}\hbox{\hbox
to0pt{\kern0.4\wd0\vrule height0.9\ht0\hss}\box0}}}

\newcommand{\mQ}{\mathchoice {\setbox0=\hbox{$\displaystyle\rm
Q$}\hbox{\raise 0.15\ht0\hbox to0pt{\kern0.4\wd0\vrule
height0.8\ht0\hss}\box0}}{\setbox0=\hbox{$\textstyle\rm Q$}\hbox{\raise
0.15\ht0\hbox to0pt{\kern0.4\wd0\vrule height0.8\ht0\hss}\box0}}
{\setbox0=\hbox{$\scriptstyle\rm Q$}\hbox{\raise 0.15\ht0\hbox
to0pt{\kern0.4\wd0\vrule height0.7\ht0\hss}\box0}}{\setbox0=
\hbox{$\scriptscriptstyle\rm Q$}\hbox{\raise 0.15\ht0\hbox
to0pt{\kern0.4\wd0\vrule height0.7\ht0\hss}\box0}}}

\newcommand{\mS}{\mathchoice
{\setbox0=\hbox{$\displaystyle     \rm S$}\hbox{\raise0.5\ht0\hbox
to0pt{\kern0.35\wd0\vrule height0.45\ht0\hss}\hbox
to0pt{\kern0.55\wd0\vrule height0.5\ht0\hss}\box0}}
{\setbox0=\hbox{$\textstyle        \rm S$}\hbox{\raise0.5\ht0\hbox
to0pt{\kern0.35\wd0\vrule height0.45\ht0\hss}\hbox
to0pt{\kern0.55\wd0\vrule height0.5\ht0\hss}\box0}}
{\setbox0=\hbox{$\scriptstyle      \rm S$}\hbox{\raise0.5\ht0\hbox
to0pt{\kern0.35\wd0\vrule height0.45\ht0\hss}\raise0.05\ht0\hbox
to0pt{\kern0.5\wd0\vrule height0.45\ht0\hss}\box0}}
{\setbox0=\hbox{$\scriptscriptstyle\rm S$}\hbox{\raise0.5\ht0\hbox
to0pt{\kern0.4\wd0\vrule height0.45\ht0\hss}\raise0.05\ht0\hbox
to0pt{\kern0.55\wd0\vrule height0.45\ht0\hss}\box0}}}
\newcommand{\mT}{\mathchoice {\setbox0=\hbox{$\displaystyle\rm
T$}\hbox{\hbox to0pt{\kern0.3\wd0\vrule height0.9\ht0\hss}\box0}}
{\setbox0=\hbox{$\textstyle\rm T$}\hbox{\hbox
to0pt{\kern0.3\wd0\vrule height0.9\ht0\hss}\box0}}
{\setbox0=\hbox{$\scriptstyle\rm T$}\hbox{\hbox
to0pt{\kern0.3\wd0\vrule height0.9\ht0\hss}\box0}}
{\setbox0=\hbox{$\scriptscriptstyle\rm T$}\hbox{\hbox
to0pt{\kern0.3\wd0\vrule height0.9\ht0\hss}\box0}}}
\newcommand{\mU}{\mathchoice {\setbox0=\hbox{$\displaystyle\rm
U$}\hbox{\hbox to0pt{\kern0.4\wd0\vrule height0.9\ht0\hss}\box0}}
{\setbox0=\hbox{$\textstyle\rm U$}\hbox{\hbox
to0pt{\kern0.4\wd0\vrule height0.9\ht0\hss}\box0}}
{\setbox0=\hbox{$\scriptstyle\rm U$}\hbox{\hbox
to0pt{\kern0.4\wd0\vrule height0.9\ht0\hss}\box0}}
{\setbox0=\hbox{$\scriptscriptstyle\rm U$}\hbox{\hbox
to0pt{\kern0.4\wd0\vrule height0.9\ht0\hss}\box0}}}

\usepackage[]{placeins}
\makeatletter
\AtBeginDocument{%
	\expandafter\renewcommand\expandafter\subsection\expandafter
	{\expandafter\@fb@secFB\subsection}%
	\newcommand\@fb@secFB{\FloatBarrier
		\gdef\@fb@afterHHook{\@fb@topbarrier \gdef\@fb@afterHHook{}}}%
	\g@addto@macro\@afterheading{\@fb@afterHHook}%
	\gdef\@fb@afterHHook{}%
}
\makeatother

\journal{ArXiv submission}

\begin{document}

\begin{frontmatter}

\title{Refractive Geometry for Underwater Domes}


\author[mymainaddress]{Mengkun She\corref{mycorrespondingauthor}}
\cortext[mycorrespondingauthor]{Corresponding author}
\ead{mshe@geomar.de}

\author[mymainaddress]{David Nakath}

\author[mymainaddress]{Yifan Song}

\author[mymainaddress]{Kevin K{\"o}ser}


\address[mymainaddress]{Oceanic Machine Vision,\\ GEOMAR Helmholtz Centre for Ocean Research Kiel, Kiel, Germany\\
	Tel.: ++49 431 600 2595\\}

\begin{abstract}
Underwater cameras are typically placed behind glass windows to protect them from the water. Spherical glass, a dome port, is well suited for high water pressures at great depth, allows for a large field of view, and avoids refraction if a pinhole camera is positioned exactly at the sphere's center. Adjusting a real lens perfectly to the dome center is a challenging task, both in terms of how to actually guide the centering process (e.g. visual servoing) and how to measure the alignment quality, but also, how to mechanically perform the alignment. Consequently, such systems are prone to being decentered by some offset, leading to challenging refraction patterns at the sphere that invalidate the pinhole camera model.	
We show that the overall camera system becomes an axial camera, even for thick domes as used for deep sea exploration and provide a non-iterative way to compute the center of refraction without requiring knowledge of exact air, glass or water properties. We also analyze the refractive geometry at the sphere, looking at effects such as forward- vs. backward decentering, iso-refraction curves and obtain a 6th-degree polynomial equation for forward projection of 3D points in thin domes.
We then propose a pure underwater calibration procedure to estimate the decentering from multiple images. This estimate can either be used during adjustment to guide the mechanical position of the lens, or can be considered in photogrammetric underwater applications.
\end{abstract}


\end{frontmatter}

\nolinenumbers

\section{Introduction}
More than two-thirds of Earth's surface is covered by water -- or more specifically -- by the oceans.
Underwater imaging, vision, photogrammetry and robotic applications include recovering sunken cultural heritage, offshore installations, habitat mapping, resource estimation, deposited munition monitoring, optical quantification of processes in the ocean, human impact on ecosystems and last but not least exploration of the last uncharted terrain on Earth: the deep sea.

Underwater cameras are usually placed inside pressure housings and observe the outer world through some kind of \textit{window}.
Incident light rays travel through water, glass, and air before being sensed by the camera and each time they traverse media with different optical densities their direction might be changed. In particular for flat glass interfaces this requires more complex geometric reasoning \cite{Treibitz_12_flatrefractivegeometry, Agrawal_2012-UnwCalib,maas2015accuracy,jordt_2016_refractive}, such that alternative configurations are being explored.
\begin{figure*}[t]
	\begin{center}
		\includegraphics[height=2.0cm]{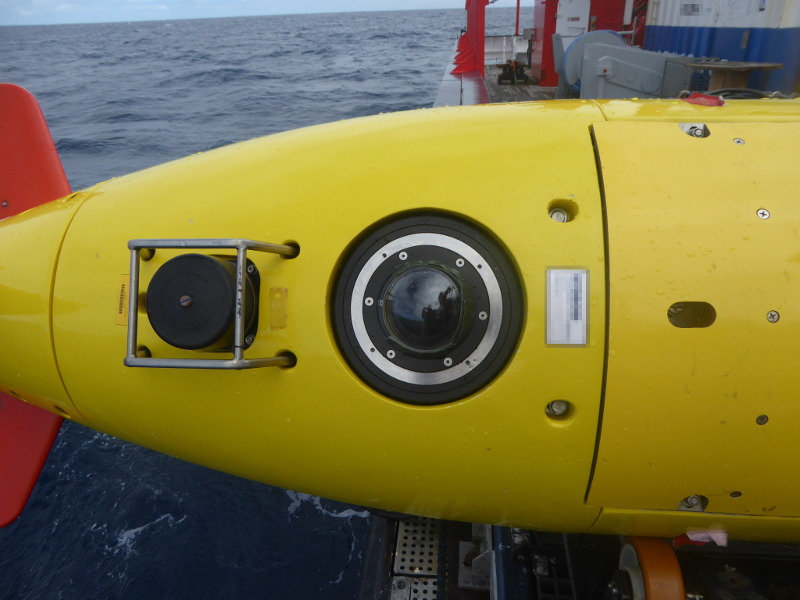}
		\includegraphics[height=2.0cm]{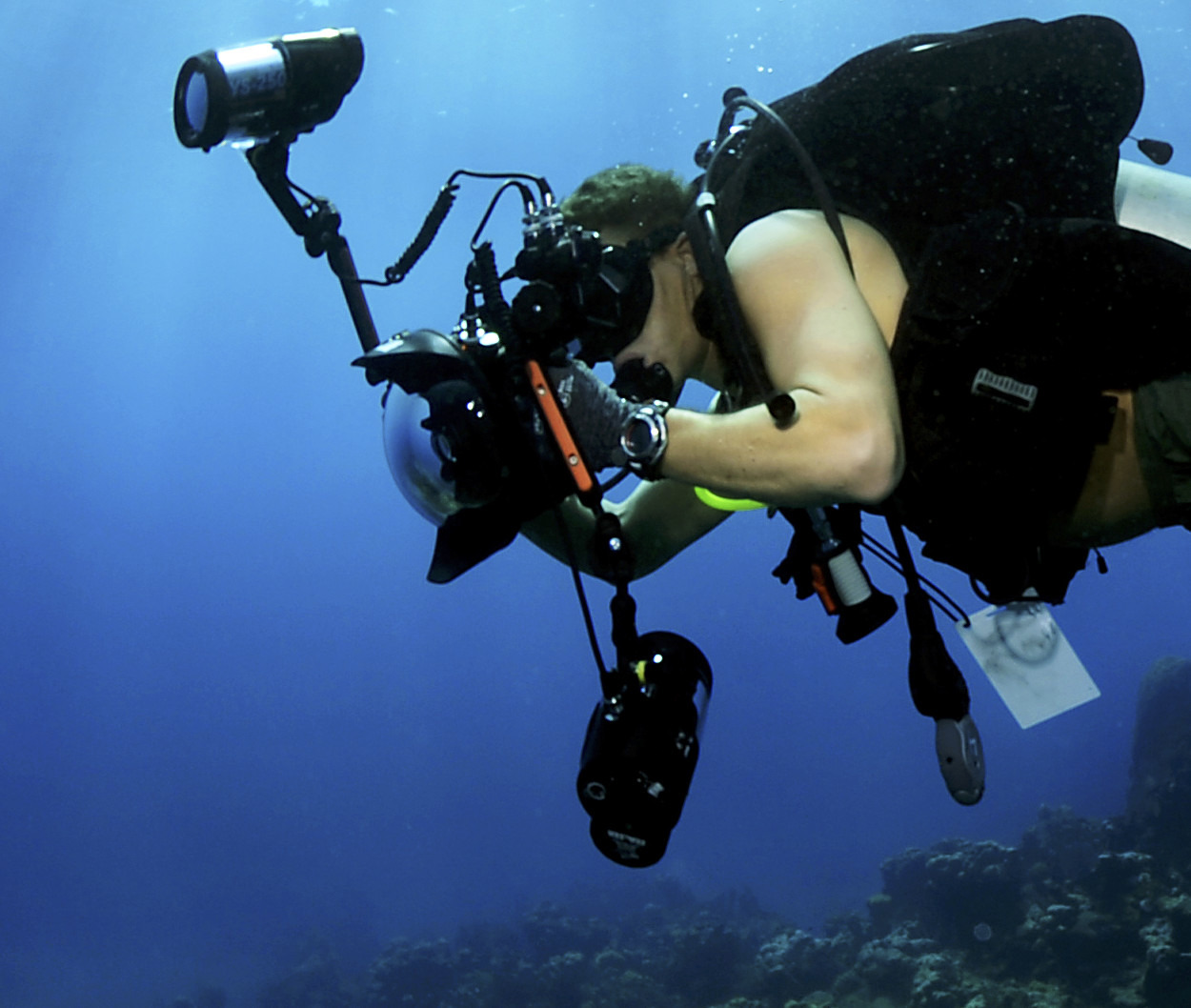}
		\includegraphics[height=2.0cm]{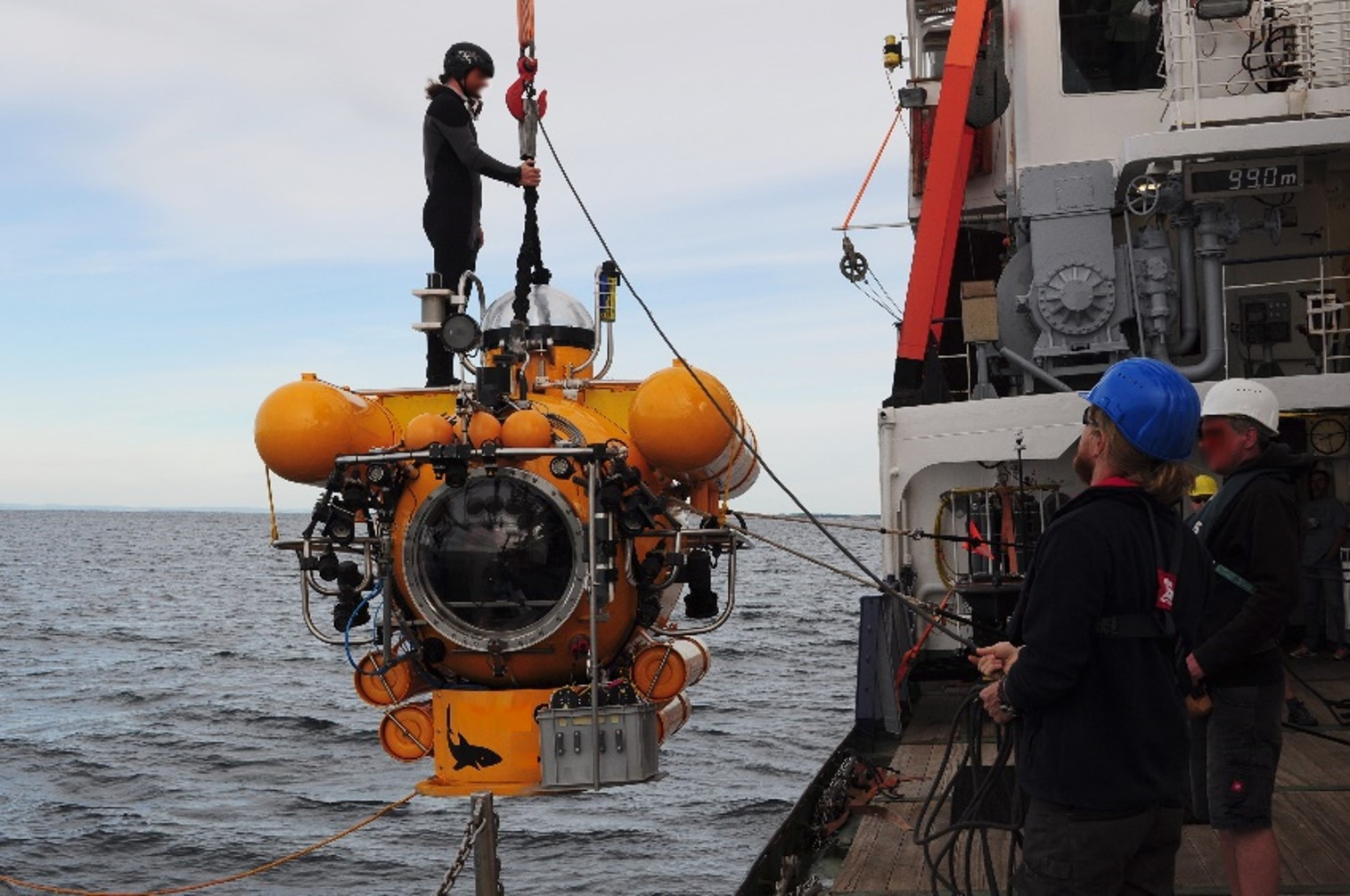}						
		\includegraphics[height=2.0cm]{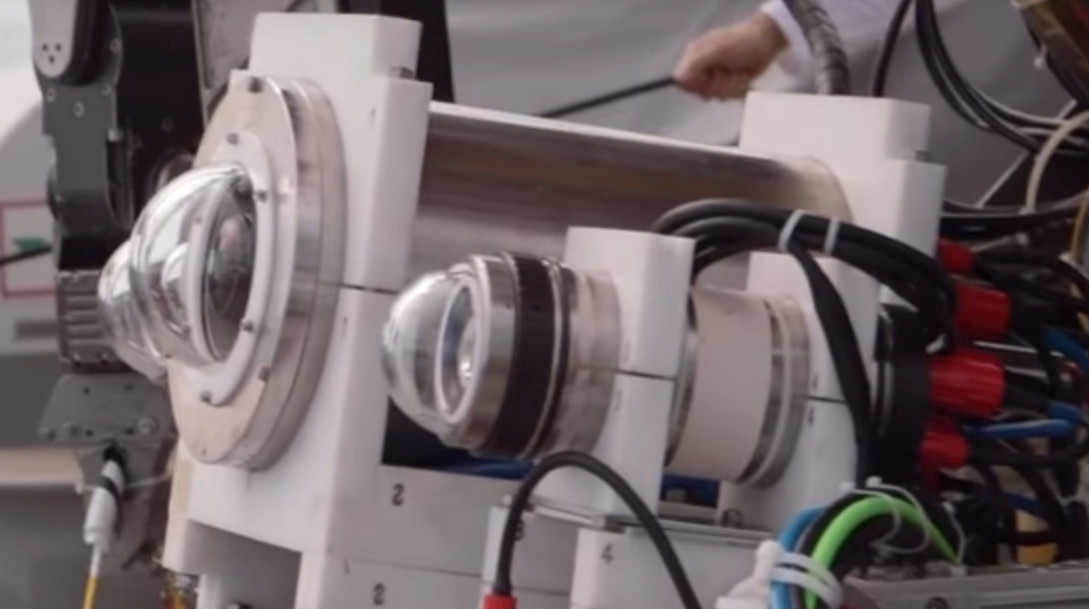}
		\caption{From left to right: Autonomous underwater vehicle with dome port camera, an underwater photographer with dome for DSLR, submersible with huge spherical window, multi dome port camera system mounted on remotely operated deep sea robot. These domes can avoid refraction effects when all principal rays of the lenses pass the glass spheres in direction of the local normal. Perfectly centering the lenses in the dome is however a hard task, since it is neither obvious where the center of the dome is, nor where the nodal point of the camera is, nor how good an alignment has been achieved.} \label{fig:domeport}
	\end{center}
\end{figure*}
In oceanographic applications, but also for professional photographers, so-called {\em dome port systems} have become popular that rely on a spherical window to avoid view limitations.
They are also mechanically more stable and relatively thin spherical glass can resist extremely high pressure at several kilometers of depth. In principle, a well centered lens behind the dome port is able to avoid refraction, but it still remains challenging as both optical centers are physically invisible, especially for large lenses and rugged equipment as shown for example in Fig. \ref{fig:domeport}. 
If the lens is not well centered, again refraction will occur as with flat ports, but now the decentered dome geometry produces even more complex refraction effects in the 2D image, a threat that might discourage people from using dome systems. 
Current practical solutions adopt standard pinhole calibration parameters to compensate the remaining refraction for an approximately centered lens, which allows to achieve high accuracy and has been widely applied in shallow water survey tasks \cite{drap2012underwater,menna2013photogrammetric,figueira2015accuracy,nocerino2020coral}. 
However, this solution is usually performed at an ideal working distance with a well designed control network, which is difficult to achieve in less controllable scenarios such as  robotic mapping applications in the deep ocean. 

In this contribution, we geometrically analyze refraction at the sphere for the common case that the lens is not exactly centered with the dome and show that the system is actually an axial camera. 
We derive a chessboard-based direct solver for the refraction axis, distinguish positive and negative decentering, and propose a complementary decentering calibration procedure that can support mechanical adjustment of the lens to avoid refraction in the first place.
In case entirely avoiding refraction is not possible, the analyses and the methods for estimating the decentering presented in this paper are intended to facilitate further research, whether and when it could be benefitial to explicitely consider such physical parameters in photogrammetric surveys.

\begin{figure*}[!h]
	\begin{center}
	\def\svgwidth{0.4\textwidth}
	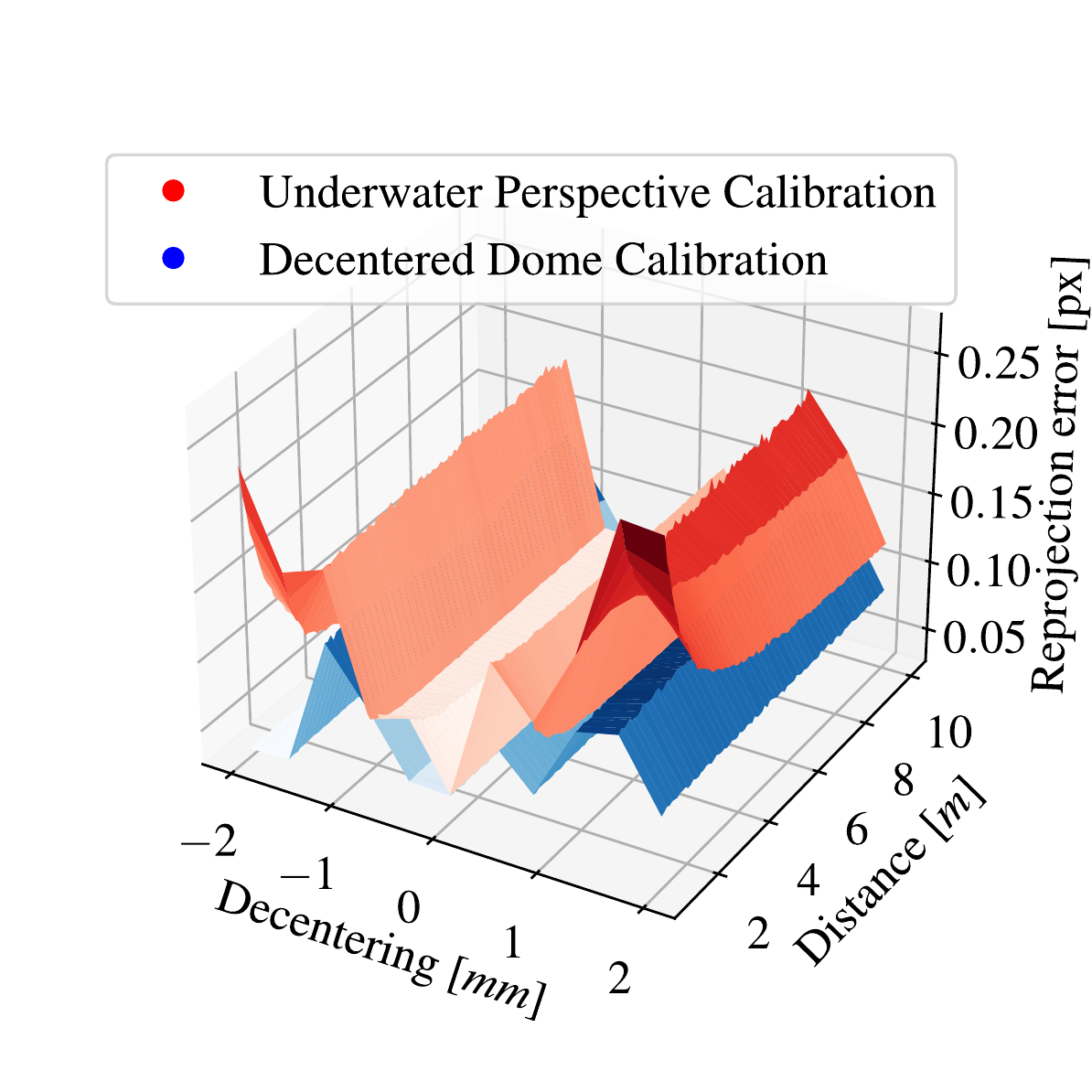
	\def\svgwidth{0.4\textwidth}
	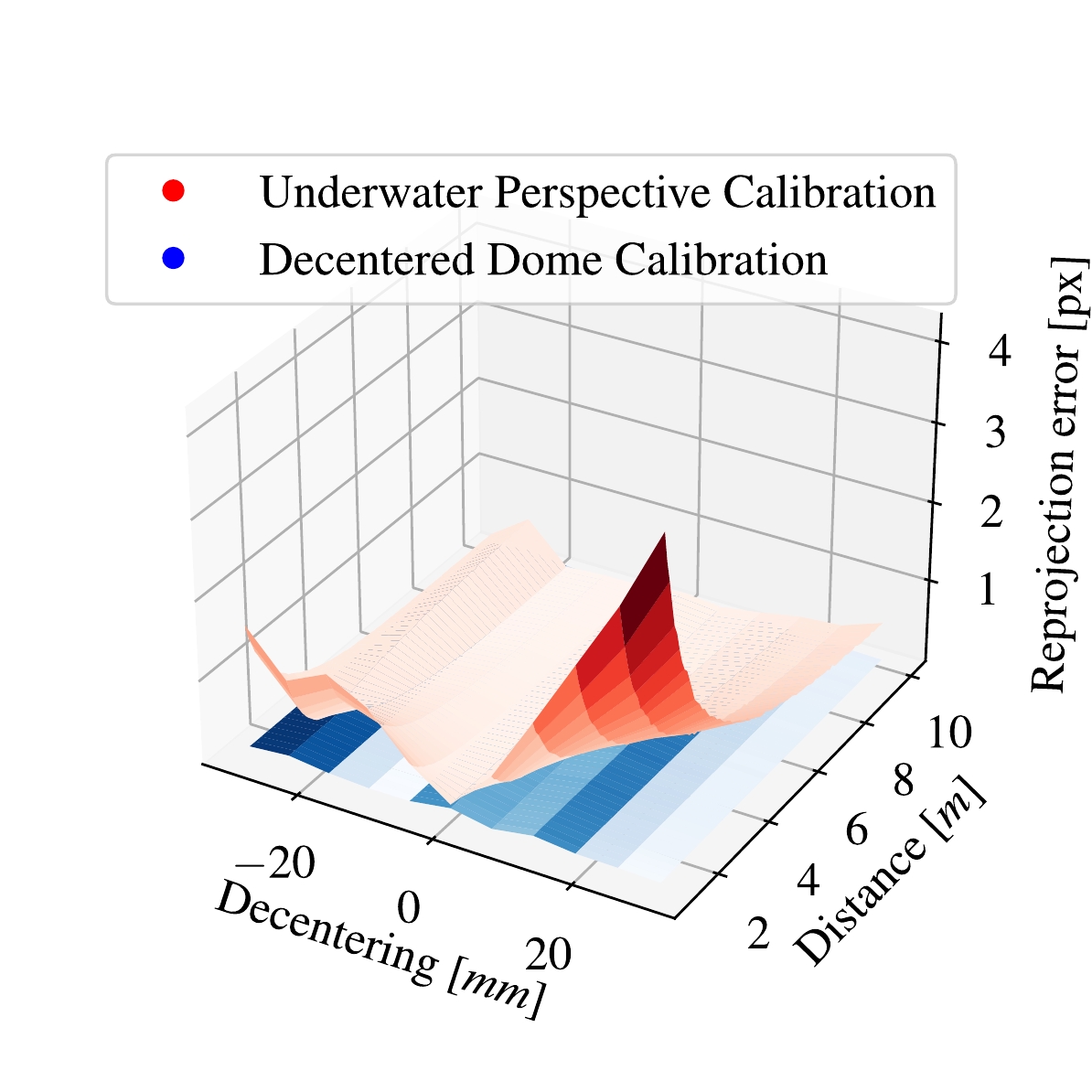
	\end{center}
	\begin{center}
	\def\svgwidth{0.4\textwidth}
	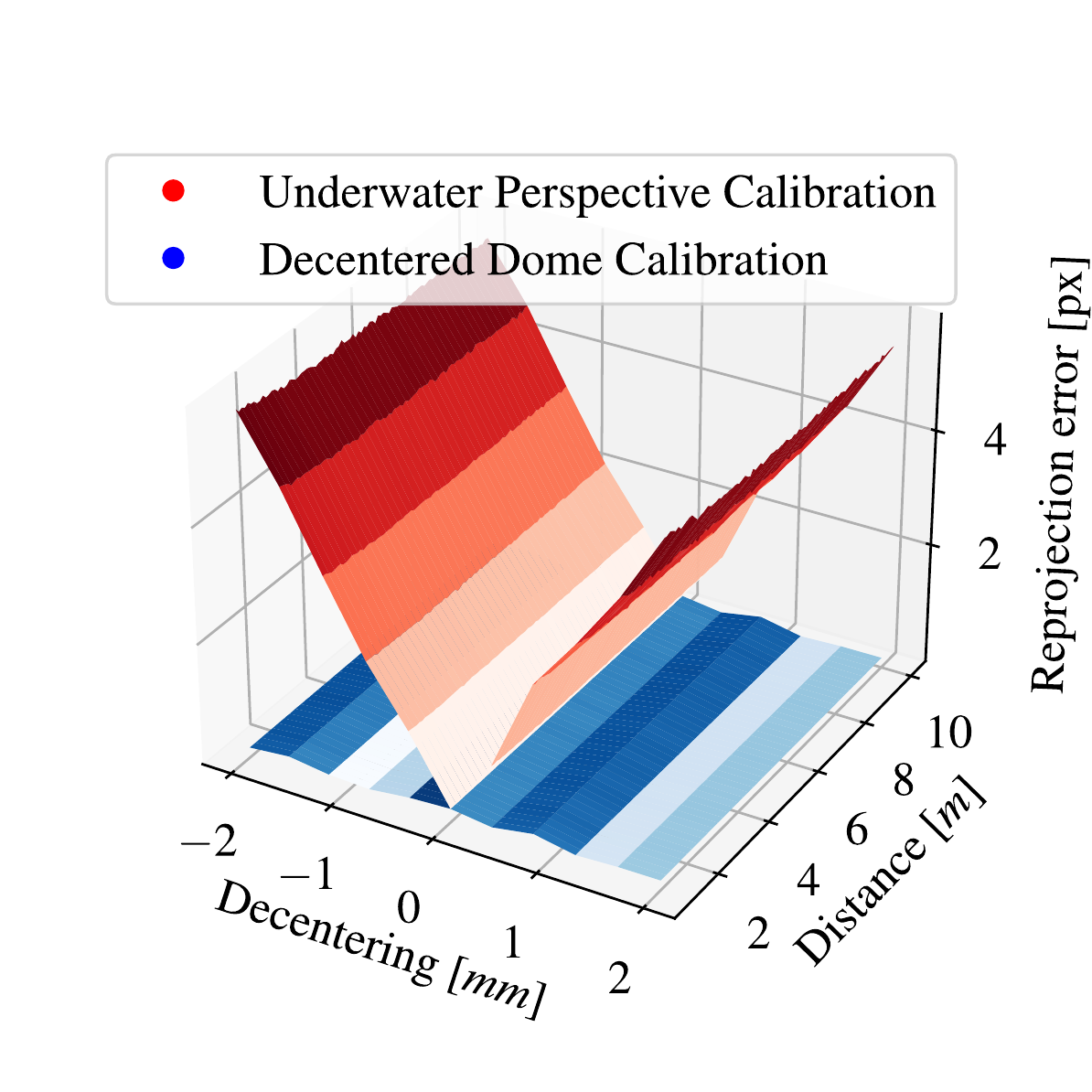
	\def\svgwidth{0.4\textwidth}
	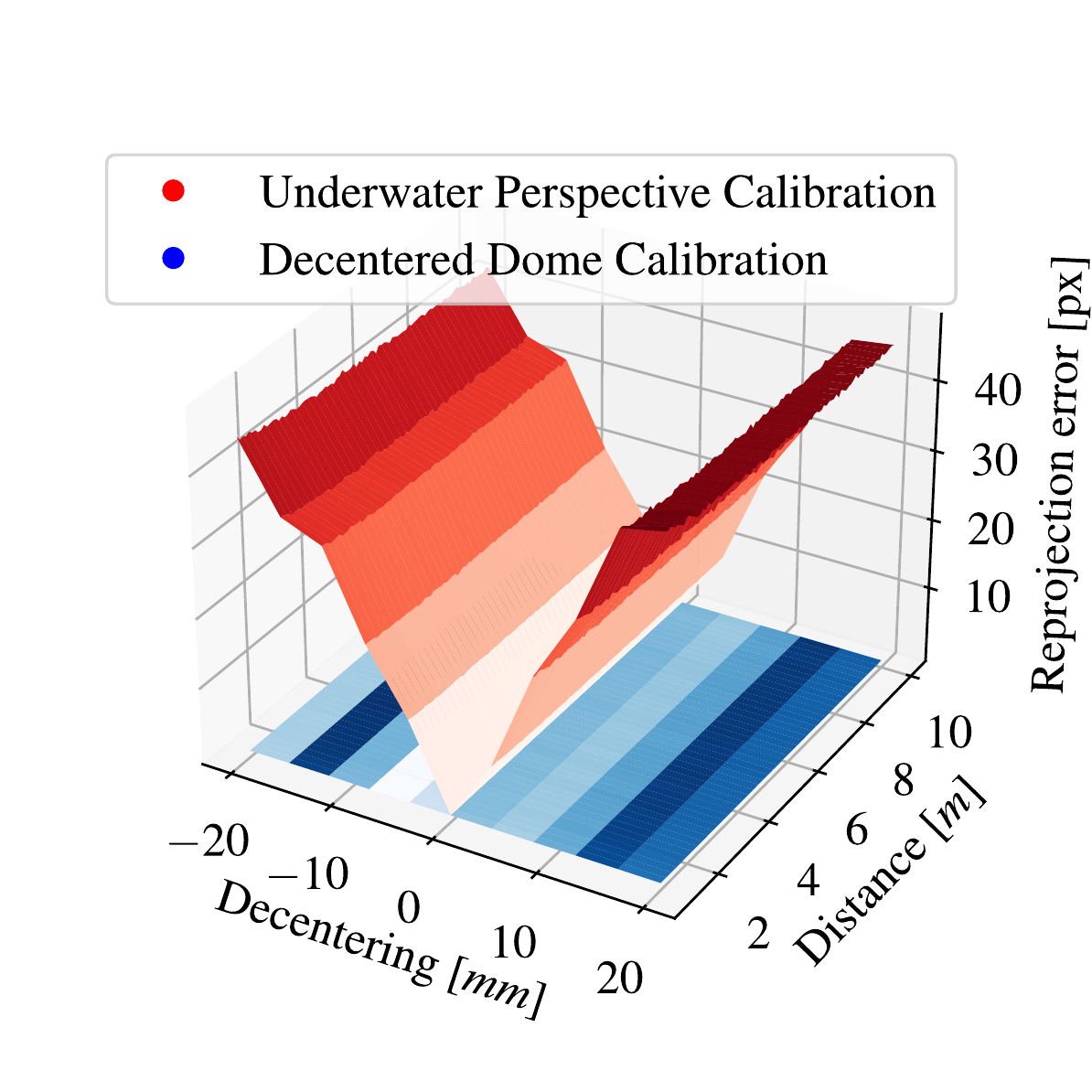	
	\end{center}
	\caption{
		Motivation experiment, we numerically compare the validation accuracy between the physically-based decentered dome port model and the underwater perspective approximation. 
		To this end, we first model a dome port camera system synthetically, and the camera projection center is decentered in forward-backward direction and left-right direction with different magnitudes. 
		Next, for each decentering, we simulate 30 planar chessboard calibration images with $0.2$ pixel noise as if the target was photographed underwater.
		Then, the underwater dome port calibration and the perspective camera model calibration are both performed with the same set of virtual images.
		After calibration, we evaluate the calibration quality by computing the reprojection error of a set of ground truth 3D-2D point pairs.
		The ground truth 3D-2D point pairs are created by randomly and uniformly sampled image points in the image space, and back-projected to 3D space with known scene distances (vary from $1m-10m$, 2000 samples in each distance). 
		Afterwards, the 3D points are reprojected onto the image space using the calibrated decentered dome port and the calibrated perspective model respectively, and the reprojection errors are measured and visualized.
		Top row, decentering in forward-backward direction.
		Bottom row, decentering in left-right direction.
	}
	\label{fig:motivation_experiment}
\end{figure*}
%

\section{Related Work}
It is well known in photogrammetry that refraction occurs when photographing through different media and in particular in underwater scenarios (see e.g. \cite{Shmutter67twomedia,Moore_76-underwaterphotogrammetry,Kotowski_88-refracPhotogrammetry,Freyer_86-UnderwaterCalib,Maas_95-refraction,Harvey_98-underwaterStereoStability,Jaffe_2001-underwaterImg,kunz2008hemispherical,Drap12}).
Early photogrammetric methods have analytically analyzed the refractive scenario \cite{Kotowski_88-refracPhotogrammetry} and suggested linearized correction terms in multi-media photogrammetry at flat interfaces \cite{Maas_95-multimediaPhotogrammetry}, whereas other approaches suggested to absorb refractive effects into 2D lens distortion.
Many practical underwater photogrammetry applications have shown that approximating the dome port camera system as a perspective camera model with standard distortion parameters is sufficient and convincing 3D survey results can be obtained.

Gläser and Schröcker\cite{glaeser_2000-reflectionsOnRefractions} found the analytical forward projection for a single refraction. Treibitz et al. \cite{Treibitz_12_flatrefractivegeometry} have shown that overall systems of cameras behind flat glass interfaces in the water can be considered axial cameras and Chari and Sturm \cite{chari2009refractiveplane} have inspected related multi-view relations. 
Agrawal et al. \cite{Agrawal_2012-UnwCalib} have generalized analytic forward projection to multiple layers of flat interfaces. 

For spherical glass domes, on the contrary, much less work exists, although they can sustain more pressure - and are thus better suited for deep sea applications - and they do not limit the field of view (see Fig. \ref{fig:domeport} for examples).
Nocerino et al. \cite{nocerino2016underwater} compare how aerial standard calibration parameters change when submerging dome port cameras underwater, but disregard explicit consideration of refraction. 
Kunz and Singh \cite{kunz2008hemispherical} discuss that domes avoid refraction when a pinhole camera is exactly centered with the dome and suggest that decentering could be determined using optimization. They analyze the 2D pixel error when approximating a misaligned dome port camera as a perspective camera and state that the error will increase as the camera moves closer or further to the observed scene since the refraction distortion is depth-dependent. Therefore, centering the camera with the dome port is critical to underwater vision applications.
Here, Menna et al. \cite{menna_16-characterizationdomeport} practically measure a particular dome port geometry and discuss properties of domes as compared to flat port cameras.
They also suggest how to align the nodal point of a lens with the dome's center by individually measuring both points and then mounting without a feedback loop.
Recently, in \cite{Menna2020UWSystematic}, the authors investigate the depth-dependent systematic errors introduced by refraction effects through dome ports in a statistical way and apply iterative look-up table corrections to reduce systematic residual patterns \cite{Nocerino2021Bundle}.

To support such a theory, and also to motivate this work, we conduct a numerical experiment which compares the physically-based decentered dome port model and underwater calibration using a perspective model. We particularly look at the case of a camera-equipped robot that sometimes comes close to underwater structures, and sometimes observes them from a distance, which means that calibration is required to hold for a range of different distances.
As shown in Fig. \ref{fig:motivation_experiment}, the underwater pinhole model works well for a mild decentering in optical axis direction (see top row), but it is noticeable that the standard perspective camera model has more difficulty to absorb the refraction effect with larger decentering or when there is even only a slight sideward decentering (see bottom row). In these cases explicitely considering refraction does not suffer from such errors and this motivates the work in this contribution.

Already earlier, She et al. \cite{she2019adjustment} have proposed to mechanically adjust the dome and lens using through-the-lens feedback until no refraction effects are observable.
They build their feedback loop on a human operator to judge the continuity of lines in a setup where the camera is looking parallel to the water surface in a special tank.
Later they determine the remaining offset by an image pair of a chessboard in air and in water, both taken at exactly the same pose, which makes the approach delicate and complex.
However, there are cases where the camera cannot be centered with the dome port.
For instance, Bosch et al. \cite{bosch2015omnidirectional} have designed a special underwater camera housing for an omnidirectional camera system which consists of a cylinder for the lateral cameras and a hemispherical dome for the top-looking camera, then
a calibration scheme to estimate the extrinsic parameters of each camera and the housing parameters is proposed.
Refraction effects are considered using the ray-tracing based camera model.
However, their work reports a relatively high residual after the parameter optimization in the real-world evaluation.
Iscar and Johnson-Roberson \cite{iscar2020towards} have developed an underwater stereo camera system where two cameras share a single hemispherical glass port.
Thus, the offset between the camera's optical center and the dome center is extreme in this scenario.
They propose ideas to recover the camera's position inside the dome by measuring the point spread function of the overall camera system, which requires a complex dataset collection procedure. The authors report that their calibration approach is very time-consuming and practically limited.
Besides underwater imaging, exact modeling of refraction in the image formation process can also improve in-air applications such as cameras behind the windshield of cars \cite{Verbiest2020Windshield} or behind an optical cover for coal mining vehicles \cite{yang2021non}. Note however, that when imaging in air through a relatively thin glass pane, rays essentially only undergo a slight lateral shift depending on the thickness of the glass, whereas in our scenario the refraction effects are dominated by the other two media (air and water) with significantly different optical densities that cause large direction changes of rays.

In summary, in practise dome port cameras are often successfully modeled by the pinhole camera system and very special setups of domes (e.g. stereo camera in a single dome) needed very special treatment that does not directly apply to the case of mildly decentered dome systems. In this contribution we therefore want to characterize the geometrical properties of decentered dome systems such as important axes, directions, symmetries,  what kind of refraction patterns are to be expected and how refraction effects can be exploited to infer the physical parameters of the camera system.

Therefore, our contribution in this paper is two-fold:
In the next section, we derive the geometrical properties for decentered dome port systems and show that even thick domes are actually also axial cameras, similar to the findings of \cite{Treibitz_12_flatrefractivegeometry,Agrawal_2012-UnwCalib} for flat interfaces. 
In contrast to flat interfaces, spheres have two intersections with the refraction axis, and we discuss their different properties.
Second, we derive how to estimate the refraction axis from a single picture of a chessboard drawing on prior insights into radial distortion center estimation \cite{hartley2007parameter}.
We also show that the apparent chessboard curvature (barrel vs. pincushion) is related to backward and forward decentering respectively and propose how to distinguish these cases in a 2D image.
Afterwards, we analyze the forward projection of a 3D scene point onto the image plane given a calibrated decentered dome port camera system.
We show that a 6th degree polynomial equation for thin domes can be derived by leveraging the property of the axial camera \cite{agrawal2010analytical,Agrawal_2012-UnwCalib,glaeser_2000-reflectionsOnRefractions}. 
Next, we derive an efficient optimization scheme to find the exact decentering from multiple underwater chessboard pictures, which provides a practical calibration procedure that does not require cumbersome in-air/underwater pair of the same chessboard.
Finally, we discuss the limitation and the calibration accuracy when the decentering is not significant and propose an image pre-selection scheme to maximize the observable refraction effect to achieve a high calibration accuracy.

The remainder of this paper is structured as follows: In the following section 3, we will discuss and derive the refractive dome geometry based upon which we propose the different novel steps of our calibration algorithm in section 4.
The single components and an overall calibration approach are evaluated in section 5 using synthetic and real datasets.
Then, we discuss key findings and limitations in section 6 before the conclusion in the final section 7.

\begin{figure*}[t]
	\begin{center}
		\def\svgwidth{0.42\textwidth}
		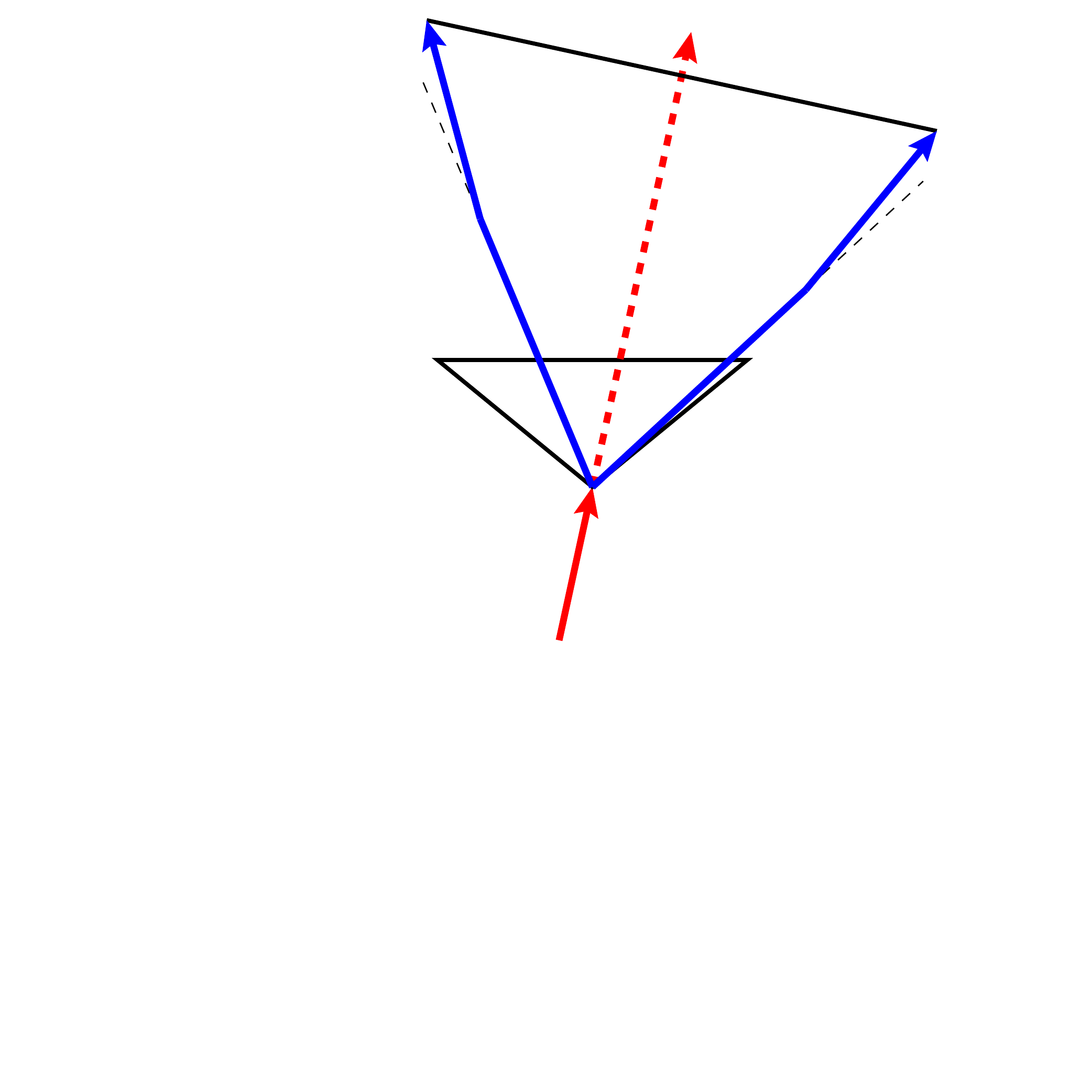
		\def\svgwidth{0.42\textwidth}
		{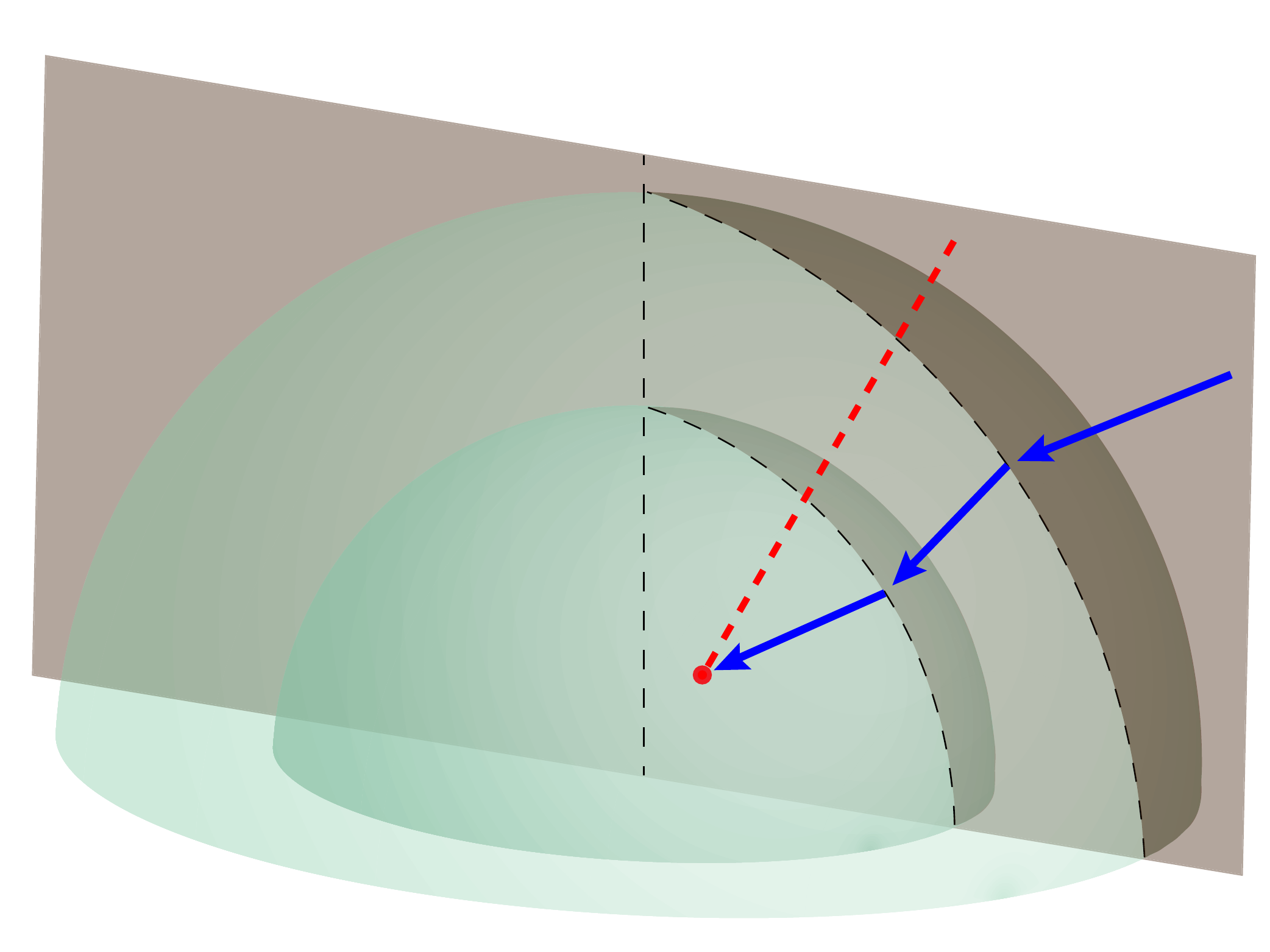}
		\def\svgwidth{0.42\textwidth}
		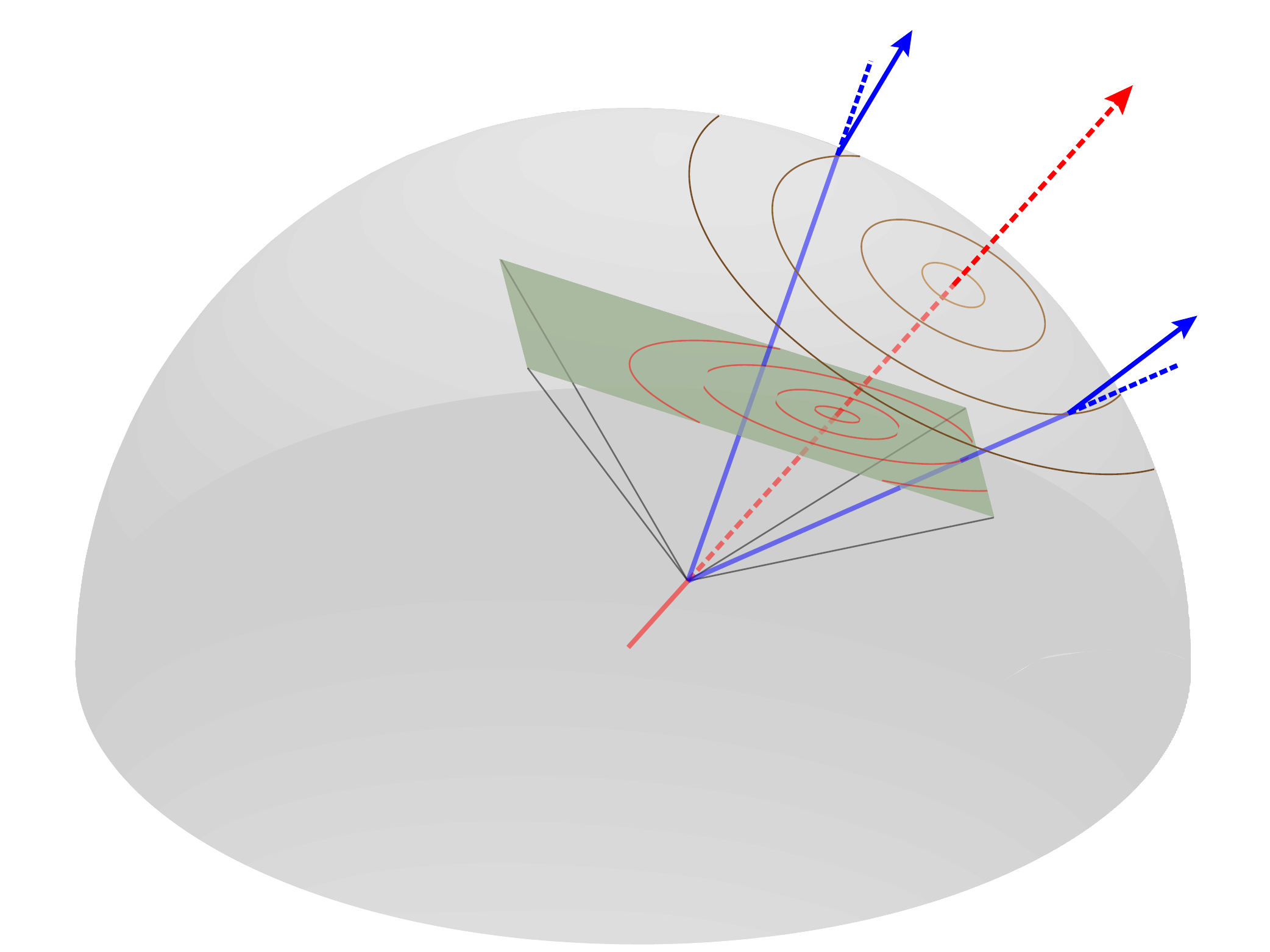
	\end{center}
	\caption{Sketch of the dome refraction setup. Up left: The camera center is not located at the center of a thin dome (the origin). Still, the ray from the dome center through the camera center will not be refracted at the sphere. Another ray, going from the camera more to the left, will not continue straight (dashed line) but rather be refracted. The dotted line, the refracted line and the line from the origin to the dome center (the refraction axis) lie all in one plane. The upper right image shows the intersection of this plane with the sphere and for each viewing ray, a plane of refraction exists, even in case we use a thick dome.
		In the lower sketch, the dashed light paths form the same angle $\phi$ with the refraction axis at the camera center. Consequently, they will be refracted by the same angle $\theta$ at the interface. The cone of all those rays intersects the sphere in a 3D circle, and intersects the image plane in an ellipse.
	}
	\label{fig:main_sketch}
\end{figure*}

\section{Decentered Dome Geometry}
\label{sec:geometry}

The setting is displayed in Fig. \ref{fig:main_sketch}: 
We assume that a camera is positioned inside a sphere.
The medium inside the sphere (e.g. air) has a different optical density $\mu_\mathrm{air}$ as compared to the medium outside the sphere (e.g. water, with density $\mu_\mathrm{water}$).
The separating layer (e.g. glass) can either be considered of almost zero thickness (thin dome model), or, in particular for deep sea housings sustaining several hundred bars of pressure, can consist of several millimeters glass (thick dome model) with optical density $\mu_\mathrm{glass}$. The exact numbers for the $\mu$ parameters depend on the composition of the water, the exact "glass" material, but in the remainder of this paper we will assume $\mu_\mathrm{air} \leq \mu_\mathrm{water} \leq \mu_\mathrm{glass}$ to reason about inward- or outward bending.
In case a pinhole camera\footnote{ In this contribution we consider the mathematical pinhole camera model, cf. to \cite{Hartley2004_MultipleView}.
Real cameras use lenses to collect light over a larger area and need to be focused to the object under consideration.
When used behind a dome port, the focusing behaviour of the lens depends on the medium outside the dome, and is different in air or in water (cf. to \cite{menna_16-characterizationdomeport}).
While focusing strategies are an interesting problem in their own right, in this contribution we consider the principal rays of the lens, and use the pinhole model to describe the mapping from world coordinates to image coordinates.}
is positioned exactly in the center, all viewing rays will pass all dome layers in direction of the local normal and no refraction will occur. In practice, aligning the centers is very challenging, and therefore some decentering is likely to remain when assembling without visual feedback.
The vector from the dome center to the camera center is the decentering offset vector $\q v_{\mathrm{off}}$, or short {\em decentering}.
The line through the dome center and the camera center is the refraction axis with direction $\q a$.
It has two intersections with the (thin) dome surface which will be called the refraction poles, where we distinguish the pole closer to the camera center (positive refraction pole $\q I_{\mathrm{pol´e+}}$) and the pole further away (negative refraction pole $\q I_{\mathrm{pole-}}$). The refraction axis also intersects the image plane at the refraction center, which has the position $\q r$ in image coordinates.

We set the origin of the world coordinate system to the center of the dome. Further, we assume that the camera is calibrated, and omit the camera intrinsics for the sake of readability, so it can be described by the projection matrix
\begin{equation}
\mq P = (\mq R \;\;|\; -\mq{R} \q v_{\mathrm{off}}).
\end{equation}

Because of refraction effects according to Snell's law, light rays will change their direction at the interface between different media, unless they hit the interface at an angle coinciding with the normal of the intersection point.

Along the light path from an object in the water, through the glass dome and into the camera we consider 3 segments here, the water segment with light ray direction $\q l_{\mathrm{water}}$, the glass segment with direction $\q l_{\mathrm{glass}}$ and the air segment with $\q l_{\mathrm{air}}$ (see Fig. \ref{fig:main_sketch}, Center).
Note that incoming rays that travel along the refraction axis will not be refracted, as they hit the outer interface at 90$^\circ$ and then also the inner interface at 90$^\circ$ before they move towards the camera.
We will now trace back a different light ray $\q l$ from the pinhole to its intersection $\q I_{\mathrm{inner}}$ with the inner interface.

\begin{lemma}
	\label{lem:linearcombination}	
	The surface normal $\d n_{\mathrm{inner}}$ of the inner interface at $\q I_{\mathrm{inner}}$ is a linear combination of $\q v_{\mathrm{off}}$ and $\q l_{\mathrm{air}}$.
\end{lemma}
\begin{proof}
	The intersection point $\q I_{\mathrm{inner}}$ must be somewhere on the ray starting from the camera center  $\mq v_{\mathrm{off}}$ in direction of $\q l_{\mathrm{air}}$, consequently $\q I_{\mathrm{inner}} =  \q v_{\mathrm{off}} + \lambda \q l_{\mathrm{air}}$. 
	Since the surface normal at some point $\q X$ on an origin-centered sphere is simply $\q X / ||\q X||$, we obtain
	$\q n_{\mathrm{inner}}= \alpha_1 \q v_{\mathrm{off}} + \alpha_2 \q l_{\mathrm{air}} $. 
\end{proof}

\begin{lemma}All segments $\q l_\star $ of a light path towards the pinhole, the normals at the inner and outer dome intersection point and the refraction axis all lie in one plane.
	\label{lem:planeofrefraction}	 
\end{lemma}
\begin{proof}
	From lemma \ref{lem:linearcombination}, it can be seen that $\q I_{\mathrm{inner}}$ and the normal at this position are both linear combinations of $\q v_{\mathrm{off}}$ and $\q l_{\mathrm{air}}$, consequently, they all lie in the plane spanned by $\q v_{\mathrm{off}}$ and $\q l_{\mathrm{air}}$. According to Snell's law, the glass segment direction $\q l_{\mathrm{glass}}$ can be computed as a linear combination of the incoming direction $\q l_{\mathrm{air}}$ and the local surface normal $\q n_{\mathrm{inner}}$ \cite{glassner1989introRayTracing}:
	\begin{equation}
	\label{eq:calc_refracted_ray}
	\q l_{\mathrm{glass}} = r \q l_{\mathrm{air}} + c \q n_{\mathrm{inner}}
	\end{equation}
	where $r = \mu_{\mathrm{air}} / \mu_{\mathrm{glass}}$, and
	\begin{multline}
	c = \frac{-\mu_{\mathrm{air}} \q l_{\mathrm{air}} \trans \q n_{\mathrm{inner}}}{\mu_{\mathrm{glass}} (\q n_{\mathrm{inner}}\trans \q n_{\mathrm{inner}})} -\\ \frac{\sqrt{\mu_{\mathrm{air}}^{2} (\q l_{\mathrm{air}} \trans \q n_{\mathrm{inner}})^{2} - (\mu_{\mathrm{air}}^{2} -\mu_{\mathrm{glass}}^2 ) (\q l_{\mathrm{air}} \trans \q l_{\mathrm{air}}) (\q n_{\mathrm{inner}}\trans \q n_{\mathrm{inner}})}}{\mu_{\mathrm{glass}} (\q n_{\mathrm{inner}}\trans \q n_{\mathrm{inner}})}.
	\end{multline}
	
	Consequently, following the same reasoning, also the outer interface point, its normal and the water segment are linear combinations of $\q v_{\mathrm{off}}$ and $\q l_{\mathrm{air}}$ and all lie in the same plane.
\end{proof}
In the flat refractive case \cite{Agrawal_2012-UnwCalib}, this plane is called the plane of refraction.
We will now investigate by what angle a light ray changes its direction at the inner interface.
We conceptually group light rays that start from $\q v_{\mathrm{off}}$ and change their direction at the inner interface by some angle $\theta$ into the set $\mathcal{R}_\theta$ (see Fig. \ref{fig:main_sketch}, Right). We call the corresponding set of image coordinates {\em iso-refraction-angle curves}.

\begin{theorem} Iso-refraction-angle curves are conic sections in the image. \label{theorem:isorefraction} \end{theorem}
\begin{proof}For symmetry reasons it can be seen that all rays starting at $\q v_{\mathrm{off}}$ that form some angle $\phi$ with the refraction axis (the $\phi$-cone) intersect the inner sphere in a 3D circle (see Fig. \ref{fig:main_sketch}). They all form the same angle with the local normal at the sphere, and due to Snell's law, they will thus all be refracted by the same amount. The intersection of the $\phi$-cone with the image plane forms a conic section.
\end{proof}
Finally, we will inspect the displacement induced by the different optical densities of the media. In case all media (air, glass, water) had the same refraction index, we could compute the image position $\q x_a$ of an observed 3D point $\q X$ simply by using the projection matrix: $\q x_a \simeq \mq P \q X$. This is the position where we would expect to see the point without a dome, just "in air". In case the three media have different optical densities the light will undergo refraction and we will observe the same 3D point at another position $\q x_r$. 3D points that lie on the refraction axis will not be refracted and rather be projected onto the refraction center $\q r$ in the image.

\begin{theorem} The "in-air" observation $\q x_a$ of a point $\q X$ and the underwater "refracted" position $\q x_r$ of that point $\q X$ form a line with the refraction center $\q r$. \label{theorem:line}\end{theorem}
\begin{proof}
	From lemma \ref{lem:planeofrefraction} it can be seen that all the segments of the light path, including the 3D point $\q X$, are in a plane jointly with the refraction axis. The intersection of that plane with the image plane forms a line. 
	Since the camera center is in that plane too, also the direct line from the camera center to $\q X$ is in that plane, and thus the unrefracted "in-air" observation must be in the same intersection line of that plane with the image plane.	
\end{proof}
This means that refraction happens actually along a line containing the unrefracted "in-air" observation and the refraction center (cf. also Fig. \ref{fig:chessrefraction}).
Now we consider the overall dome camera system in the water as a special kind of camera. We simply extend the water segments that lead from 3D points to the outer sphere to (infinitely long) lines:

\begin{theorem} The decentered dome port camera system is an axial camera. \end{theorem}
\begin{proof}
	From lemma \ref{lem:planeofrefraction} it can be seen that each water segment of a light path is in a plane jointly with the refraction axis. The water segment is either parallel to the axis (intersection at an ideal point at infinity) or its extension will have a Euclidean intersection with the axis. Consequently, all water segments of paths reaching the pinhole, i.e. the viewing rays in water, intersect the axis and the overall system is an axial camera.
\end{proof}

\begin{theorem} In the thin dome port camera system in water, all the rays from the camera are refracted towards the positive refraction pole. \label{theorem:refractiondirection}\end{theorem}
\begin{proof}
	From lemma \ref{lem:planeofrefraction} we know that all ray segments of a light path lie in a single plane with the axis, i.e. the incoming and outgoing ray at the refraction interface change direction inside this plane. According to Snell's law, when going from the lighter to the denser medium, rays change their direction towards the surface normal. Since all the surface normals cross the dome center, and since the camera is positioned closer to the positive refraction pole, the incident rays from the camera are thus always refracted towards the positive refraction pole (see also Fig. \ref{fig:refraction_all_rays}, Left).
\end{proof}
This means that in the image the 2D displacement direction with respect to the refraction center (inwards or outwards) depends on the sign of the 3D decentering vector and is the explanation of the barrel and pincushion distortion (depending on backward or forward decentering) known to underwater photographers and empirically reported e.g. in \cite{menna_16-characterizationdomeport}.

\begin{figure}[t]
	\begin{center}
		\def\svgwidth{0.48\textwidth}
		{
\begingroup%
  \makeatletter%
  \providecommand\color[2][]{%
    \errmessage{(Inkscape) Color is used for the text in Inkscape, but the package 'color.sty' is not loaded}%
    \renewcommand\color[2][]{}%
  }%
  \providecommand\transparent[1]{%
    \errmessage{(Inkscape) Transparency is used (non-zero) for the text in Inkscape, but the package 'transparent.sty' is not loaded}%
    \renewcommand\transparent[1]{}%
  }%
  \providecommand\rotatebox[2]{#2}%
  \newcommand*\fsize{\dimexpr\f@size pt\relax}%
  \newcommand*\lineheight[1]{\fontsize{\fsize}{#1\fsize}\selectfont}%
  \ifx\svgwidth\undefined%
    \setlength{\unitlength}{600bp}%
    \ifx\svgscale\undefined%
      \relax%
    \else%
      \setlength{\unitlength}{\unitlength * \real{\svgscale}}%
    \fi%
  \else%
    \setlength{\unitlength}{\svgwidth}%
  \fi%
  \global\let\svgwidth\undefined%
  \global\let\svgscale\undefined%
  \makeatother%
  \begin{picture}(1,1)%
    \lineheight{1}%
    \setlength\tabcolsep{0pt}%
    \put(0,0){\includegraphics[width=\unitlength,page=1]{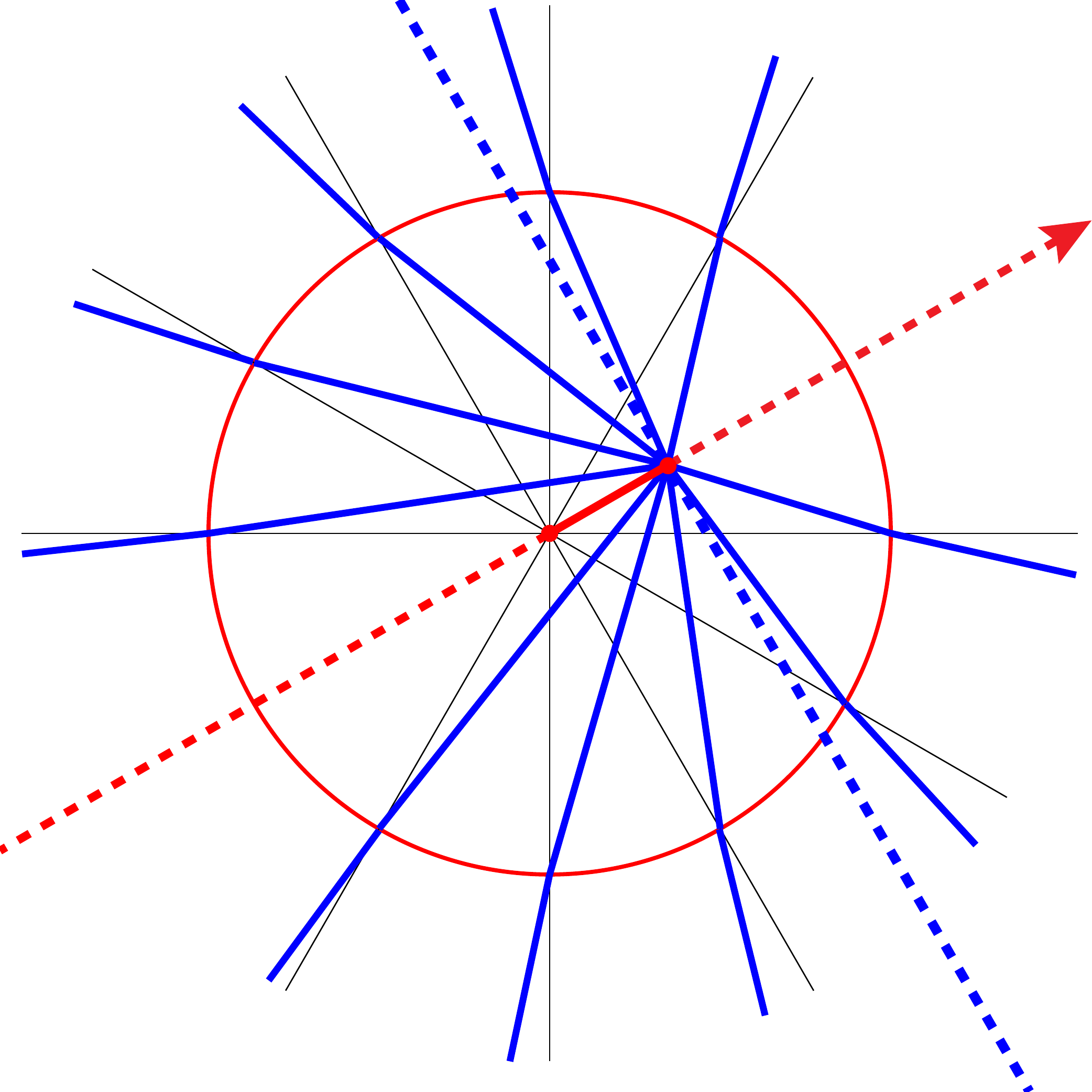}}%
    \put(0,0){\includegraphics[width=\unitlength,page=2]{refraction_dome_all_rays.pdf}}%
    \put(0,0){\includegraphics[width=\unitlength,page=3]{refraction_dome_all_rays.pdf}}%
    \put(0.801814,0.64824038){\makebox(0,0)[lt]{\lineheight{1.25}\smash{\begin{tabular}[t]{l}$\q I_\mathrm{pole+}$\end{tabular}}}}%
    \put(0.66758,0.54338263){\makebox(0,0)[lt]{\lineheight{1.25}\smash{\begin{tabular}[t]{l}Camera center \end{tabular}}}}%
    \put(0.42221313,0.40894287){\makebox(0,0)[lt]{\lineheight{1.25}\smash{\begin{tabular}[t]{l}Dome center \end{tabular}}}}%
    \put(0.0267065,0.92464112){\makebox(0,0)[lt]{\lineheight{1.25}\smash{\begin{tabular}[t]{l}Water\end{tabular}}}}%
    \put(0.28376038,0.684834){\makebox(0,0)[lt]{\lineheight{1.25}\smash{\begin{tabular}[t]{l}Air\end{tabular}}}}%
  \end{picture}%
\endgroup%
}\quad
		\def\svgwidth{0.48\textwidth}
		{
\begingroup%
  \makeatletter%
  \providecommand\color[2][]{%
    \errmessage{(Inkscape) Color is used for the text in Inkscape, but the package 'color.sty' is not loaded}%
    \renewcommand\color[2][]{}%
  }%
  \providecommand\transparent[1]{%
    \errmessage{(Inkscape) Transparency is used (non-zero) for the text in Inkscape, but the package 'transparent.sty' is not loaded}%
    \renewcommand\transparent[1]{}%
  }%
  \providecommand\rotatebox[2]{#2}%
  \newcommand*\fsize{\dimexpr\f@size pt\relax}%
  \newcommand*\lineheight[1]{\fontsize{\fsize}{#1\fsize}\selectfont}%
  \ifx\svgwidth\undefined%
    \setlength{\unitlength}{595.27559055bp}%
    \ifx\svgscale\undefined%
      \relax%
    \else%
      \setlength{\unitlength}{\unitlength * \real{\svgscale}}%
    \fi%
  \else%
    \setlength{\unitlength}{\svgwidth}%
  \fi%
  \global\let\svgwidth\undefined%
  \global\let\svgscale\undefined%
  \makeatother%
  \begin{picture}(1,1)%
    \lineheight{1}%
    \setlength\tabcolsep{0pt}%
    \put(0.1151129,0.79820405){\makebox(0,0)[lt]{\lineheight{1.25}\smash{\begin{tabular}[t]{l}Water\end{tabular}}}}%
    \put(0.32802788,0.64102282){\makebox(0,0)[lt]{\lineheight{1.25}\smash{\begin{tabular}[t]{l}Air\end{tabular}}}}%
    \put(0,0){\includegraphics[width=\unitlength,page=1]{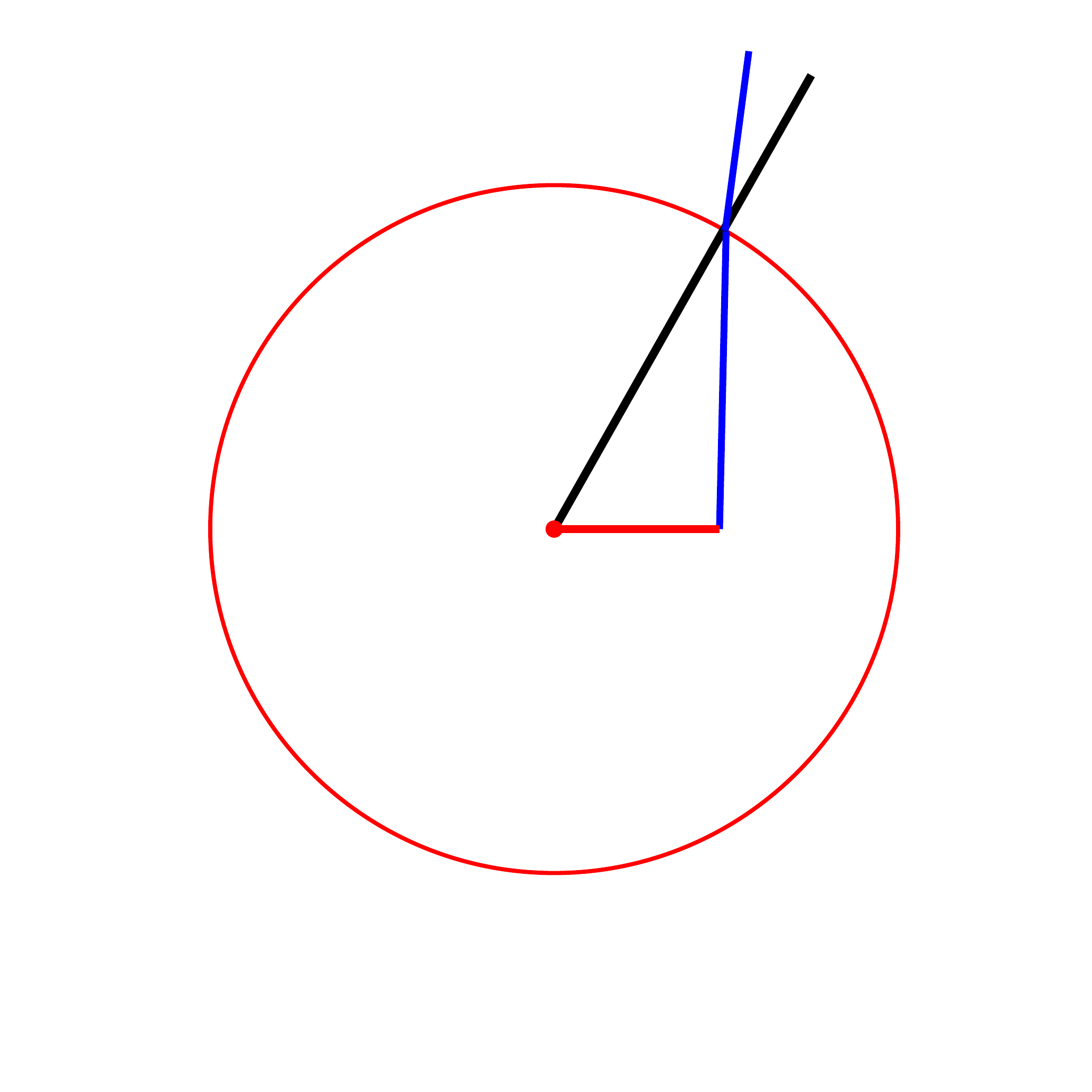}}%
    \put(0.66196518,0.47079573){\makebox(0,0)[lt]{\lineheight{1.25}\smash{\begin{tabular}[t]{l}C\end{tabular}}}}%
    \put(0,0){\includegraphics[width=\unitlength,page=2]{SupplementalMaterial.pdf}}%
    \put(0.46538999,0.47499498){\makebox(0,0)[lt]{\lineheight{1.25}\smash{\begin{tabular}[t]{l}O\end{tabular}}}}%
    \put(0.70927133,0.78072142){\makebox(0,0)[lt]{\lineheight{1.25}\smash{\begin{tabular}[t]{l}$\rm P$\end{tabular}}}}%
    \put(0,0){\includegraphics[width=\unitlength,page=3]{SupplementalMaterial.pdf}}%
    \put(0.60012084,0.6285239){\makebox(0,0)[lt]{\lineheight{1.25}\smash{\begin{tabular}[t]{l}$\alpha$\end{tabular}}}}%
    \put(0.70359723,0.94970281){\makebox(0,0)[lt]{\lineheight{1.25}\smash{\begin{tabular}[t]{l}$\beta$\end{tabular}}}}%
  \end{picture}%
\endgroup%
}
	\end{center}
	
	\caption{Left: Refraction of various blue rays from the camera center happens towards the positive refraction pole. Right: Maximum change of direction happens to the rays that enter the camera center from the plane perpendicular to the axis.}
	
	\label{fig:refraction_all_rays}
\end{figure}

\begin{theorem}\label{theo:perpRefrac} In the thin dome port model the maximum change of direction (refraction at sphere) happens to the rays that approach the camera center inside the plane perpendicular to the refraction axis. \end{theorem}

\begin{proof}
	Given a unit circle $\odot \rm O$, a point $\rm C$ inside the circle has the distance $k\in[0,1]$ to the circle center $\rm O$. A ray from $\rm C$ intersects $\odot \rm O$ at $\rm P$ with the incidence angle $\angle\alpha$ and an outgoing angle $\angle\beta$ (see Fig. \ref{fig:refraction_all_rays}, Right). The change of direction can be represented as $\angle \rm diff = \angle\alpha-\angle\beta$. According to Snell's Law, $n_{air}\sin{\alpha} = n_{water}\sin{\beta}$, change of direction in range $[0 , \pi/2]$ can be rewritten to:  
	
	\begin{align}
	\angle \rm diff = \alpha - \arcsin{(\frac{n_{air}}{n_{water}}\sin{\alpha})}, \quad  (\alpha \in [0 , \pi/2])
	\end{align}
	Its first derivative is:
	
	\begin{multline}
	\frac{\partial \angle \rm diff}{\partial \alpha}  = 1-\frac{1}{\sqrt{1-(\frac{n_{air}}{n_{water}} \sin{\alpha})^2}} \cdot{\frac{n_{air}}{n_{water}}}\cos{\alpha} 
	\\ 
	= 1 - \sqrt{\frac{1}{\frac{ n_{water}^2 - n_{air}^2 }{n_{air}^2} \cdot \frac{1}{\cos^2{\alpha}}+1}} > 0
	\end{multline}
	Since the derivative is strictly positive, $\angle \rm diff$ is monotonically increasing ($\alpha \in [0 , \pi/2]$) and does not have local maxima. Then the problem of finding a point $\rm P$ on the circle which has the largest changes of direction $\angle \rm diff$ is equivalent to find $\rm P$ which has the largest incident angle $\angle\alpha$. 
	
	Then, according to the Law of Sines, 
	\begin{equation}
	\frac{\overline{\rm OC}}{\sin \alpha} = \frac{\overline{\rm OP}}{\sin {\angle \rm POC}},
	\end{equation}
	since $\rm OC$ and $\rm OP$ are fixed, and $\sin \alpha$ is monotonically increasing in the range $[0 , \pi/2]$, 
	$\angle\alpha$ must have its largest value when $\sin {\angle \rm POC}$ reaches its maximum 1.
	Since $\angle \rm POC \in [0 , \pi]$, it follows that $\angle \rm POC = \pi/2$. 
	Therefore, the maximum incident angle $\angle\alpha$ on the circle happens when $\rm PC \bot OC$.
	
	Since, for the sphere all refractions happen in a plane of refraction, which always include the axis,
	we can subdivide the sphere surface into circles that include the poles, and in each of them consider the problem only as a 2F problem inside the specific plane of refraction. As shown above, in each of them the maximum change of direction happens perpendicular to the axis.
	
\end{proof}
This means that the ray that meets the camera center perpendicular to the refraction axis has suffered from the largest 
angular change (see Fig. \ref{fig:refraction_all_rays}, Left), whereas the ray on the axis is not at all refracted.
This is an important finding for setting up experiments to observe or to calibrate the decentering.
Note that the largest effect in the image also depends on the orientation of the camera, since the angular resolution of a pinhole camera increases towards the boundaries: Lateral (left-right, or up-down, in the camera coordinate system) decentering will provide a much clearer signal-to-noise ratio of refraction effects vs. corner detector uncertainty, as compared to forward-backward decenterings. 

\section{Decentered Dome Calibration}

\begin{figure*}[t]
	\begin{center}
		\def\svgwidth{1.0\textwidth}
		\scriptsize{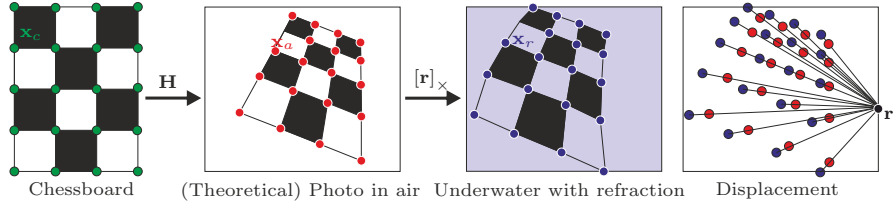}
	\end{center}
	\caption{Sketch of offset axis direction estimation principle: A chessboard (left sketch, corners $\q x_c$ marked in green) is photographed from an oblique viewpoint (second image), the projected positions without refraction $\q x_a$ marked in red. When submerging the camera and chessboard into the water behind a decentered dome, the light rays will be refracted (third image) and the projected positions $\q x_r$ are marked in blue. The displacement vectors (right image) between corresponding $\q x_a$ and $\q x_r$ all intersect in a single point, which corresponds to the refraction center $\q r$. This point can be estimated from correspondences of $\q x_c$ and $\q x_r$ (without knowledge of $\q x_a$ !) similar to an epipole (where all epipolar lines meet). }
	\label{fig:chessrefraction}
\end{figure*}
In this section, we will derive a calibration procedure from the insights of the previous section.
First, we will present the geometrical considerations that allow directly inferring the refraction axis and a distinction between positive or negative direction decentering from one underwater photo showing refracted chessboard corners. The result can be used to initialize a $\q v_\mathrm{off}$-optimization, using multiple images to actually measure the decentering.
Throughout the calibration, we assume the camera intrinsics are known.

\subsection{Direct Estimation of the Refraction Center}
We describe a corner's position on the original chessboard by $\q x_c$ (cf. to Fig. \ref{fig:chessrefraction}). 
When photographing a chessboard without refraction, the "as in air" image coordinates $\q x_a$ and the original chessboard pattern positions are related by a perspectivity \citep{Hartley2004_MultipleView}, a special kind of homography: $\q x_a \simeq \mq H \; \q x_c$

Keeping the chessboard pose, now consider the camera being behind a dome port and that the entire system of camera, dome and chessboard is submerged in water (underwater in Fig. \ref{fig:chessrefraction}). Imagine a line $\q q$ through the unrefracted point in air $\q x_a$ and the refraction center: $\q q = \left[ \q r \right]_\times \q x_a$

By theorem \ref{theorem:line}, the refracted point $\q x_r$ must be somewhere on this line:
$\q x_r\trans \q q = 0$.
If we now replace $\q q$ and $\q x_a$ we obtain a constraint that must hold between all points in the refracted image and their corresponding chessboard position:
\begin{equation}
\q x_r\trans \underbrace{\left[ \q r \right]_\times \mq H}_{\mq F} \q x_c = 0
\label{eq:refractionconstraint}
\end{equation}
This relation is reminiscent of epipolar geometry, where also all "displacement" (due to parallax) happens towards or away from the epipole. This principle has been exploited by \cite{hartley2007parameter} for calibration of radial distortion. We use it in a similar way to find the refraction center, but working on 3D rays rather than 2D points, and we do not have radial symmetry with respect to displacement in the image. Note that \cite{Agrawal_2012-UnwCalib} obtain an algebraically similar setting for refractive projection through flat interfaces. Essentially, as in epipolar geometry estimation, one can rearrange this equation using the Kronecker product and vectorization operator \citep{Fusiello2007AMO} to obtain constraints on the matrix $\mq F$:
\begin{equation}
\q x_c\trans \otimes \q x_r\trans \vec{\left( \mq F \right) } = 0
\label{eq:festimation}
\end{equation}
Basically, many of these equations can be stacked as for the eight-point algorithm for fundamental matrix estimation in order to estimate the vectorized matrix $\mq F$. After recomposition of $\mq F$, the refraction center $\q r$ can then be extracted as the left null vector of $\mq F$.

Since $\mq H$ is actually a perspectivity, it would even be possible to use a 5-point algorithm for essential matrix estimation, but as we will use the refraction center only as an initial guess for subsequent optimization, and as many reliable correspondences are obtained using a chessboard detector, the normalized 8-point algorithm will be a good fit for our purposes (also avoiding the ambiguities of up to ten solutions).
Multiple images of chessboards can be combined in the same way as described in \cite{hartley2007parameter} for 2D radial distortion center estimation. 
Note however that according to theorem \ref{theorem:isorefraction} iso-refraction curves (with respect to refraction angle) are conic sections in the image and that the refraction effect in pixels is depth-dependent, so it cannot be described by radial distortion and underwater images cannot simply be "unrefracted" without 3D scene knowledge.

\paragraph{Degenerate Cases}
In case the camera is already perfectly centered, the second-smallest singular value of the resulting equation system will also be zero and no unique $\mq F$ can be obtained. This corresponds to the case of no radial distortion in \cite{hartley2007parameter} or fundamental matrix estimation in a planar scene and can easily be detected. Besides that the same algebraic conditions hold as for fundamental matrix estimation (e.g. number of correspondences, non-collinearity \cite{Hartley2004_MultipleView}).

Now, using the camera orientation $\m R$, the corresponding refraction axis in world coordinates can be computed as:
\begin{equation}
\q a = \frac{\mq{R}\trans\q{r}}{ \parallel\mq{R}\trans\q r\parallel}
\label{eq:axis}
\end{equation}
Note that so far, we have obtained the refraction axis, but the sign of the decentering along the axis (forward vs. backward) is still missing. When drawing a line through two refracted chessboard corner points $\q x_{r,1}, \q x_{r,3}$ (using the cross product operator $\times$), we can determine whether a third chessboard point $\q x_{r,2}$ in between them is projected onto the same side of their connection line as the refraction center by using the convexity test:
\begin{multline}
\mathrm{convexity}(\q x_{r,1}, \q x_{r,2}, \q x_{r,3}, \q r)
\\
= \mathrm{sign}\left(\left(\q x_{r,1} \times \q  x_{r,3}\right)\trans \q x_{r,2} \; \cdot \left(\q x_{r,1} \times \q  x_{r,3}\right)\trans \q r\right)
\end{multline} 
A positive sign in this equation relates to barrel distortion, whereas a negative sign determines pincushion distortion of the refracted chessboard corners. From theorem \ref{theorem:refractiondirection} we can thus infer the sign of the decentering vector $\q v_\mathrm{off}$. In most practically relevant cases one will have an estimate on the maximum expected decentering (e.g. 10\% of the dome radius) allowing for a good initial guess for a later parameter optimization.

\subsection{Analytical Forward Projection for Thin Domes}

\begin{figure}[!h]
	\begin{center}
		\def\svgwidth{0.7\linewidth}
		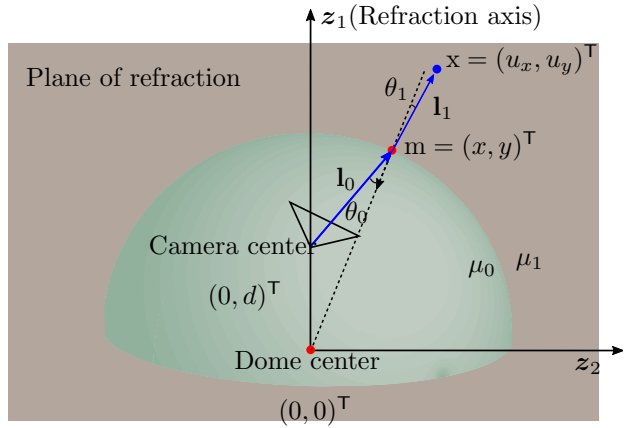
	\end{center}
	
	\caption{Due to the property of the axial camera, the refraction on the spherical layer can be analyzed on the plane of refraction. In the local coordinate system of the plane of refraction, $\d z_1$ aligns with the refraction axis and points from the spherical center to the camera center.}
	
	\label{fig:afp_derive}
\end{figure}

Now, we derive the analytical forward projection for decentered thin domes, where there is only one spherical layer of refraction. By lemma \ref{lem:planeofrefraction}, we know that all segments of a light path from the camera center and the refraction axis all lie in a single plane, therefore, the refraction on the spherical layer can be analyzed on the plane of refraction, which is similar to the derivation for multi-layer flat interfaces \citep{Agrawal_2012-UnwCalib}. But the difference is that the plane of refraction is constructed by the refraction axis and the 3D scene point.
Let 3D vectors $\d z_1$ and $\d z_2$ be the vertical axis and the horizontal axis of the 2D local coordinate system on this plane, and the spherical center be at the origin. Let $\d z_1$ align with the refraction axis and point from the spherical center to the camera's optical center, as illustrated in Fig. \ref{fig:afp_derive}.
Then, the normal of the plane can be found by taking the cross product between 3D point $\q X$ and the vertical axis $\d z_1$.
Afterwards, we can find the horizontal axis $\d z_2$ as the cross product between $\d z_1$ and the normal of the plane.
Therefore, a point on the plane of refraction has a 2D Euclidean coordinate of $(z_2, z_1)\trans$.

Assume the camera is decentered by a distance of $d$, thus the camera center has a coordinate of $(0, d)\trans$. The 2D coordinate of the 3D scene point $\q X$ on the plane is given by $\q x = (u_x, u_y)\trans$, where $u_x = \q X\trans
\d z_2$ and $u_y = \q X\trans \d z_1$.
The 2D refraction point $\q m = (x,y)\trans$ on the spherical interface satisfies the 2D circle equation $x^2 + y^2 = r^2$, where $r$ represents the radius of the sphere.
The 2D normal on the interface at the refraction point can also be simply computed as $n = -(x,y)\trans / r$.
Now, the 2D ray shot from the camera is $\q l_0 = (x, y-d)\trans$, and the refracted ray $\q l_1$ can be computed by equation \ref{eq:calc_refracted_ray}.
The line joining the refraction point $\q m$ and the scene point $\q x$ should be parallel to the refracted ray $\q l_1$, which provides another constraint.
By removing the square root term and substituting $x^2$ by $r^2 - y^2$, we end up with a 6th degree polynomial equation of a single variable $y$. Solving the equation results in up to 6 solutions, the correct solution can be found by checking Snell's law at the spherical refraction layer.

\subsection{Iterative Forward Projection for Thick Domes}
As outlined in the previous section, the thin flat port projection equation by \cite{glaeser_2000-reflectionsOnRefractions} of degree 4 becomes degree 6 for the thin dome port, essentially because of the quadratic nature of the refraction surface.
For thick flat ports, \cite{Agrawal_2012-UnwCalib} had derived a 12th degree polynomial for the analytical forward projection. Both for this case, and also 
for the case of imaging through a solid glass sphere (degree 10, also by \cite{agrawal2010analytical}), only some extra constraints of the special setting helped to bring down the polynomial degree to 12 resp. 10. For the thick dome, we haven't found a similar extra constraint, but even if found, the quadratic nature of the refraction surface comes into play and it is likely that the degree of the polynomial will be significantly higher than 12.
Solutions to high degree polynomials can become numerically unstable, and will also require iterative, numerical solvers.
Consequently, for forward projection through thick domes we turn to the numerical approach as proposed in \cite{kunz2008hemispherical} to find the projection by iteratively solving the inverse problem (back-projection), until the correct 2D point is found.
Note that according to theorem \ref{theorem:line}, the correct 2D point lies on the line joining the refraction center $\q r$ and the "in-air" observation $\q x_a$ of the 3D point. Rather than a 2D search, the search could now be restricted to the 1D line connecting $\q x_a$ and $\q r$, which would simplify iterative forward projection.

\subsection{Decentering Estimation} 
Having obtained a good start value from the direct solver and convexity test, in this section we describe an optimization procedure to optimize the decentering vector using several images at the same time.

Different from the approach presented in \cite{she2019adjustment}, this paper estimates the decentering vector $\q v_{\mathrm{off}} \in \mathcal{R}^3$ only from underwater imagery. 
Since there is no in-air photo (like in \cite{she2019adjustment}) to provide accurate pose estimation, the poses $\mq P_i$ of the chessboard images have to be estimated jointly with $\mq v_{\mathrm{off}}$. 
This results in $6m+3$ parameters $\Theta = (\q v_{\mathrm{off}}, \mq P_1, \mq P_2, \dots, \mq P_m) \trans $, where $m$ defines the number of chessboard images
and the estimation relies on $2 \cdot m \cdot \# corners $ measurements. Assume Euclidean coordinates: $\q X = (X , Y, 0)\trans$, and the measured $i^{th}$ corner in $j^{th}$ image is $\q x_{i}^{j}$. 
Then, we optimize the energy given by:

\begin{equation}
E(\Theta) =\sum_{i \in {\Omega}} \sum_{j \in {\Phi}} \Vert \q X^{j}_{i} - \hat{\q X}^{j}_{i}( \q x^{j}_{i} , \q v_{\mathrm{off}},  \mq P^{j}) \Vert ^2
\end{equation}
Herein, $\hat{\q X}$ is the 3D point on the chessboard plane back-projected from $\q x$. Instead of minimizing the re-projection error, we back-project the corners detected in the images to the 3D chessboard, and the sum of squared differences is minimized inside the chessboard plane, since this is computationally much more efficient than doing the iterative forward projection within each optimization step.

\section{Evaluation}

To validate the geometrical insights into the decentered dome and to evaluate the proposed decentering calibration approach, we have conducted three different types of experiments.
The first type of experiment was performed using synthetic data, where we numerically simulated projections of 3D points by a decentered dome port camera system with perfectly known ground truth and noise models.
This helps to understand performance and to look at the magnitude of effects.
To take into account also effects as they occur in real images (corner detection problems, reflections, field of view issues and other physical limitations) we have also conducted real-world experiments with a deep sea dome port in a test tank.
However, here evaluation becomes very indirect, as it is very hard to obtain ground truth information (in particular, real decentering), and experiments are sensitive to deformation of the tank due to the weight of the water, calibration uncertainties, inaccurate physical measurements of distances and many other effects.
While all this will also occur in complex systems and real world applications, we think it is nevertheless important to isolate and understand the refraction effects.
We have therefore put substantial effort into another step of evaluation that employs an open source ray-tracing toolbox \cite{blender} to faithfully render images with ground truth settings for evaluation.
In particular, we modeled a real deep sea dome port including all radiuses and materials in a virtual copy of our real water pool as realistic as possible and we have verified that we can accurately reproduce images taken by the real system.
In this setup, we can control all {\em physical} parameters and understand effects in valid experiments.

\subsection{Synthetic Experiments}

\begin{figure*}[t]
	\begin{center}
		\subfloat[Center backward]{\includegraphics[width=0.3\textwidth]{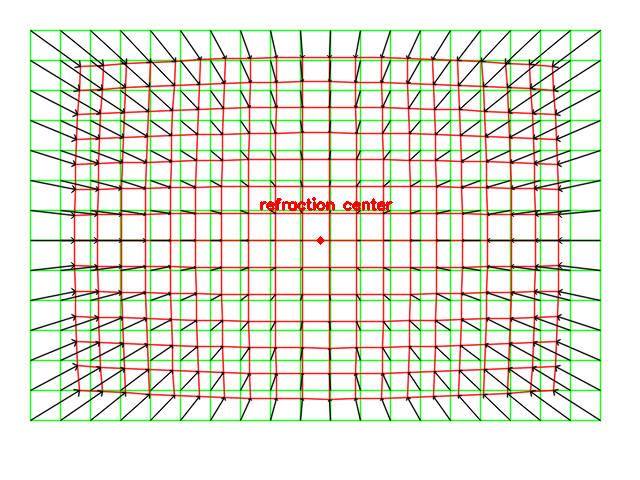}}
		\hspace{2ex}
		\subfloat[Left backward]{\includegraphics[width=0.3\textwidth]{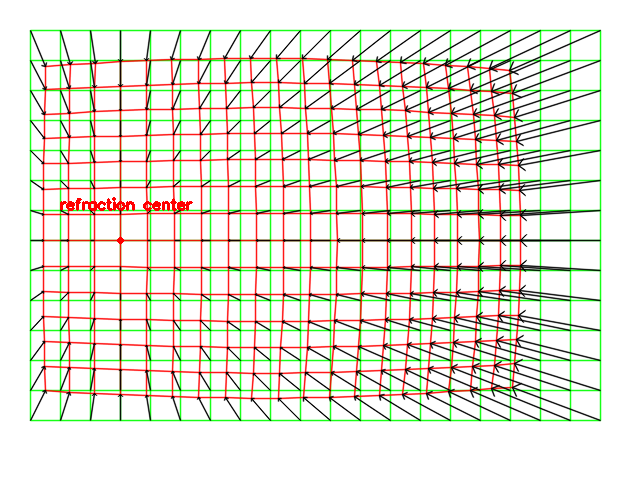}}
		\hspace{2ex}
		\subfloat[Right backward]{\includegraphics[width=0.3\textwidth]{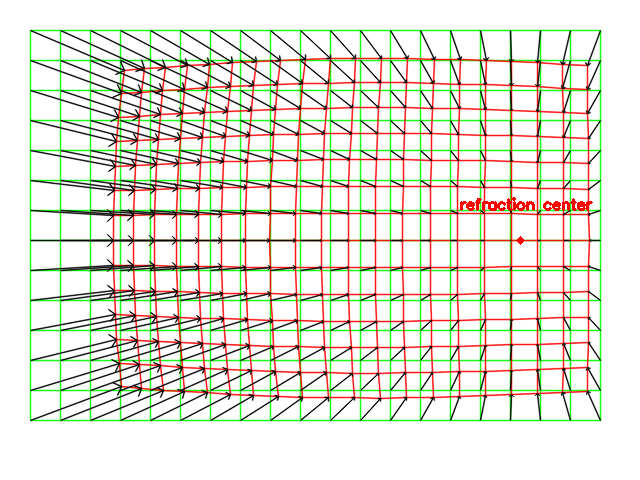}}\\
		\subfloat[Center forward]{\includegraphics[width=0.3\textwidth]{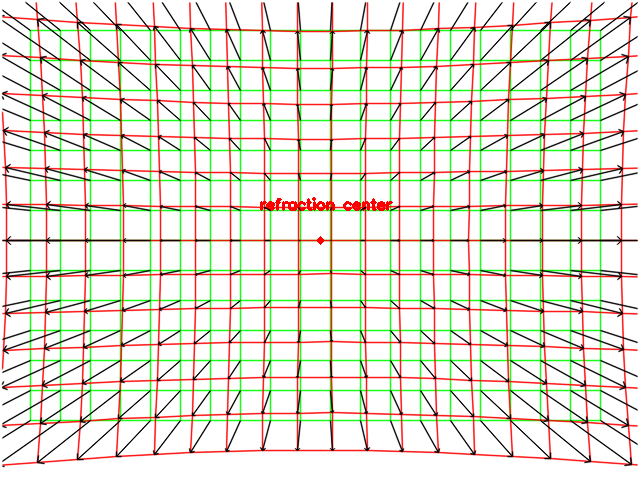}}
		\hspace{2ex}
		\subfloat[Left forward]{\includegraphics[width=0.3\textwidth]{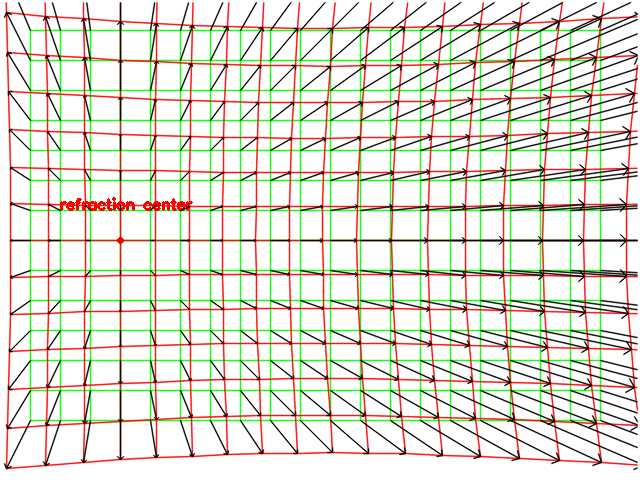}}
		\hspace{2ex}
		\subfloat[Right forward]{\includegraphics[width=0.3\textwidth]{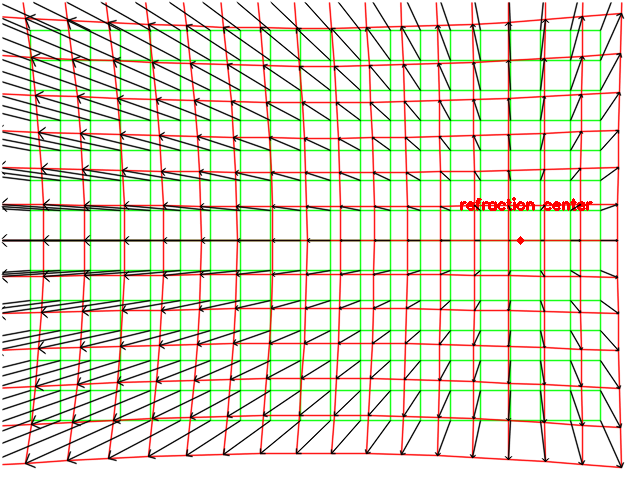}}
	\end{center}
	\caption{Refraction displacement fields created by simulation. If a camera is decentered in a dome port, the red patterns show how the grids are refracted depending on the direction of the decentering. The black arrows point from the un-refracted grids towards the refracted grids. Top: the camera is decentered backward (along the viewing direction of the camera), the grids are refracted towards the refraction center. Bottom: the grids are refracted outwards from the refraction center because the camera is decentered forward.}
	\label{fig:displacement_fields}
\end{figure*}

First, to validate the proposed decentered dome geometry, we simulated the geometric projection of a dome port camera system with known decenterings.
Here, the refraction displacement fields in the images were generated to visualize the theorems for the decentered dome geometry as shown in Fig. \ref{fig:displacement_fields}.
The dome center in the local camera coordinate system was directly projected onto the image as the refraction center.
In order to better visualize the displacement direction (inward vs. outward) with respect to the refraction center, the dome center was selected to be located in front of the camera center (Fig. \ref{fig:displacement_fields} , Top) and behind the camera center (Fig. \ref{fig:displacement_fields}, Bottom) respectively.
The original patterns shown as green grids were created as if a chessboard was observed in air (not refracted).
Then, the pattern was back-projected to the 3D space at 1$m$ depth to obtain the pattern coordinates.
Afterwards, we projected the pattern into the image considering the refraction effect, as if it was observed underwater, which is shown as red grids.
As can be seen, the grids are refracted either towards or away from the refraction center depending on the direction of the decentering.
In addition, the lines joining the un-refracted points and refracted points all intersect at the refraction center as exploited in the collinear constraint proposed in theorem \ref{theorem:line}.
Note that the refraction distortion pattern looks somewhat similar to the barrel/pincushion distortion, however, the refraction distortion center varies with the decentering direction, and the magnitude of displacement varies with the scene distance.
Although it is not explicitly marked in the figure, it is also possible to identify the \textit{iso-refraction-angle curves}.
\begin{figure*}[t]
	\begin{center}
		\subfloat[Refraction center estimation]{\includegraphics[width=0.85\linewidth]{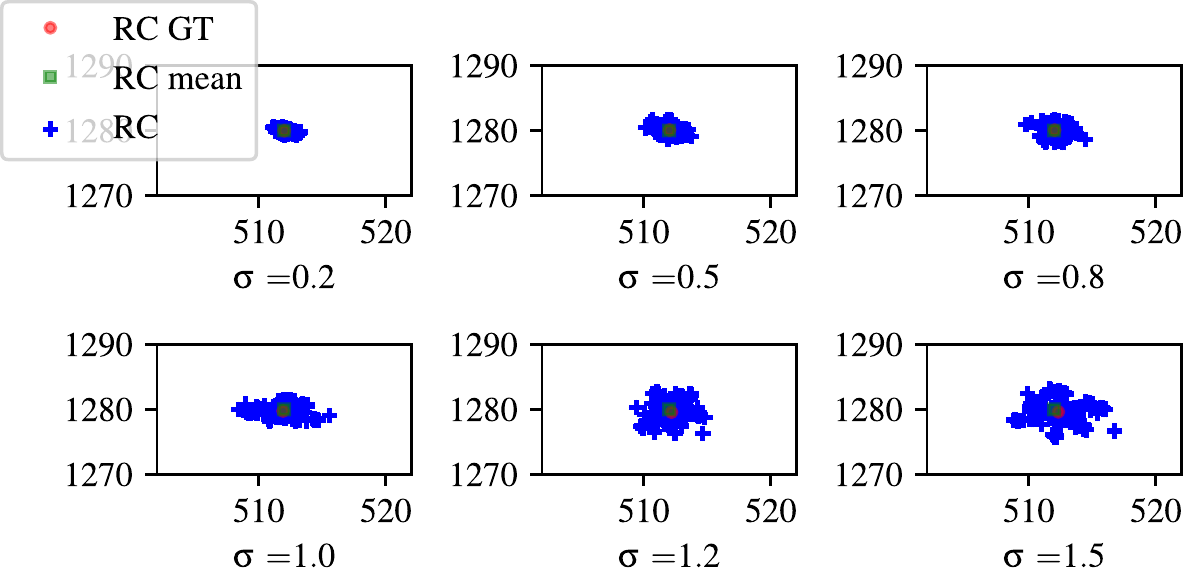}}
		\quad\subfloat[Refraction axis direction estimation]{\includegraphics[width=0.66\linewidth]{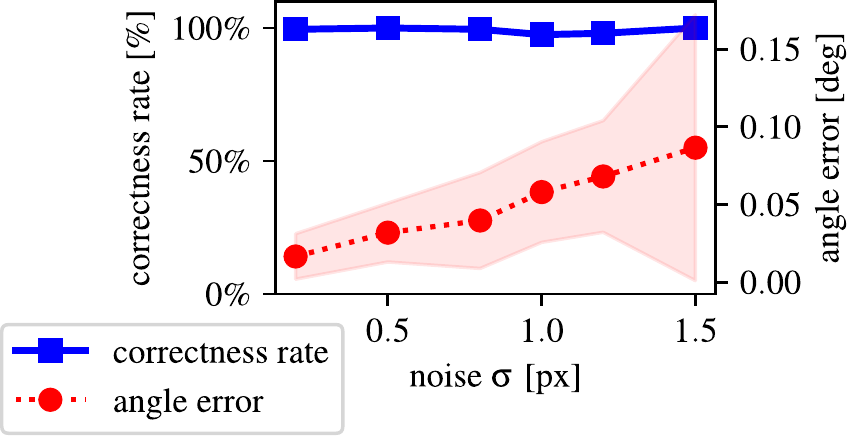}}
	\end{center}	
	\caption{Synthetic evaluation of refraction center estimation and the distinction between positive and negative decentering along the axis using synthetic data. (a), each plot shows the distribution of the estimated refraction center from 200 trials in total for each noise level. (b), the red curve shows the average and standard deviation of the angle error between the estimated refraction axis direction and the ground truth direction; the blue curve shows the correctness rate of detecting the sign of the decentering.}
	\label{fig:res_simu_eval_offdir}
\end{figure*}

Next, to evaluate the effectiveness and stability of the proposed calibration approach, synthetic experiments were conducted.
We first present the simulation experiments for the refraction center estimation and determine the sign of the decentering from noisy input data.
The calibration target was simulated as a $7 \times 8$ planar chessboard with $0.05m \times 0.05m$ square size.
The virtual camera had a resolution of $2048 \times 1536$ pixels and a field of view of $90^{\circ}$.
The dome port was simulated with a radius of $50mm$ and a thickness of $7mm$ (which mimics real domes for 6km ocean depth).
Since the refraction center estimation and the decentering sign determination require a single image of the chessboard as input, we generated 10 random poses of the chessboard to be photographed, and it is guaranteed that the chessboard is at a pose that is inside the field of view of the camera.
In order to evaluate the stability of the approach against small decentering offsets, which the refraction effect in the image is expected to be weaker in this scenario, we explicitly set the decentering vector as $\q v_{\mathrm{off}} = (-0.001, 0.001, 0.002)\trans$ (all numbers in meter, if not stated otherwise) which is an extremely small decentering in the millimeter range.
The 3D calibration targets are projected to the images using the iterative thick dome projection and zero-mean Gaussian noise with $\sigma = 0.2,0.5,0.8,1.0,1.2,1.5$ pixels is added to each projected corner.
We performed 20 trials for each of the poses and recorded the estimated refraction center and the sign of the decentering, then we measured the angle error between the estimated direction and the ground truth direction and computed the correctness rate of the sign of the axis direction.
The results are shown in Fig. \ref{fig:res_simu_eval_offdir}, from which we can infer that the refraction center can be successfully estimated despite a high noise level.
In addition, the sign of the decentering can also be correctly determined by the convexity test.

\begin{figure*}[t]
	\begin{center}
		\includegraphics[width=\linewidth]{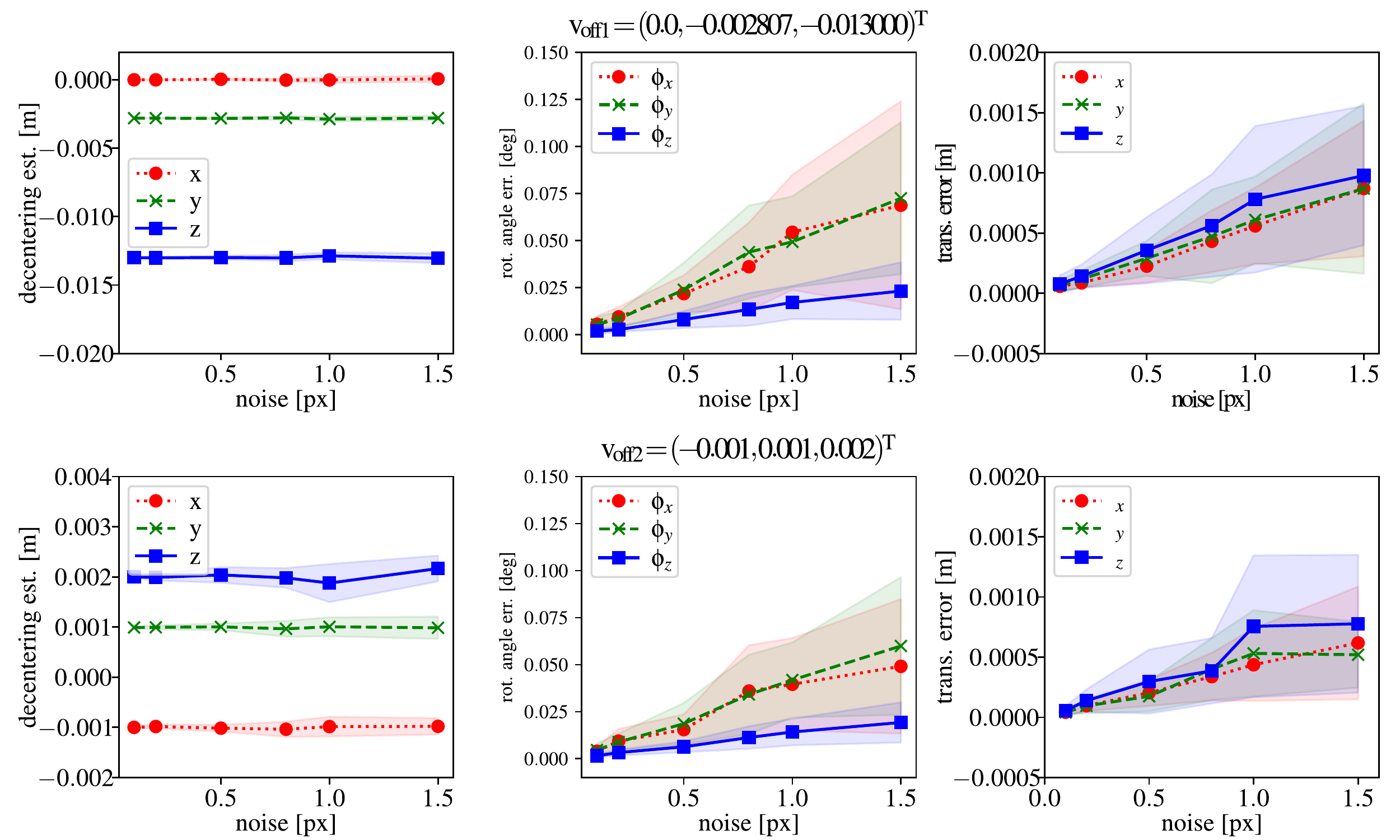}
	\end{center}
	\caption{Synthetic evaluation of decentering vector calibration in presence of a growing noise level. The plots from left to right: average and standard deviation of the estimated decentering vector, rotation angle error and translation error of the estimated camera poses.}
	\label{fig:res_simu_eval_offset_calib}
\end{figure*}

Then, we present the simulation experiments of the decentering vector calibration.
Similar to the above experiment, the dome port system, the 3D chessboard and corner observations at different poses were simulated.
The only difference is that the calibration approach is entirely based on multiple underwater images.
We set the decentering vector as $\q v_{\mathrm{off}1} = (0.0,-0.002807,-0.013)\trans m$ and $\q v_{\mathrm{off}2} = (-0.001,0.001,0.002)\trans m$.
The first set aims to validate the effectiveness of the calibration approach and the second set aims to evaluate the stability in the case of small decentering offsets.
The evaluation was repeated 20 times on each set, and 10 images were taken per trial.
The experiment results can be found in Fig. \ref{fig:res_simu_eval_offset_calib}, where the rotation angle error and the translation error show the differences between the estimated camera poses and the ground truth poses.
As can be seen, the unknown parameters were computed with high accuracy even though a considerable amount of noise was added to the measurement.
Note that the uncertainties of the estimated parameters in the second set are higher than the first set since the refraction effect is much weaker in the images, but the results can still be considered accurate.

\subsection{Experiments on Rendered Images}

\begin{figure}[t]
	\begin{center}
		\subfloat[Real experimental setup]
		{
			\includegraphics[width=0.25\textwidth]{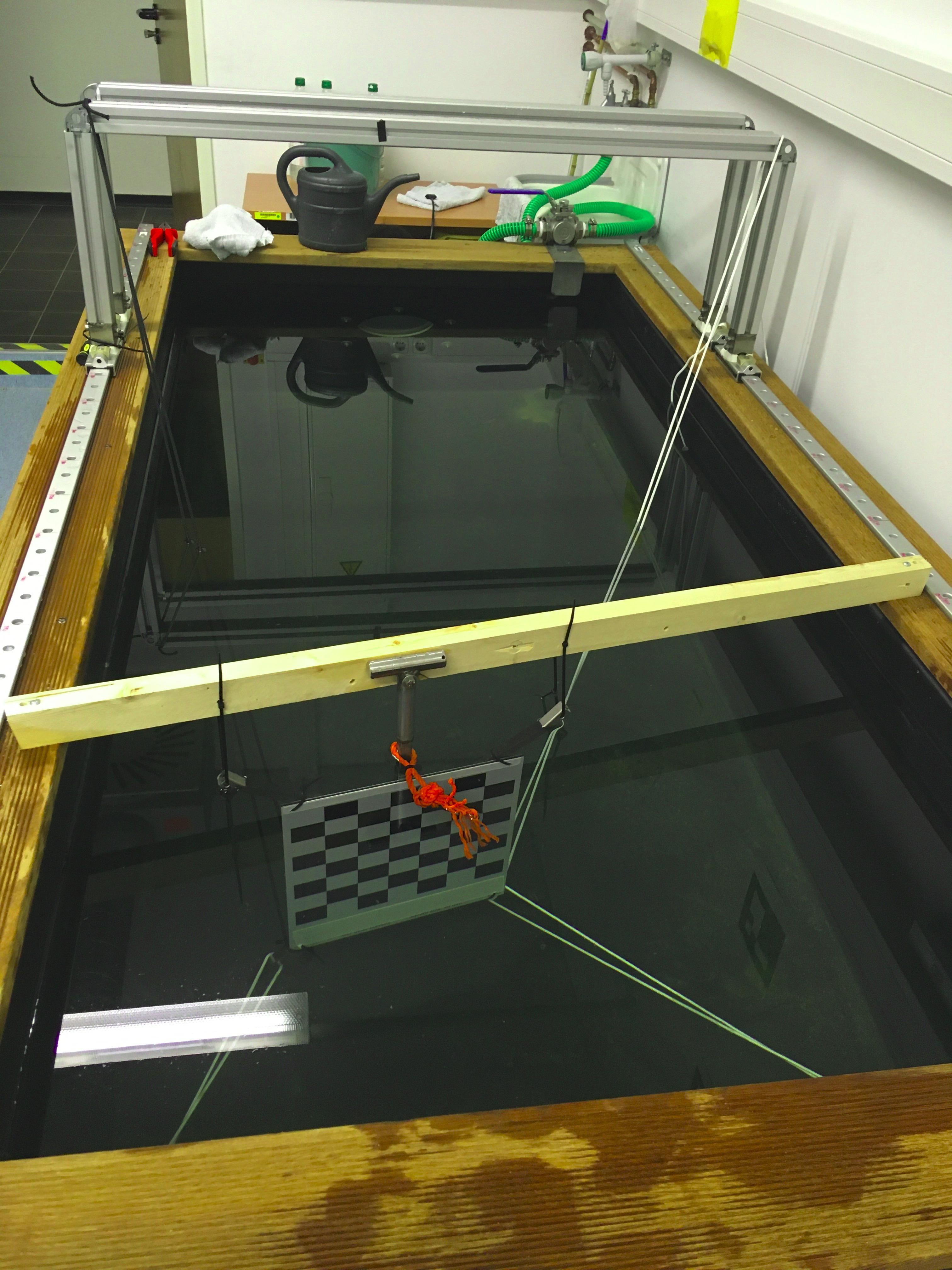}
			\includegraphics[width=0.25\textwidth]{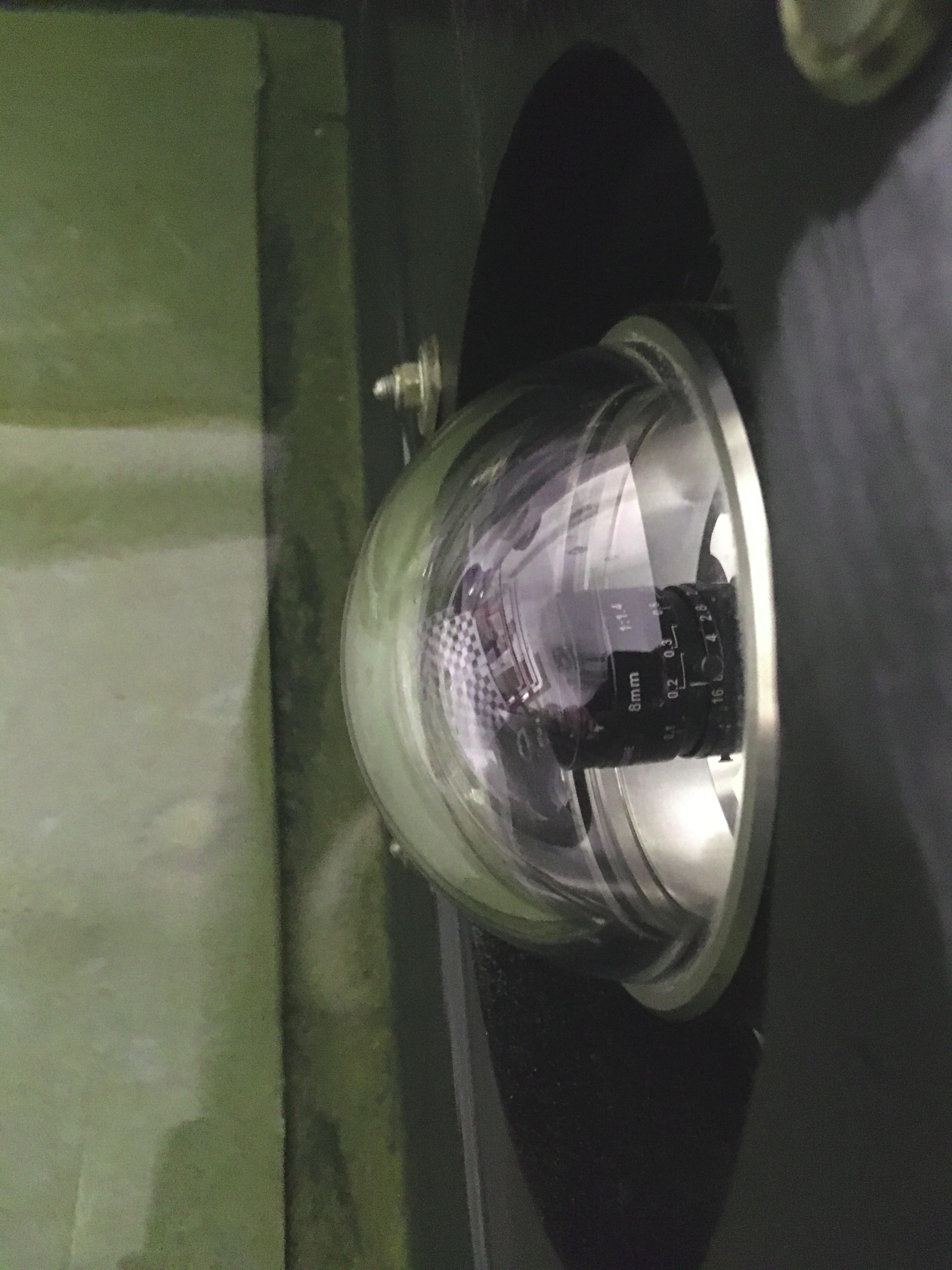}
			\includegraphics[width=0.45\textwidth]{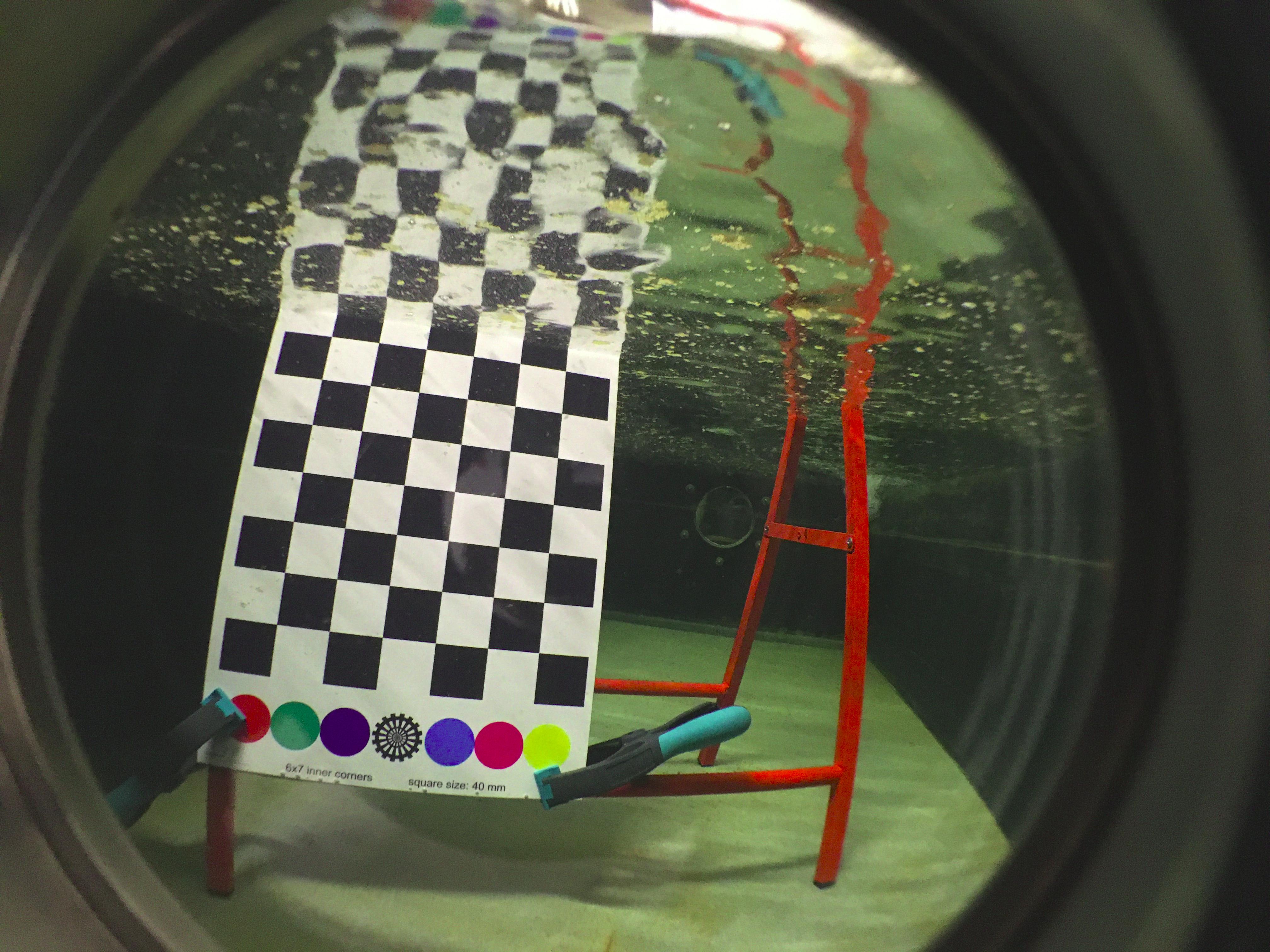}
		}\newline
		\subfloat[Blender experimental setup]
		{
			\includegraphics[width=0.3\textwidth]{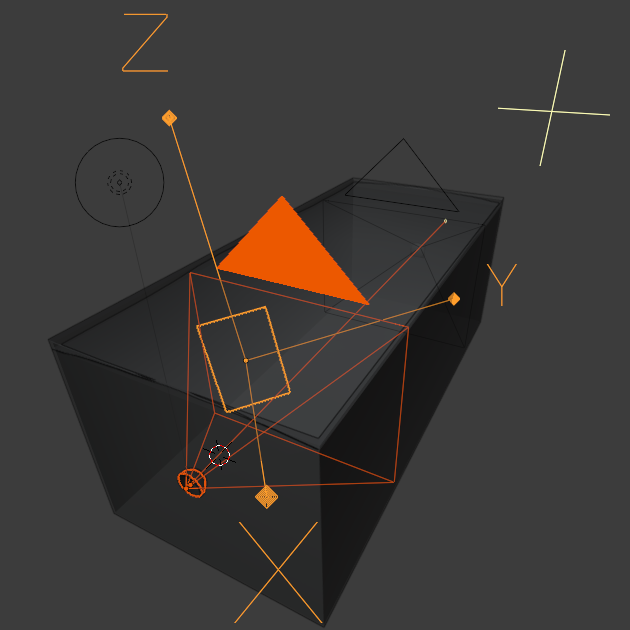}
			\includegraphics[width=0.3\textwidth]{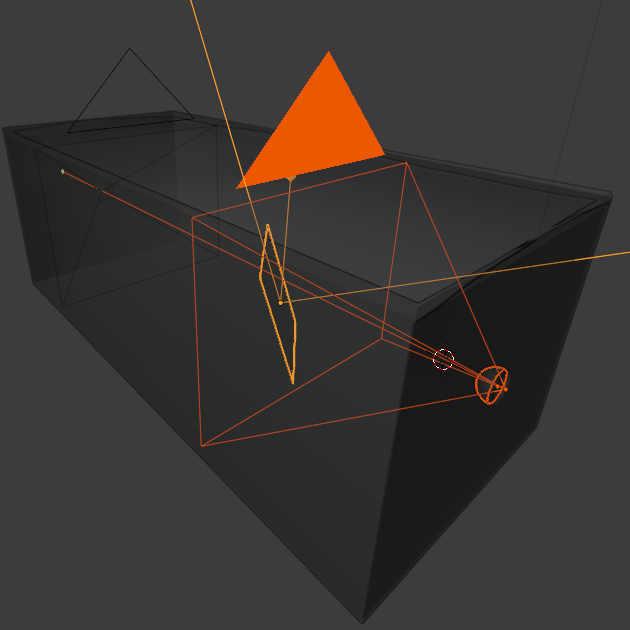}
			\includegraphics[width=0.3\textwidth]{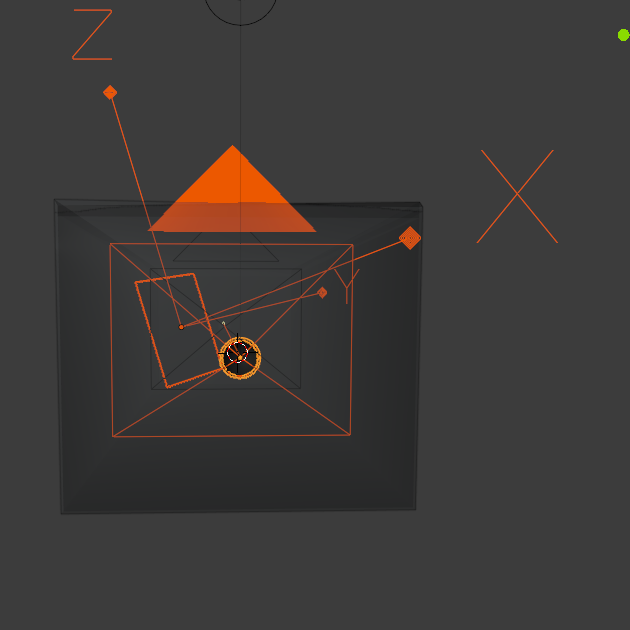}
		}
	\end{center}
	\caption{Evaluation setup in real world and in Blender. The scene in Blender mimics the setup in the real world. The black cube represents the water pool, and the orange plane with a rectangular shape represents the planar checkerboard. The dome is modeled as concentric hemispheres and it is attached to the sidewall of the pool.}
	\label{fig:blender_setup}
\end{figure}

\begin{figure*}[t]
	\begin{center}
		\subfloat[Half-air/half-underwater simulation, showing the different behaviours of the air-glass, glass-air interface as well as glass-water interface in different decentering settings. Images obtained with Blender \citep{blender}]
		{
			\includegraphics[width=0.3\textwidth]{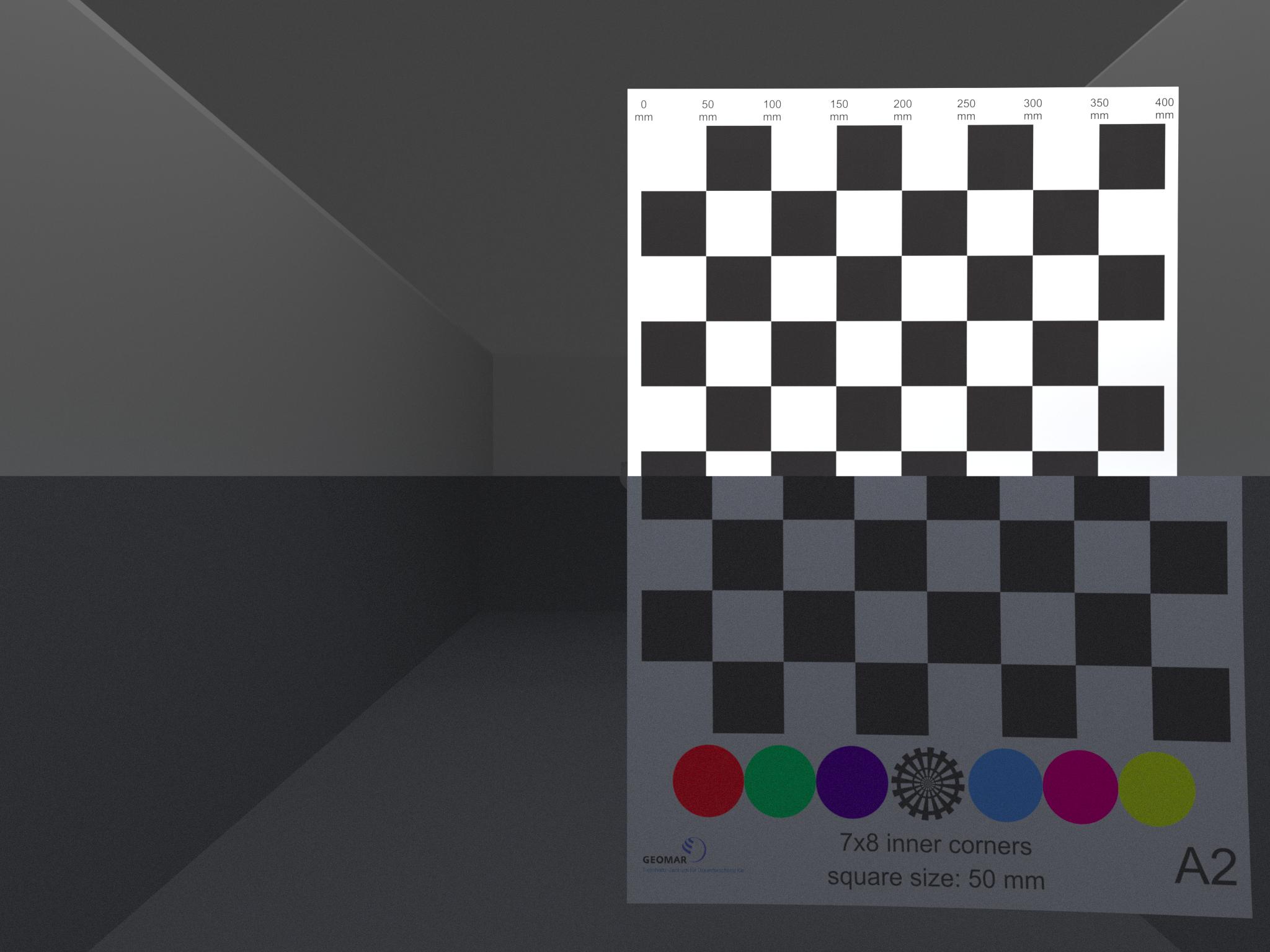}
			\includegraphics[width=0.3\textwidth]{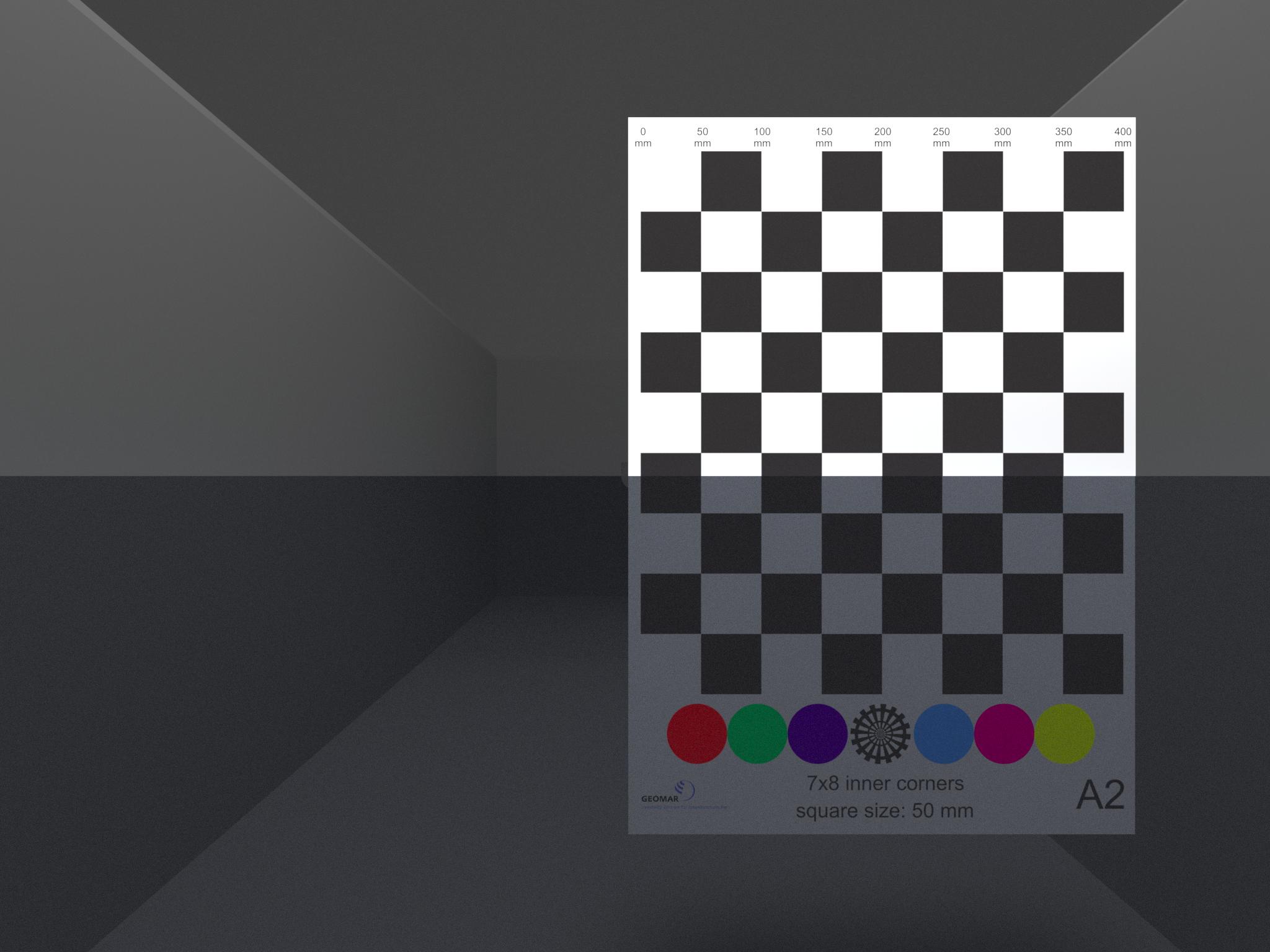}
			\includegraphics[width=0.3\textwidth]{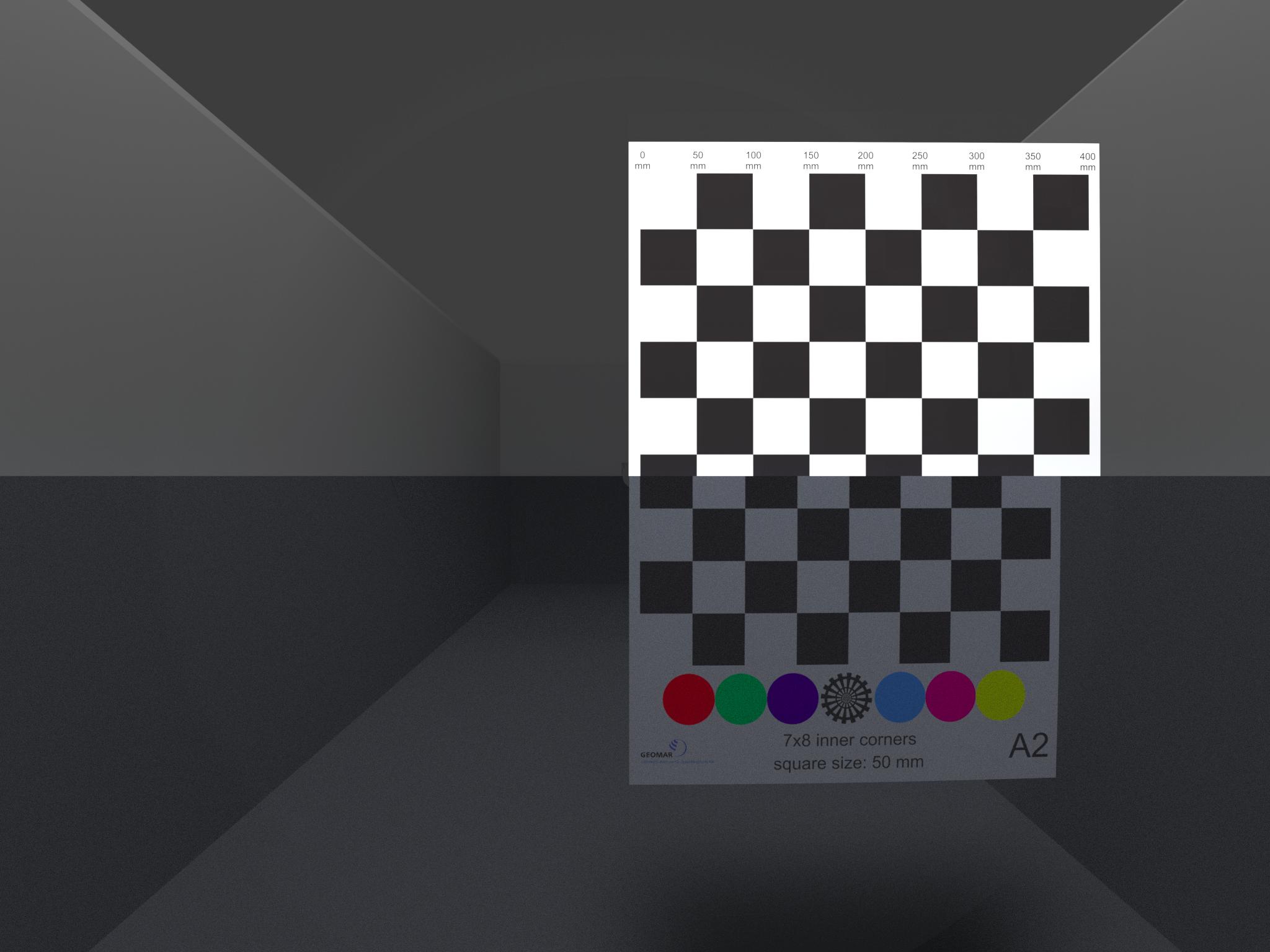}
			\label{fig:half_airwater_offset_impact}
		}\newline
		\subfloat[Refraction pattern shown as the angle between the incident light ray and the surface normal of the inner dome sphere plus the angle between the normal of the outer dome sphere and the outgoing ray.
		The \textbf{lower row} shows the same setting from another perspective, addionally inticating $\q v_{\mathrm{off}}$ as a red arrow. Images obtained with Mitsuba 2 \citep{NimierDavidVicini2019Mitsuba2}]
		{
     		\includegraphics[width=0.95\textwidth]{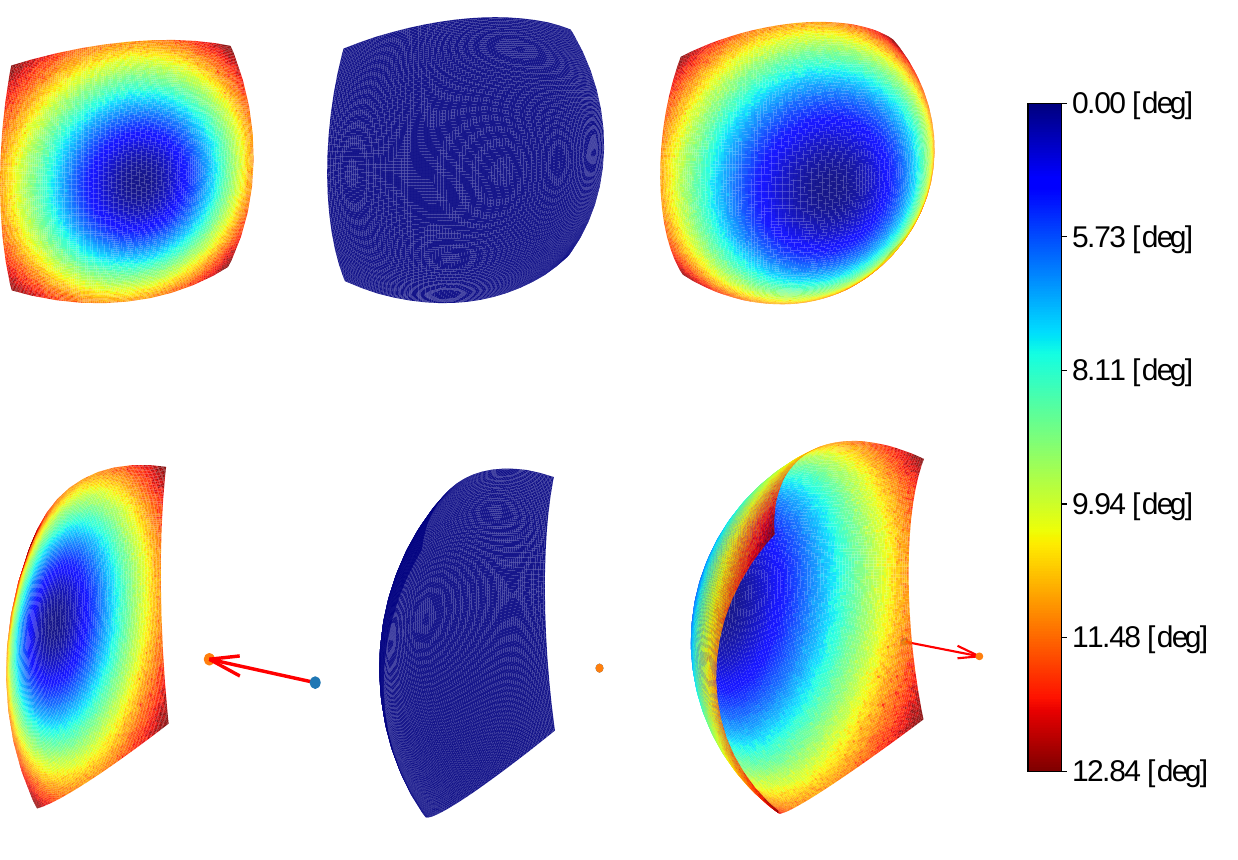}
						\label{fig:dome_refrac_angles}	
		}
		
%
%

	\end{center}
	\caption{Validation of the physically correct ray-tracing image renderer in different decentering settings. Left: Forward-decentering. Center: Perfectly centered. Right: Backward-decentering.}
\end{figure*}

\begin{figure}
	\begin{center}
		\subfloat[Real]{
			\begin{minipage}{0.21\textwidth}
				\includegraphics[width=1\textwidth]{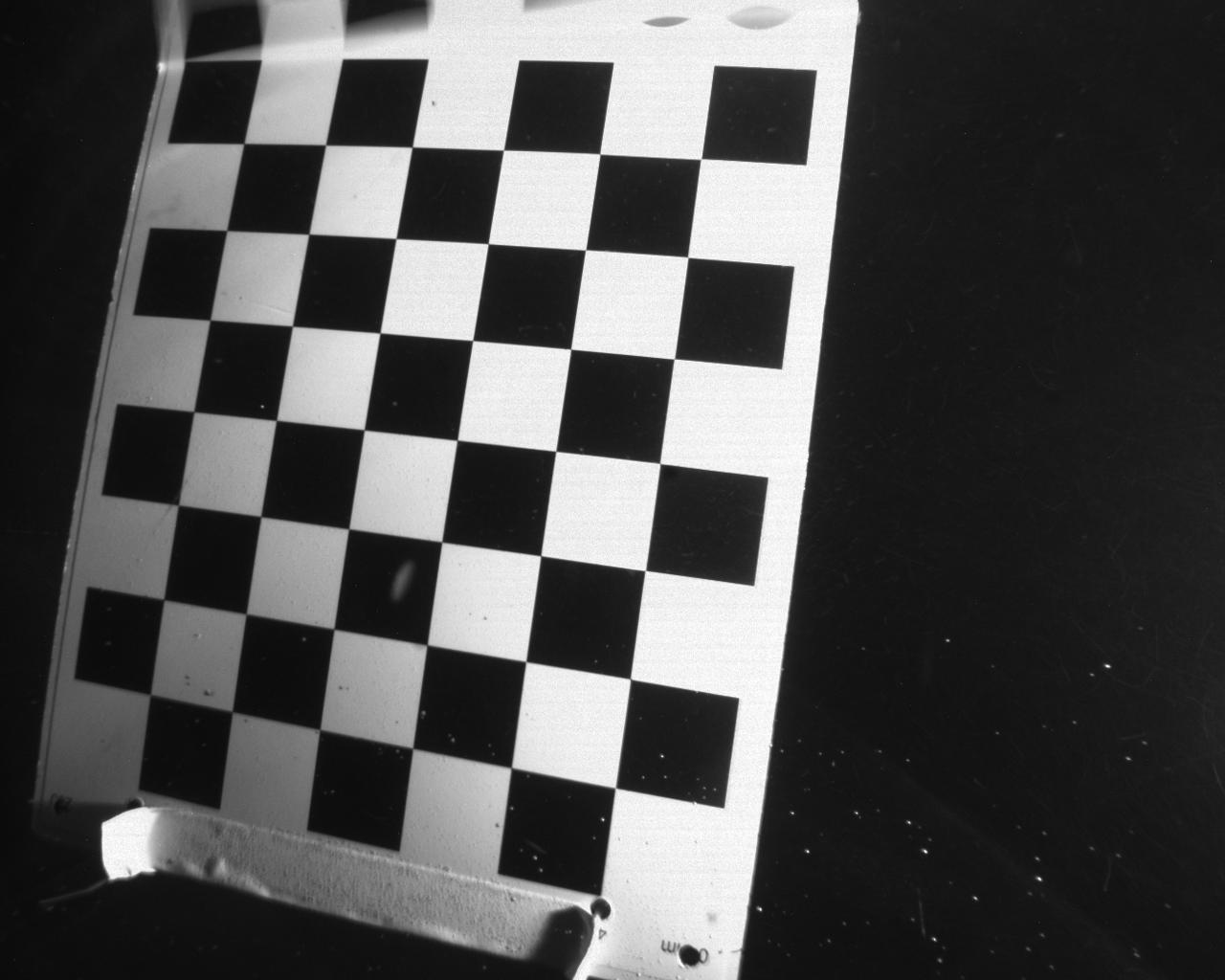}\vspace{0.8ex}
				\includegraphics[width=1\textwidth]{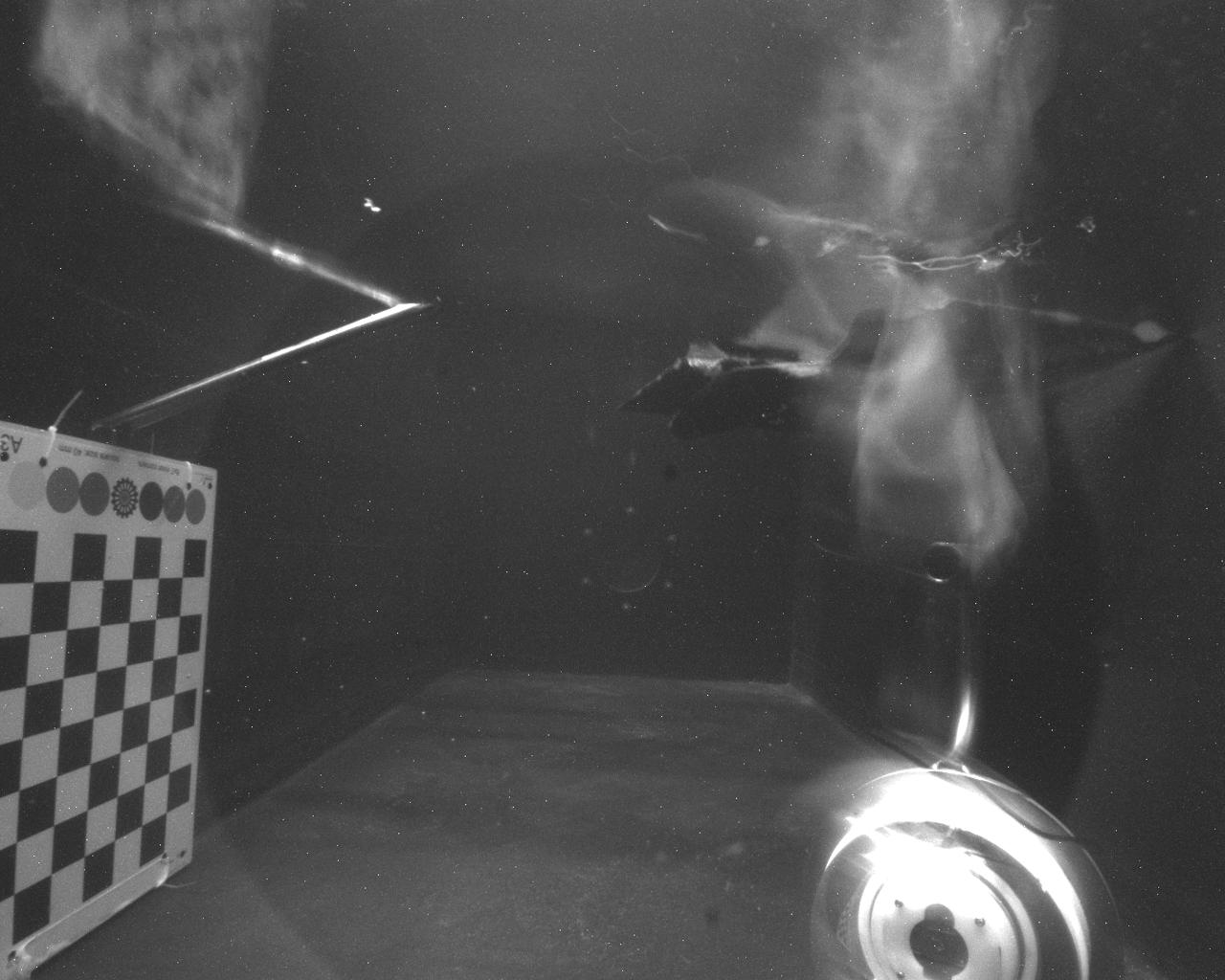}\vspace{0.8ex}
				\includegraphics[width=1\textwidth]{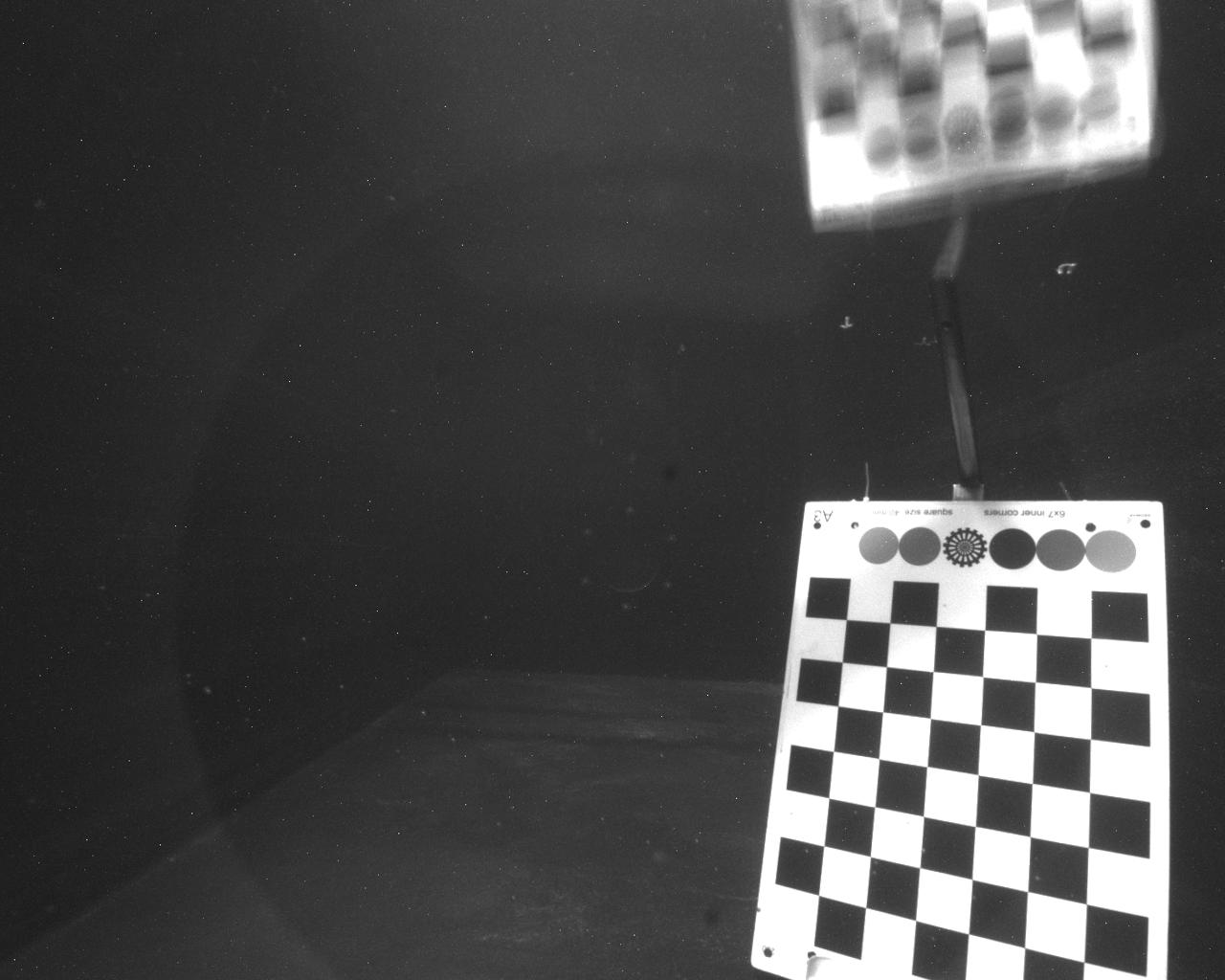}
		\end{minipage}}
		\subfloat[Blender]{
			\begin{minipage}{0.21\textwidth}
				\includegraphics[width=1\textwidth]{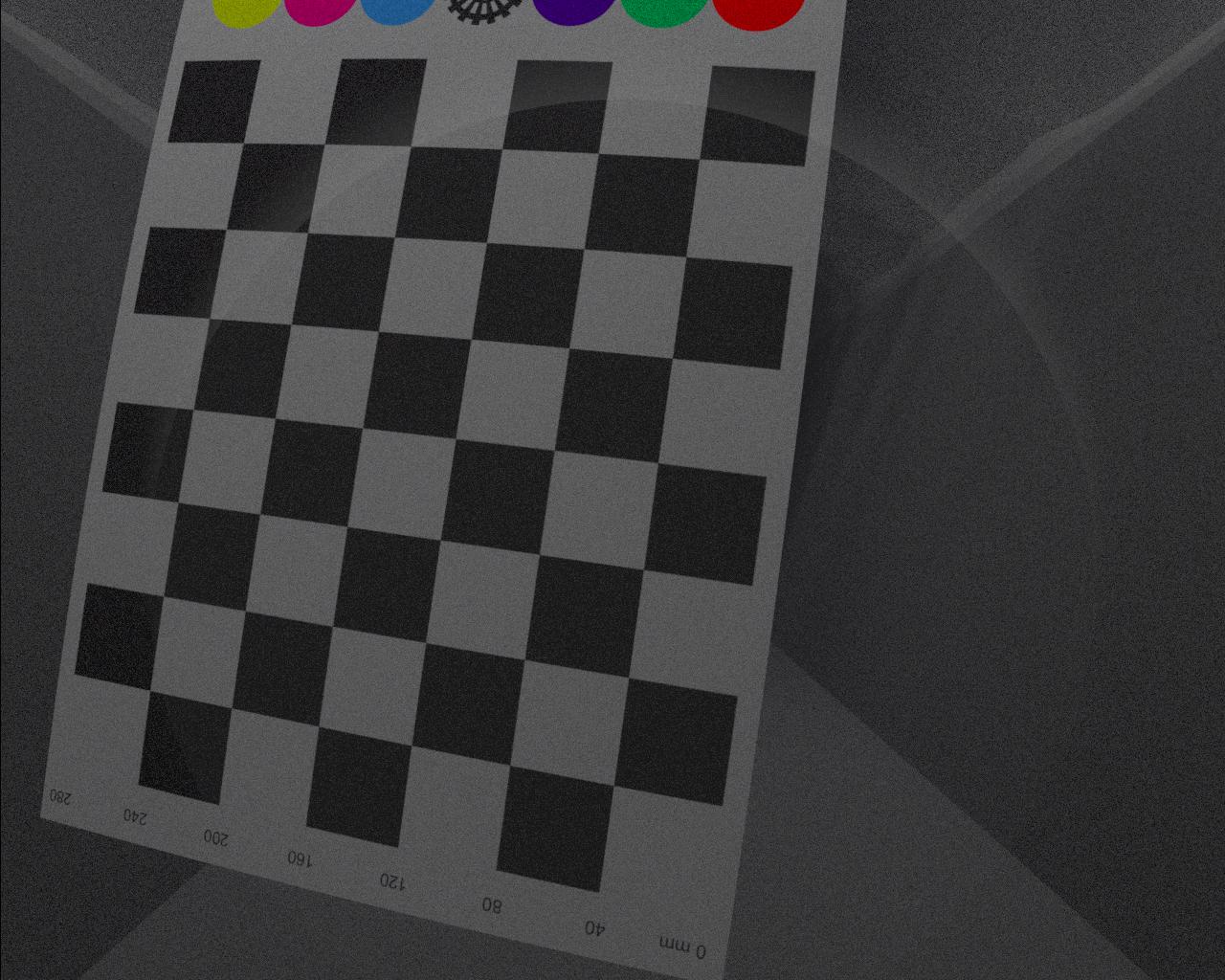}\vspace{0.8ex}
				\includegraphics[width=1\textwidth]{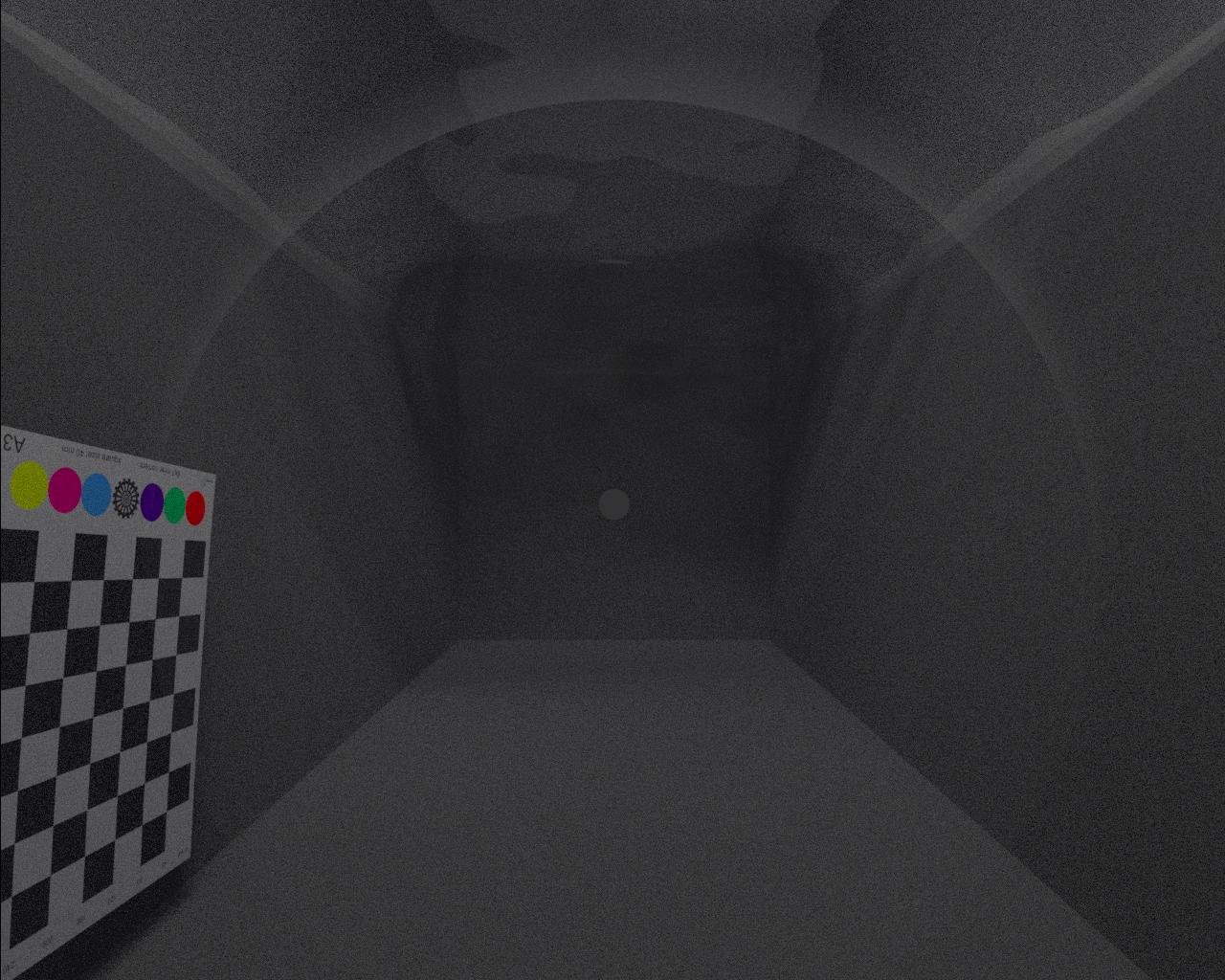}\vspace{0.8ex}
				\includegraphics[width=1\textwidth]{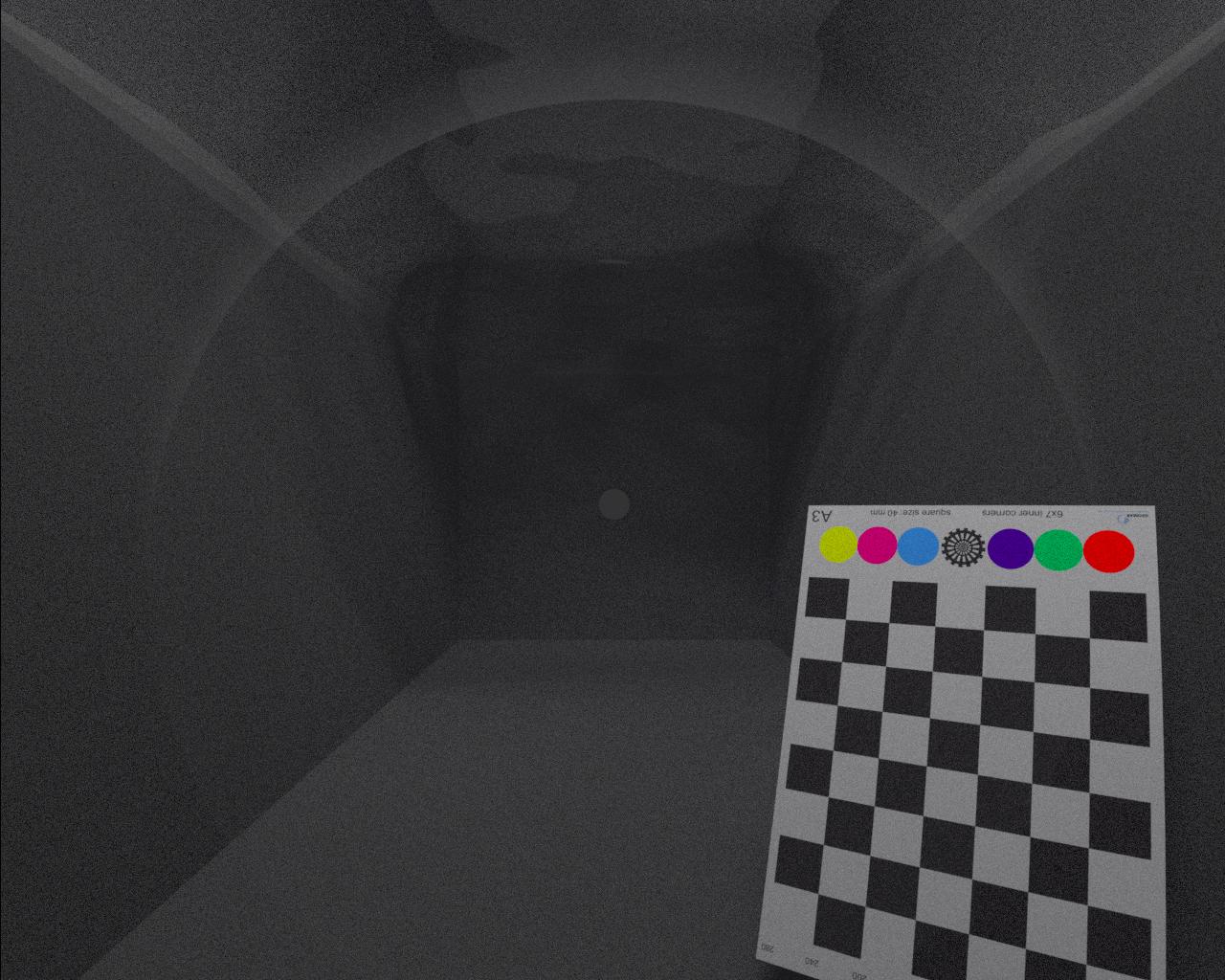}
		\end{minipage}}
		\subfloat[Ours]{
			\begin{minipage}{0.21\textwidth}
				\includegraphics[width=1\textwidth]{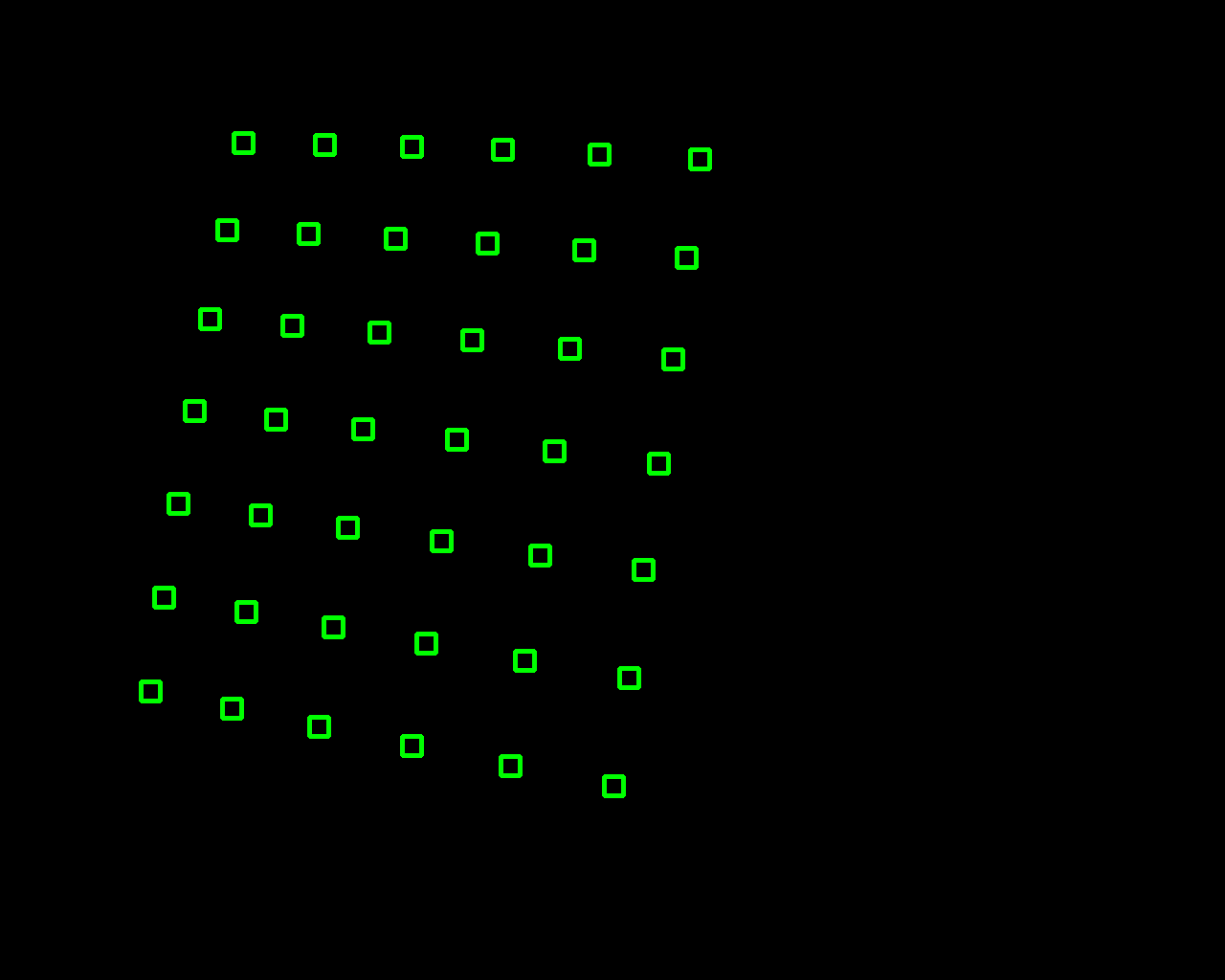}\vspace{0.8ex}
				\includegraphics[width=1\textwidth]{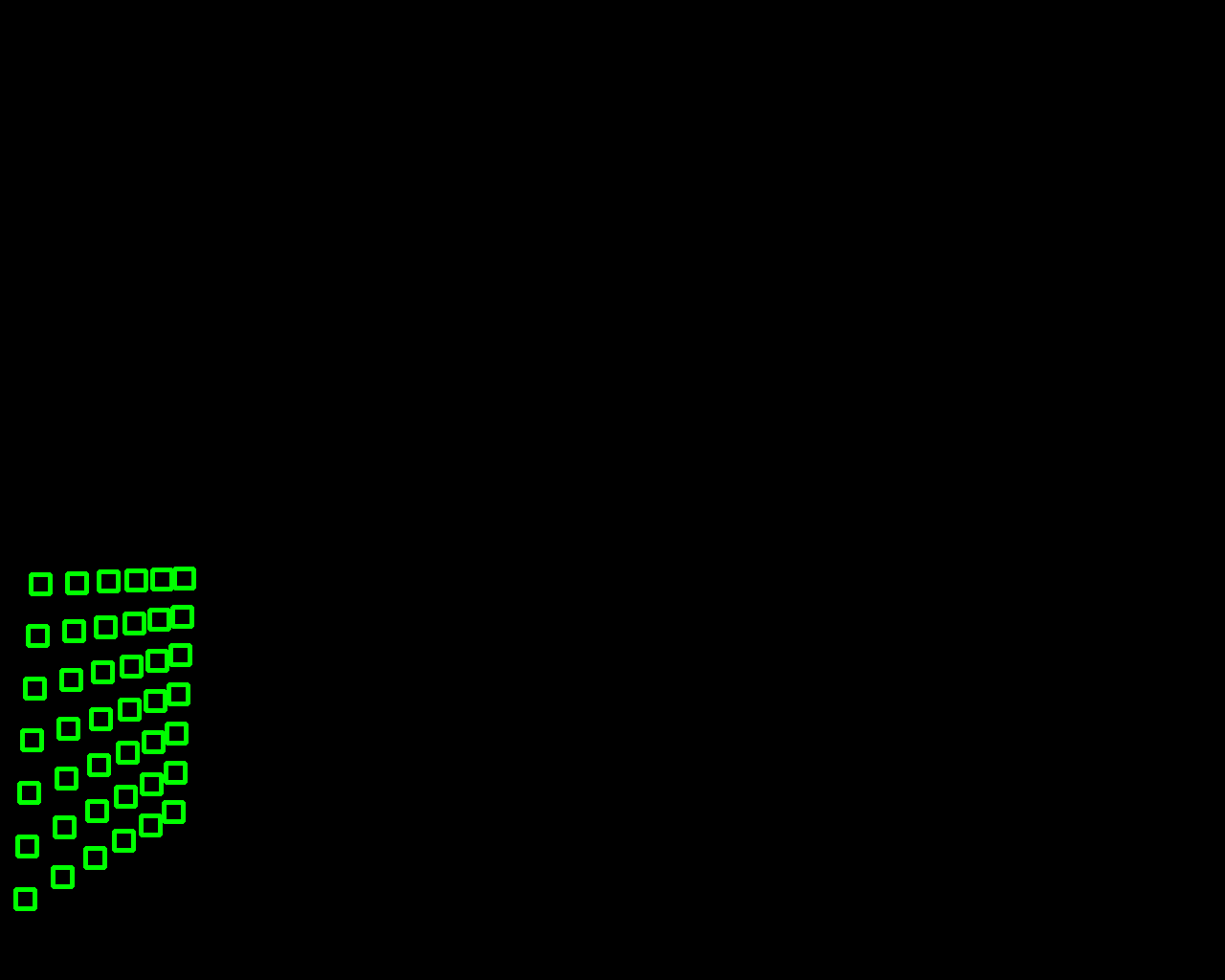}\vspace{0.8ex}
				\includegraphics[width=1\textwidth]{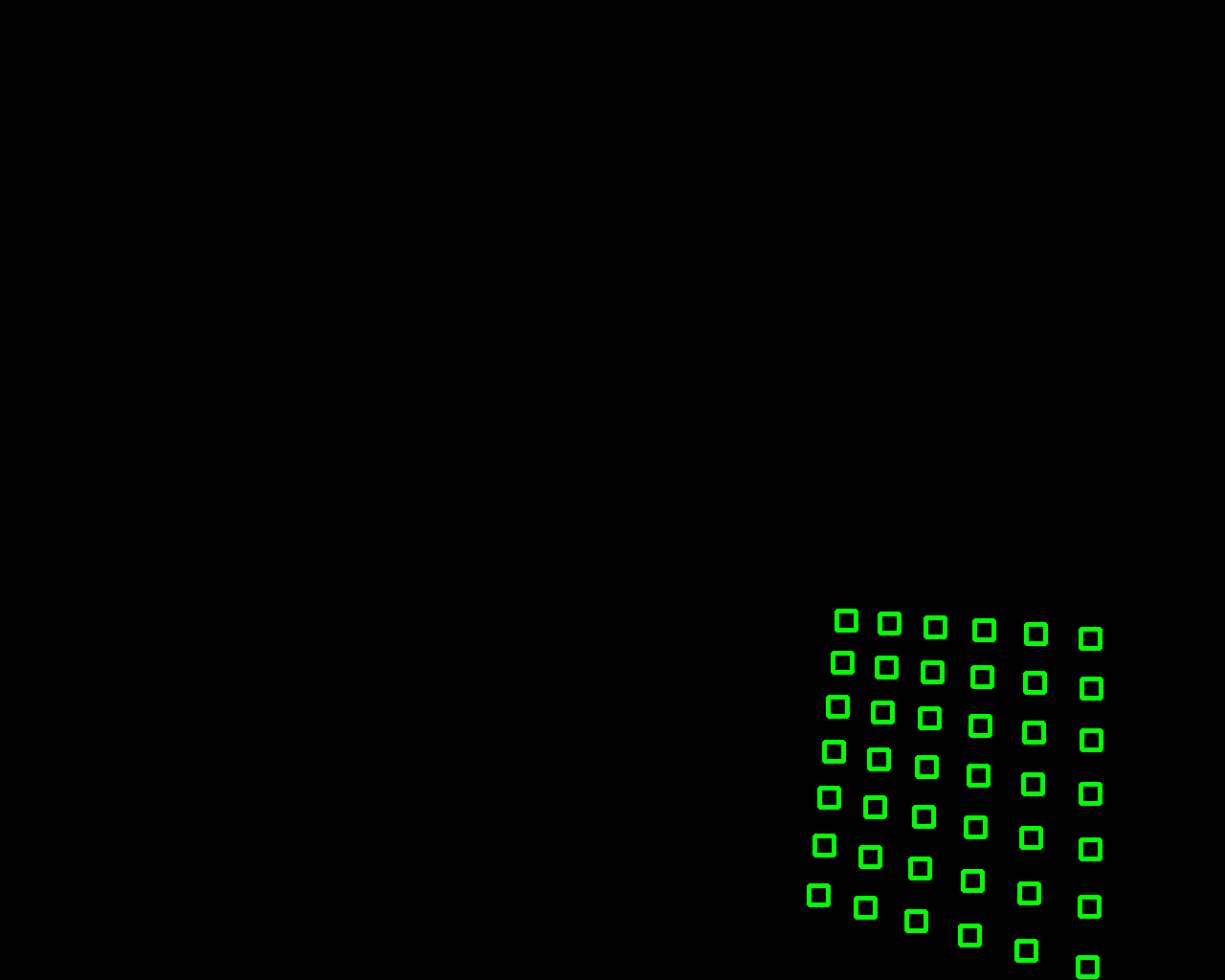}
		\end{minipage}}
	\end{center}
	\caption{Validation of the imaging formation model of the dome port camera system. From left to right: Real chessboard images, rendered chessboard images by Blender, synthetic chessboard corners by our thick dome port projection model.}
	\label{fig:blender_simu_real_compare}
\end{figure}

To validate the imaging formation model of our dome port camera system and also to take imaging and corner detection effects into account in the evaluation, we employed the ray-tracing toolbox Blender\footnote{http://www.blender.org} to render underwater imagery with physically plausible ground truth parameters.
Blender is a physics-based ray-tracing software to create 3D worlds with realistic optical parameters for lighting and media, such as the index of refraction. 
It also allows setting virtual cameras to simulate capturing the scene into 2D images.
To validate our camera model, we have modeled the water pool and the dome port into Blender mimicking the same setup as it exists in reality in our lab.
This allowed us to identify discrepancies between real images and those from the simulation, and strongly supports that our simulated experiments are valid.
The setup for rendering is visualized in Fig. \ref{fig:blender_setup}.
The dome port was modeled as concentric hemispheres with 7mm thick interface material of borosilicate glass 3.3 (index of refraction =  1.473).
Then the ray-tracer will track back light rays starting from the camera sensor and compute the light paths when traveling through the glass dome, thus, the refraction effect will be simulated separately for each pixel in the image.
This is different from refraction simulation in rasterization-based renderers (e.g. \cite{Song2019}) that apply refraction effects to the image after the actual rendering process.
Instead, using ray-tracing, the refraction is handled naturally as can be seen in Fig. \ref{fig:half_airwater_offset_impact}, showing
a chessboard half underwater and half in the air.
When the camera is exactly centered in the dome (see Fig. \ref{fig:half_airwater_offset_impact}, Center), we can clearly see that there is no refraction in the image.
When pulling the camera away from the center, the lower part of the image is refracted inwards whereas it is refracted outwards when the camera is positioned in front of the dome center.
Fig. \ref{fig:dome_refrac_angles} shows quantitatively the angular change of the viewing rays due to refraction at the dome sphere for each decentering case.

To validate the image formation model, we projected chessboard corners onto the images and compared their position against the rendered images by Blender and the real photos.
The parameters for synthesizing were exactly the same as were used in rendering, among which, the camera intrinsics were obtained via calibrating the real camera in-air, and the poses of the chessboard where the real images were taken were computed via the decentering calibration.
The resulting images are shown in Fig. \ref{fig:blender_simu_real_compare}, as can be seen, the synthesized chessboard corners, the rendered images and the real images match very well.

\begin{table*}
	\begin{center}
		\caption{Evaluation results of decentering calibration on 8 sets of rendered images.}
		\label{table:blender_calib_result}       
		\begin{tabularx}{1.0\textwidth}{llll}
			\hline\noalign{\smallskip}
			set & $\q r$ GT & $\q r$ Est. & -- \\
			\noalign{\smallskip}\hline\noalign{\smallskip}
			\textcircled{1} & $(870.40, 921.60)\trans$    & $(884.85, 942.92)\trans$   &--\\
			\textcircled{2} & $(1024,768)\trans$          & $(1027.28 ,774.88)\trans$  &--\\
			\textcircled{3} & $(511.99, 1280.01)\trans$   & $(507.94 , 1351.56)\trans$ &--\\
			\textcircled{4} & $(1075.09, -1.2e+08)\trans$ & $(1194.02 ,286.07)\trans$  &--\\
			\textcircled{5} & $(1024, 1342.88)\trans$     & $(1064.23, 1291.11)\trans$ &--\\
			\textcircled{6} & $(1024 , 989.10)\trans$     & $(1104.97 ,906.42)\trans$  &--\\
			\textcircled{7} & $(1183.69, 927.69)\trans$   & $(1187.7, 906.37)\trans$   &--\\
			\textcircled{8} & $(652.80, 623.77)\trans$    & $(586.46, 565.22)\trans$   &--\\
			
			\hline\noalign{\smallskip}
			set & $\q v_{\mathrm{off}}$ GT $[mm]$ & $\q v_{\mathrm{off}}$ Est. $[mm]$ & ATE $[mm]$\\
			\hline\noalign{\smallskip}
			\textcircled{1} & $(-3.0, 3.0,20.0)\trans$    & $(-3.06, 3.05,20.39)\trans$ & 0.61 \\
			\textcircled{2} & $(0.0, 0.0, 30.0)\trans$    & $(0.001, 0.09, 30.49)\trans$ & 0.51\\
			\textcircled{3} & $(-1.0, 1.0, 2.0)\trans$    & $(-1.09, 0.95, 1.66)\trans$ & 0.78\\
			\textcircled{4} & $(0.0, 2.81, 0.0)\trans$    & $(-0.08, 2.81, -0.26)\trans$ & 0.48\\
			\textcircled{5} & $(0.0, 2.81, 5.0)\trans$    & $(-0.10,2.92, 4.77)\trans$ & 0.69\\
			\textcircled{6} & $(0.0, -2.81, -13.0)\trans$ & $(-0.16, -2.77 ,-13.48)\trans$ & 0.90\\
			\textcircled{7} & $(-2.81, -2.81, -18)\trans$ & $(-2.93, -2.75, -18.39)\trans$ & 0.50\\
			\textcircled{8} & $(0.42, 3.67,28.39)\trans$  & $(0.46,3.64, 28.43)\trans$ & 0.50\\
			\noalign{\smallskip}\hline
		\end{tabularx}
	\end{center}
\end{table*}

\begin{figure*}[t]
	\begin{center}
		\includegraphics[width=0.24\textwidth]{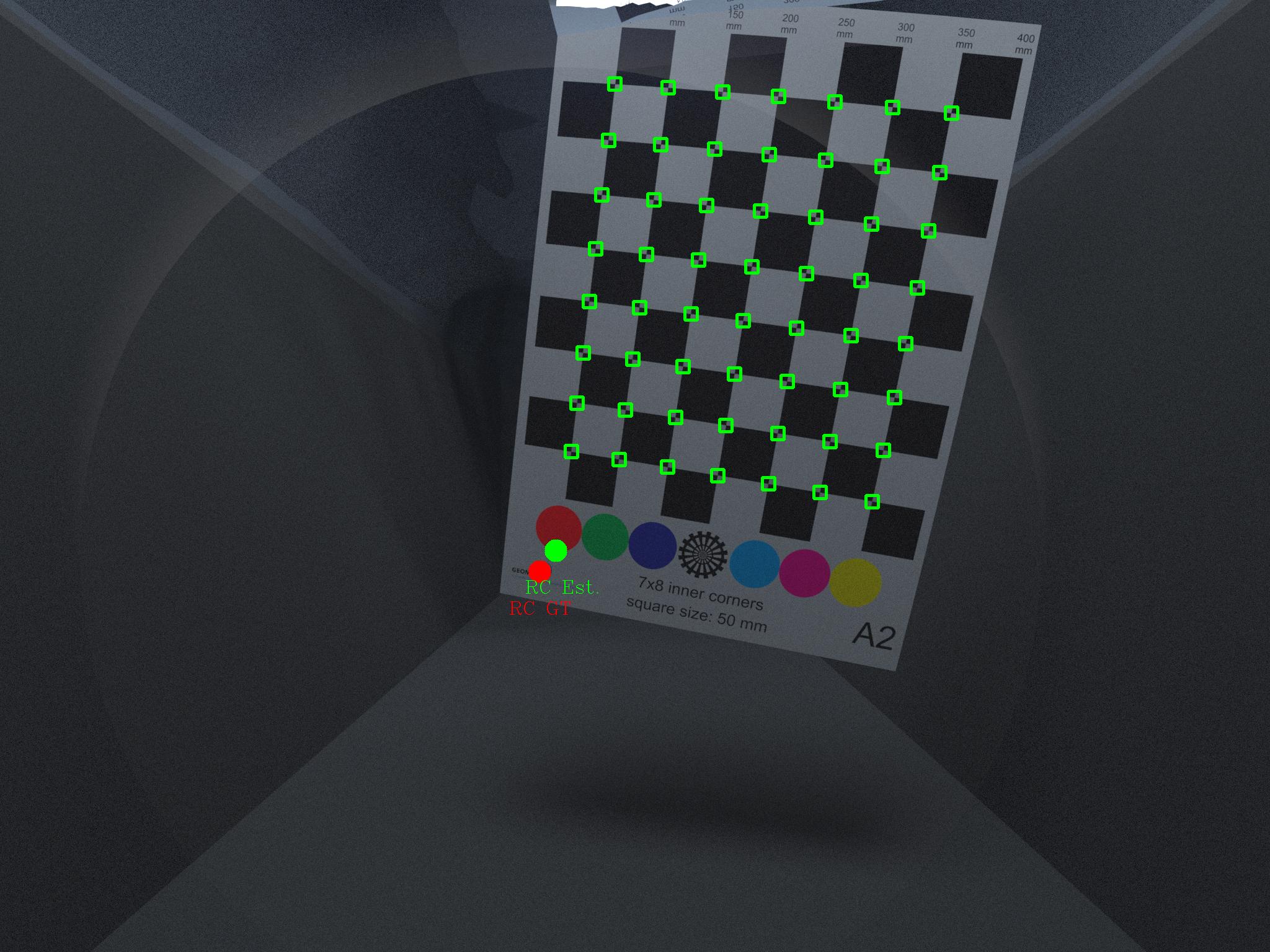}
		\includegraphics[width=0.24\textwidth]{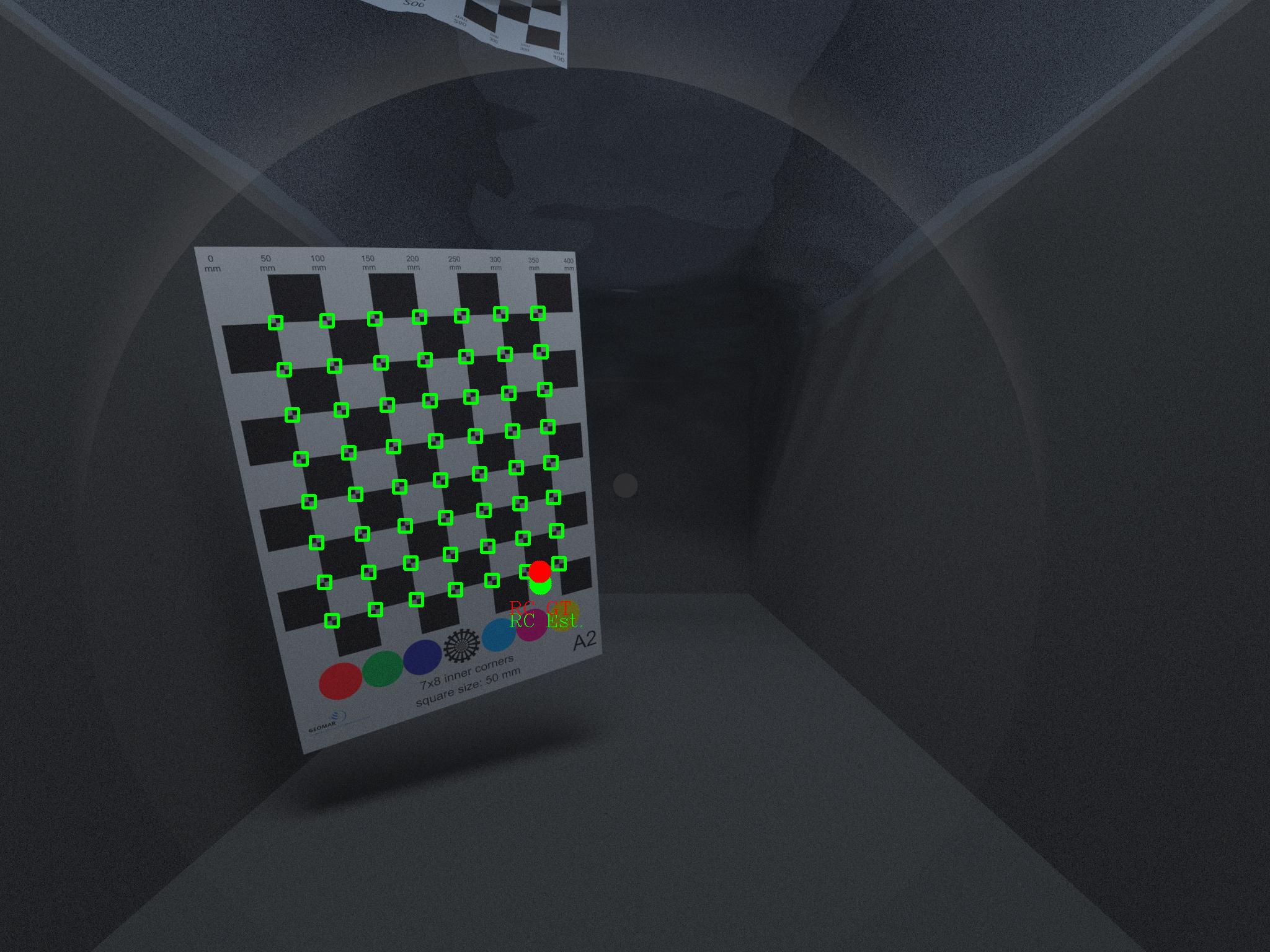}
		\includegraphics[width=0.24\textwidth]{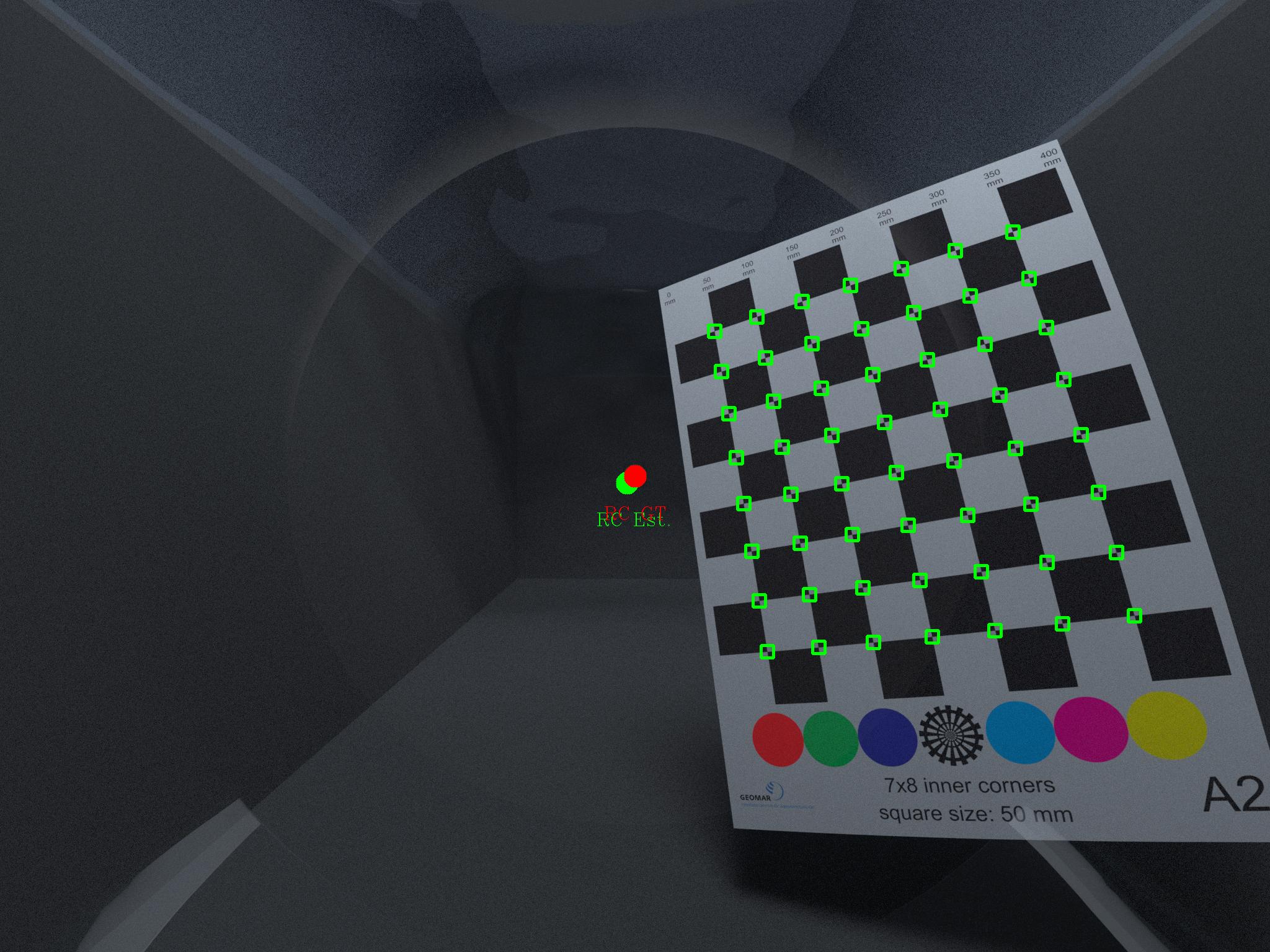}
		\includegraphics[width=0.24\textwidth]{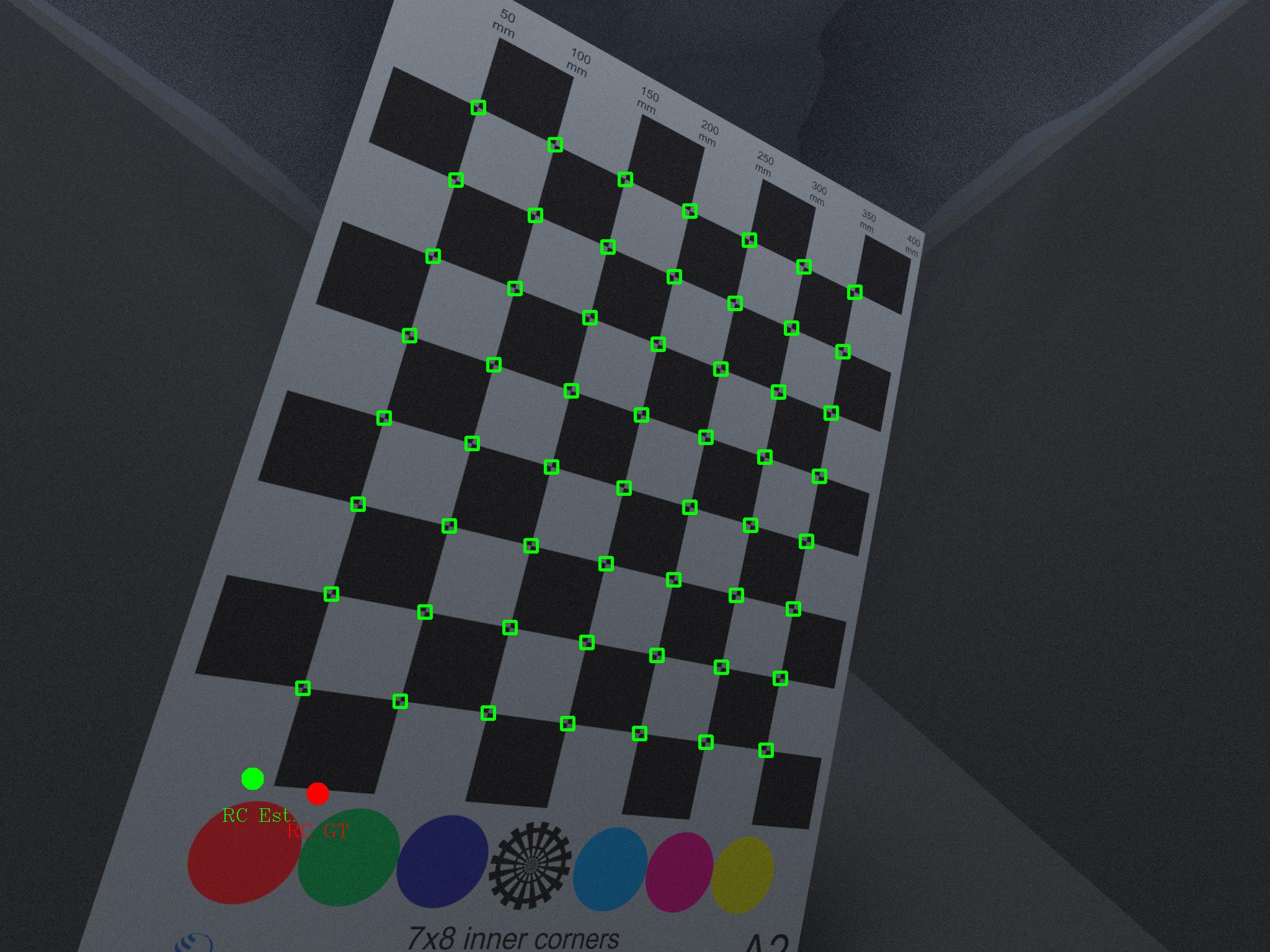}\vspace{0.8ex}
		\includegraphics[width=0.24\textwidth]{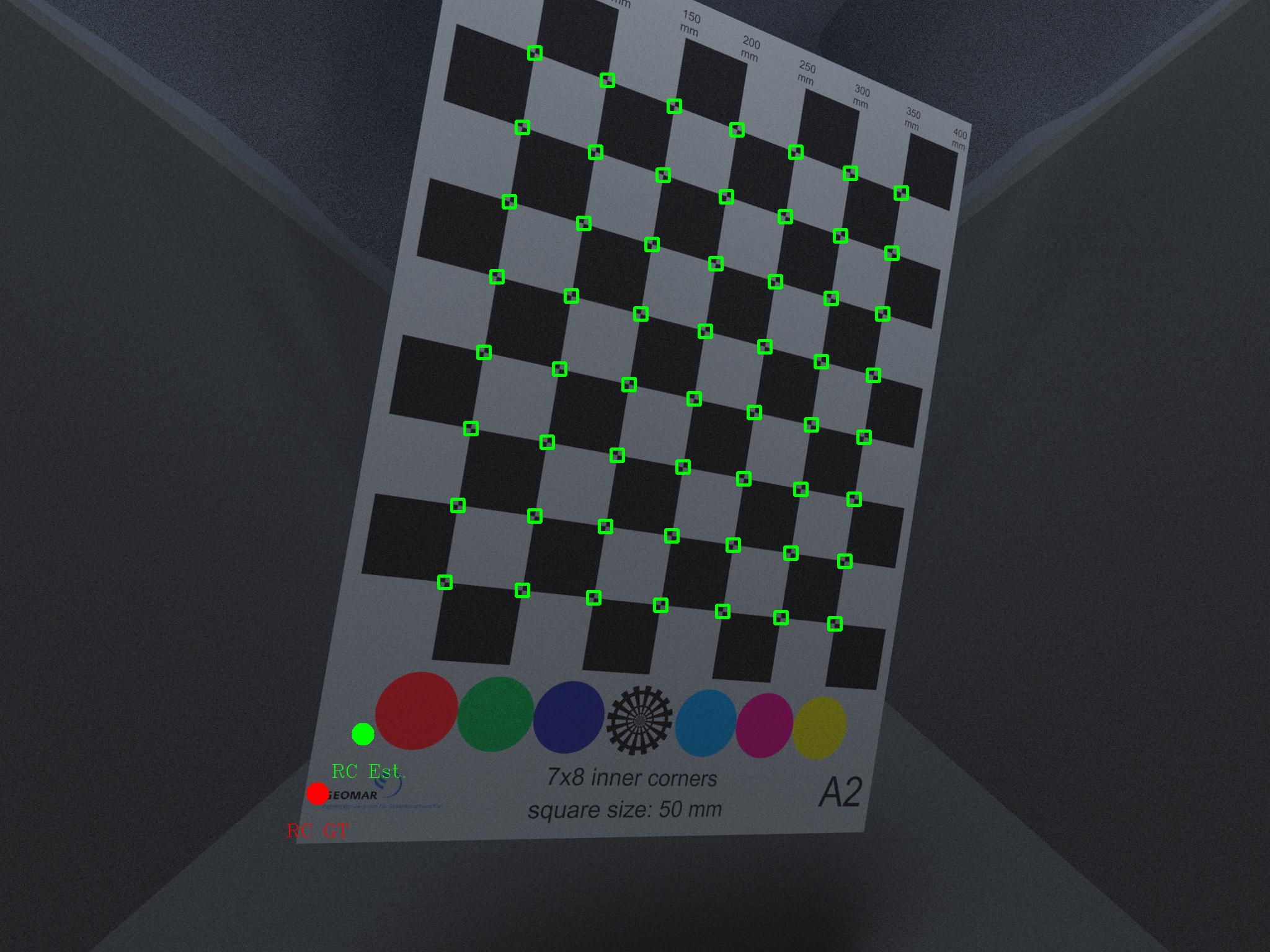}
		\includegraphics[width=0.24\textwidth]{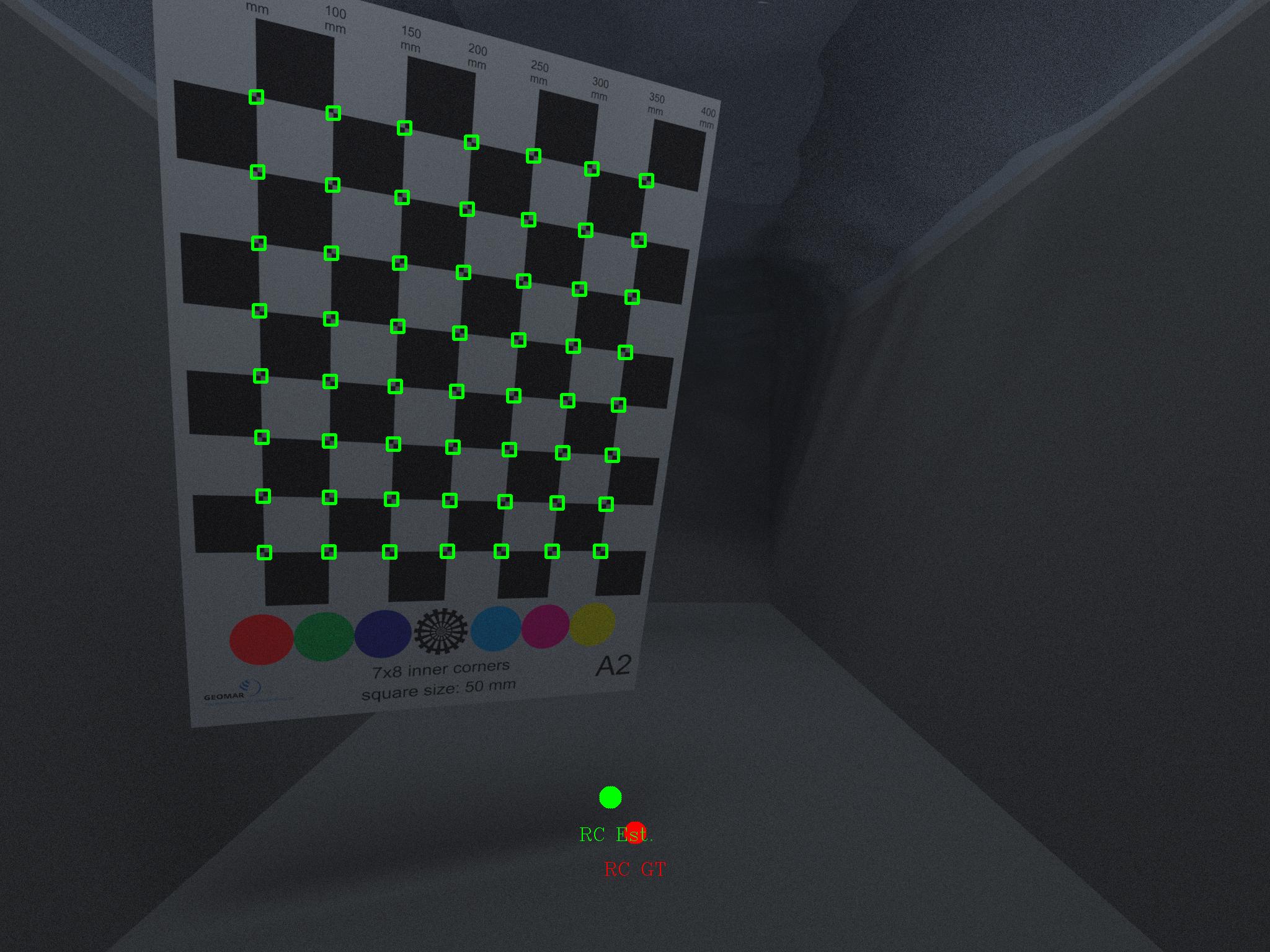}
		\includegraphics[width=0.24\textwidth]{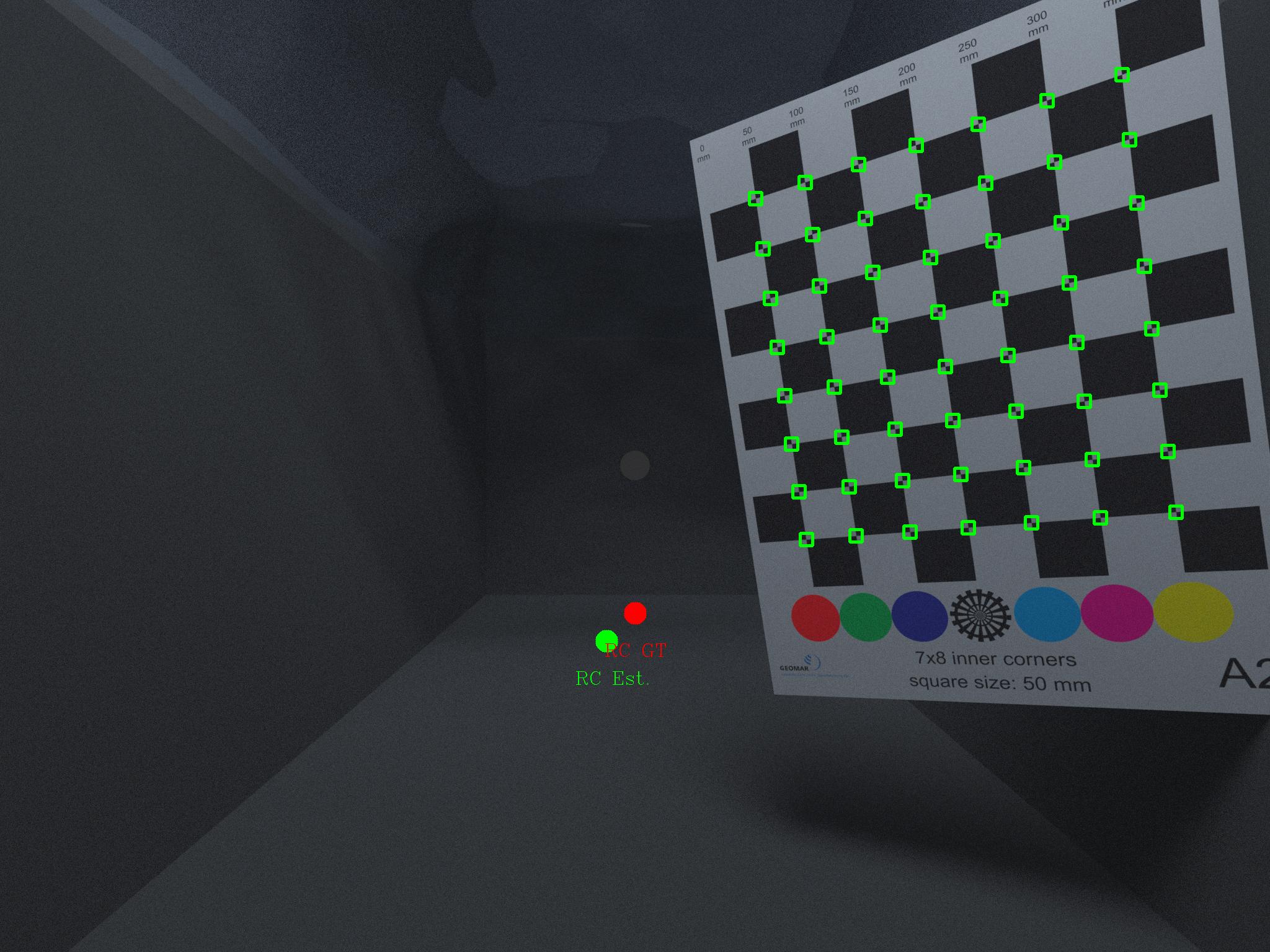}
		\includegraphics[width=0.24\textwidth]{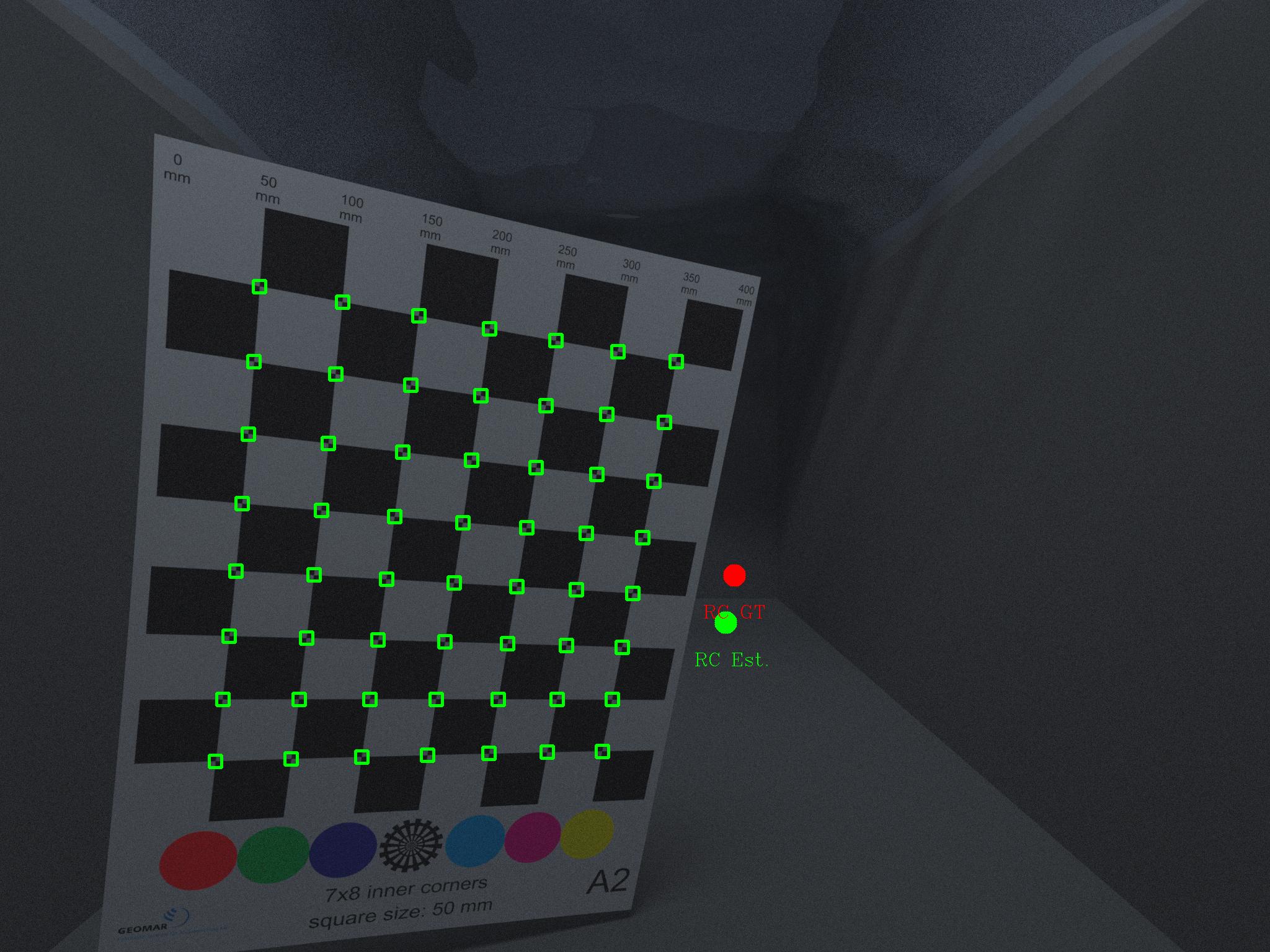}
	\end{center}
	\caption{Exemplary images showing the evaluation results of the experiments on the rendered images. The green squares and the green dots represent the reprojection results and the estimated refraction centers respectively. The ground truth refraction center is depicted as a red dot.}
	\label{fig:blender_offset_calib}
\end{figure*}

Next, we rendered 8 sets of images for different decentering situations and evaluate the refraction center estimation and the decentering vector calibration. 
The chessboard corners were automatically detected and the calibration results are shown in Table \ref{table:blender_calib_result}.
Fig. \ref{fig:blender_offset_calib} provides some of the resulting images, where the estimated refraction centers are marked as green spots, and the reprojections are marked as green crosses.
The ATE (Average Transformation Error, ATE) in the table represents the average pose error over all used images in the translational components between the estimated camera poses and the ground truth poses, which is given as \cite{sturm2012benchmark}:
\begin{equation}
\mathrm{ATE_\mathrm{trans}} = \sqrt{\frac{1}{n}\sum_{i=1}^{n}\Vert \mathrm{trans}(\mq P_{est,i}^{-1} \cdot \mq P_{gt,i}) \Vert^2_2)} 
\end{equation}
where $\mq P_{est,i}$ and $\mq P_{gt,i}$ are the estimated camera pose and the ground truth pose for the $i^{th}$ image respectively. $\mathrm{trans()}$ will extract the translational components from the transformation matrix.
As can be seen, the system can determine the decentering with high accuracy.
Note that in set $\textcircled{4}$, the camera is decentered horizontally, thus the refraction axis is parallel to the image plane and the refraction center is located at infinity.

\subsection{Real-World Experiments}
The proposed geometry insights and the calibration approach have been demonstrated on synthetic data and rendered images, but we finally want to validate them also in a real-world scenario. The experimental setup is shown in Fig. \ref{fig:blender_setup}, top, the dome port with $50mm$ radius and $7mm$ thickness was attached to the sidewall of a test tank, then the tank was filled with water. The camera had a resolution of $1280 \times 1024$ and a horizontal field of view of approximately $73^{\circ}$, then the intrinsic parameters of the camera were calibrated in air, also to obtain potential lens distortion.
The camera was then attached to the dome port afterwards. Next, we have placed a planar checkerboard inside the pool, captured images and performed chessboard detection (OpenCV). The positions of the detected chessboard corners are then corrected for lens distortion, and are given, jointly with the known chessboard corner grid positions (cf. to Fig. \ref{fig:chessrefraction}) to the eight-point algorithm for estimating the $\mq F$ matrices. From each estimated $\mq F$ the refraction center is extracted (for each of the images) without using any information about the index of refraction or thickness of the glass. Then, the sign of the decentering (along the refraction axis) was determined by the convexity test. Finally, non-linear optimization was performed to calibrate the decentering vector. The estimated decentering vector was $\q v_{\mathrm{off}} = (-0.04, -0.02, -20.18)\trans mm$, which is in agreement with our coarse manual measurements, the average refraction center was $(641.56 , 513.62)\trans$, and the RMSE of the reprojection error was 0.32 pixels. Fig. \ref{fig:eval_real_test_set_a} presents some exemplary images of reprojection results. Since measuring the actual offset between the camera center and the dome center is not trivial, we employed two strategies to validate the calibration results.

\begin{figure*}[!h]
	\begin{center}
		\subfloat[Images used for evaluation]
		{
			\includegraphics[width=0.24\textwidth]{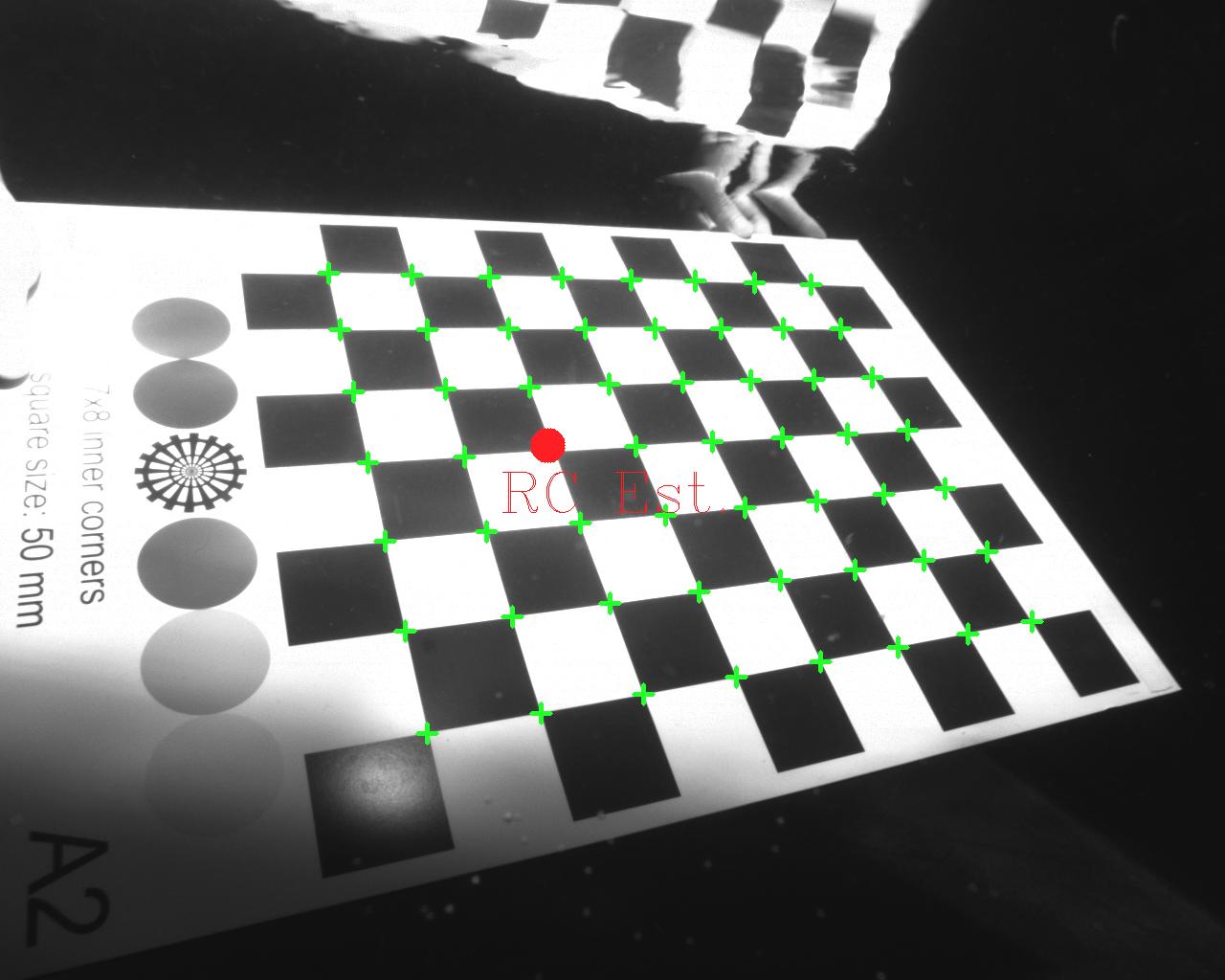}
			\includegraphics[width=0.24\textwidth]{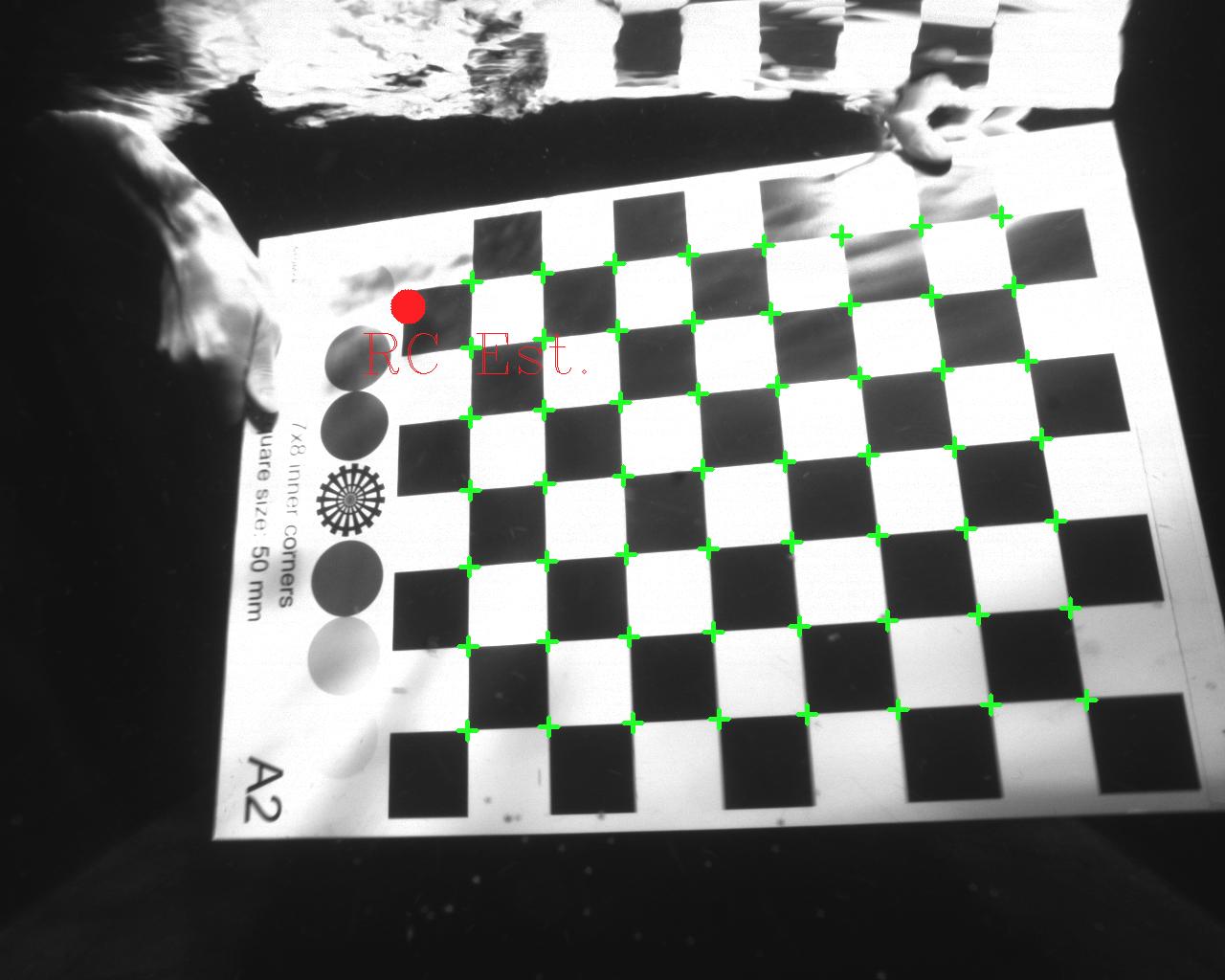}
			\includegraphics[width=0.24\textwidth]{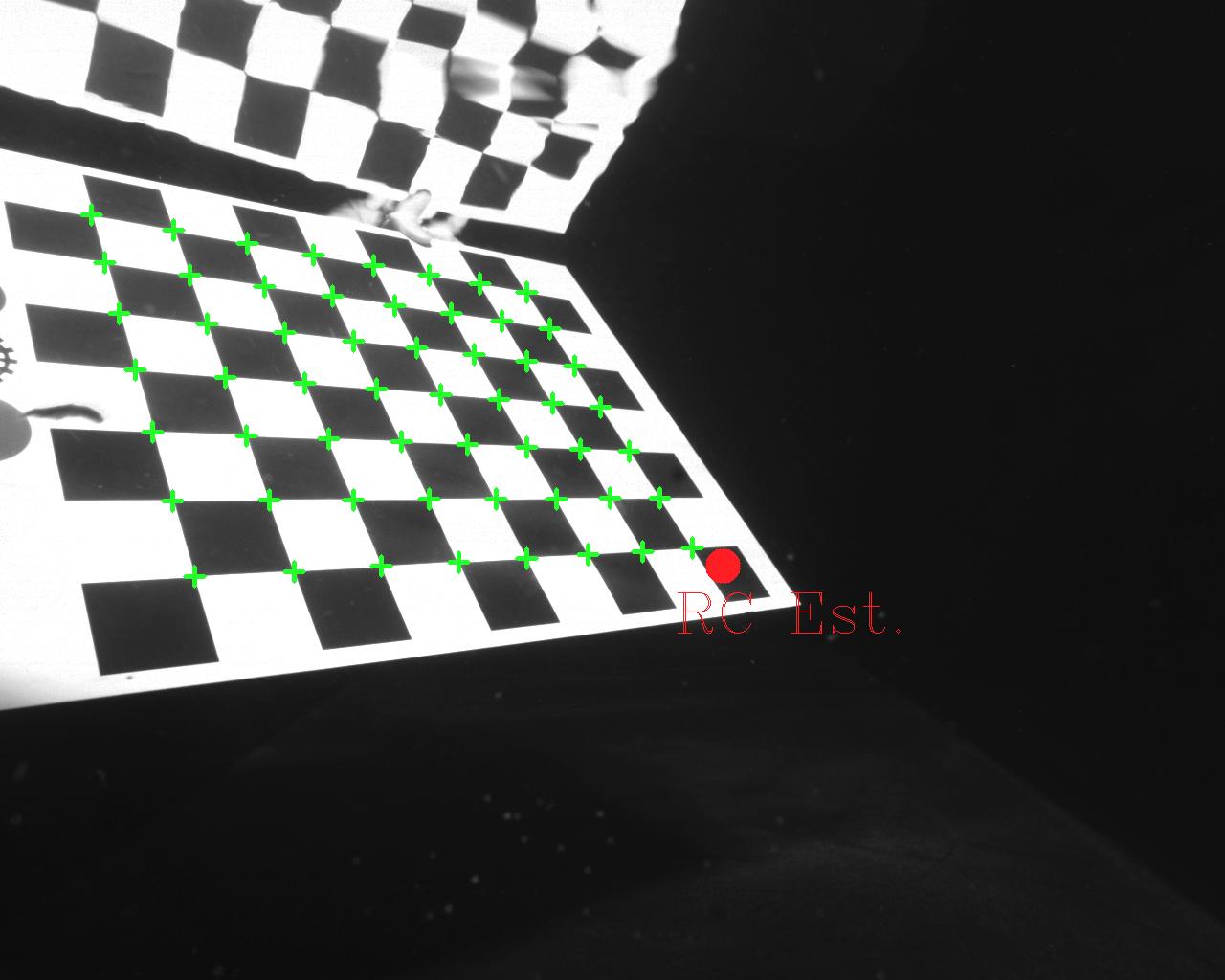}
			\includegraphics[width=0.24\textwidth]{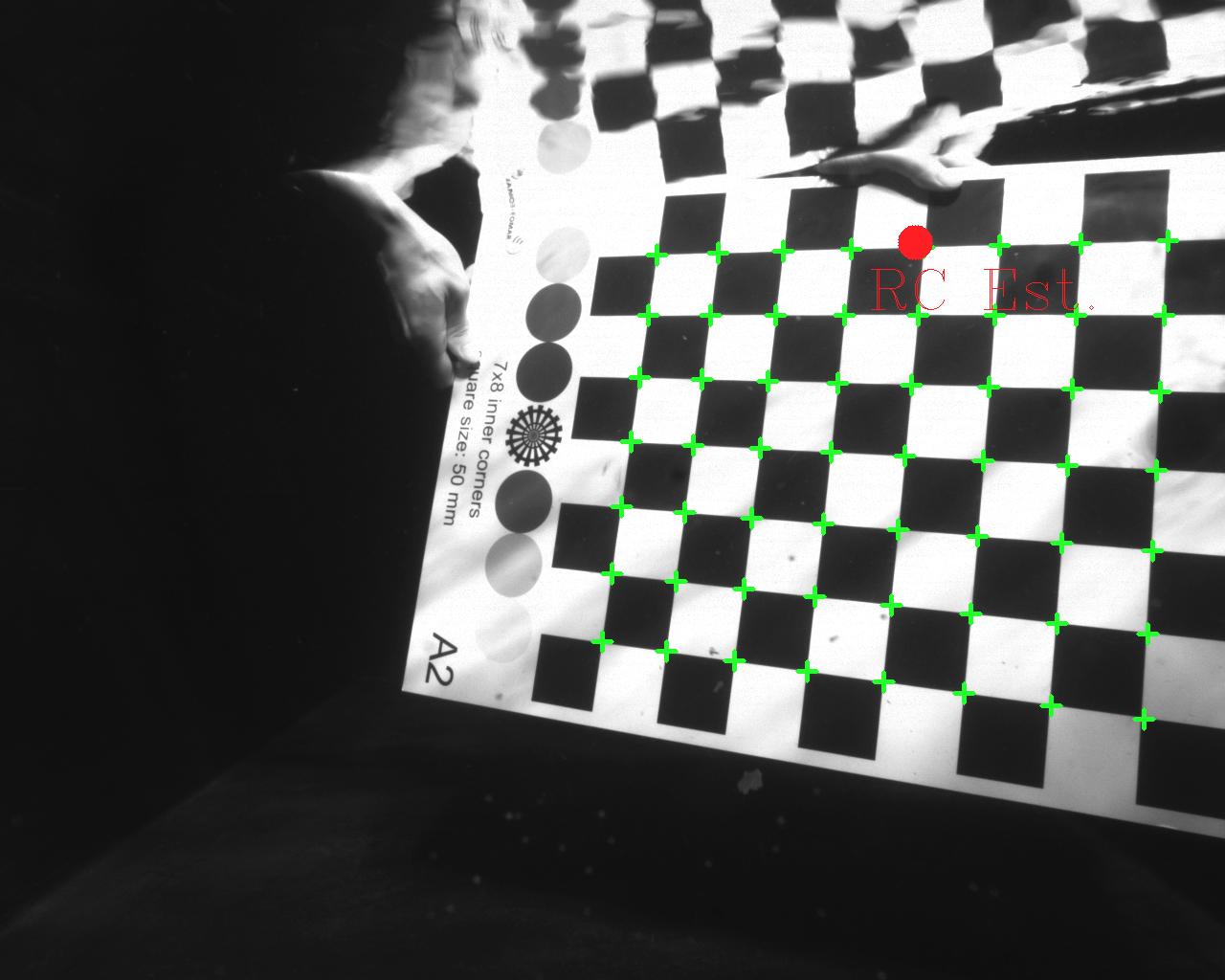}
			\label{fig:eval_real_test_set_a}
		}\newline
		\subfloat[Images used for validation]
		{
			\includegraphics[width=0.24\textwidth]{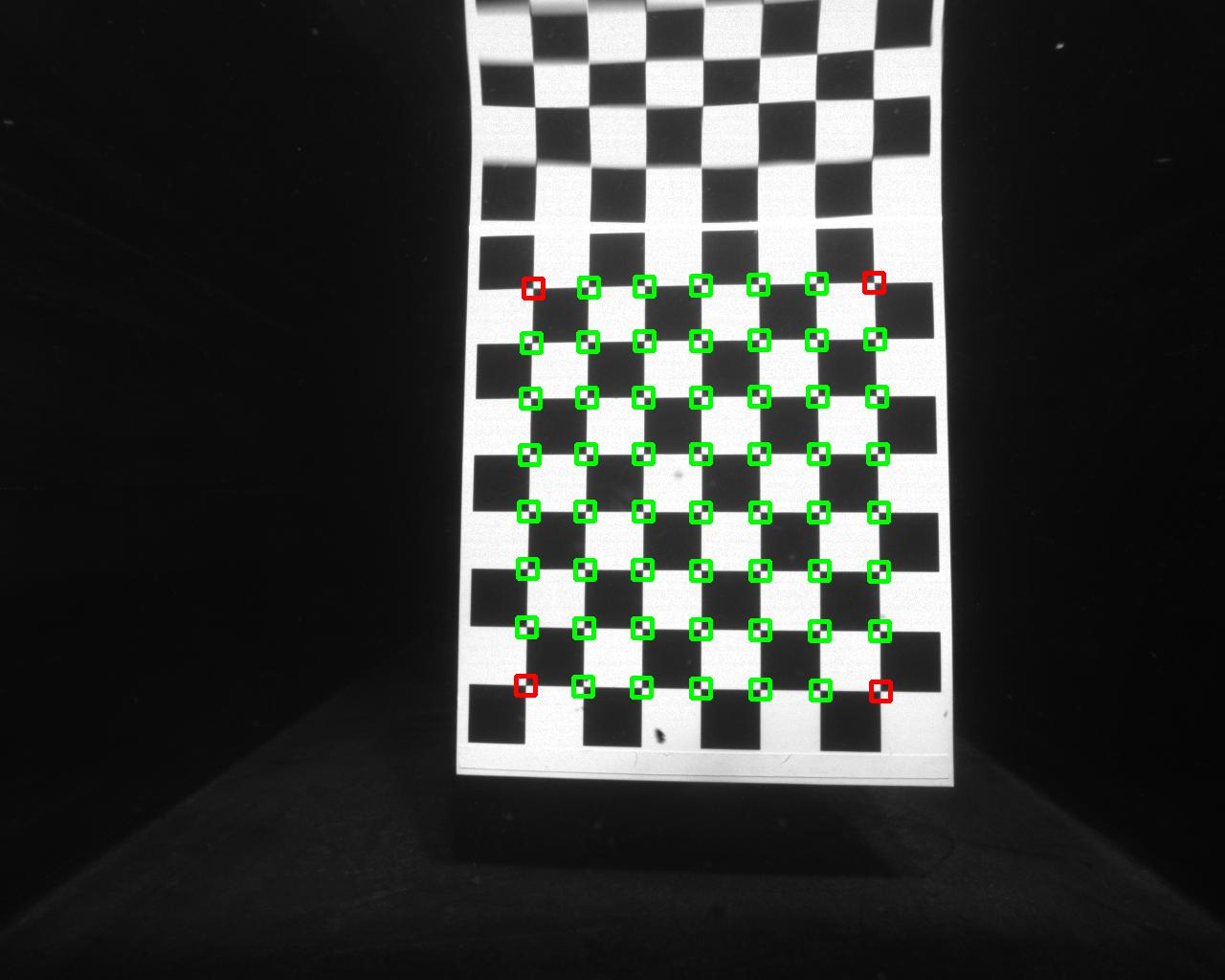}
			\includegraphics[width=0.24\textwidth]{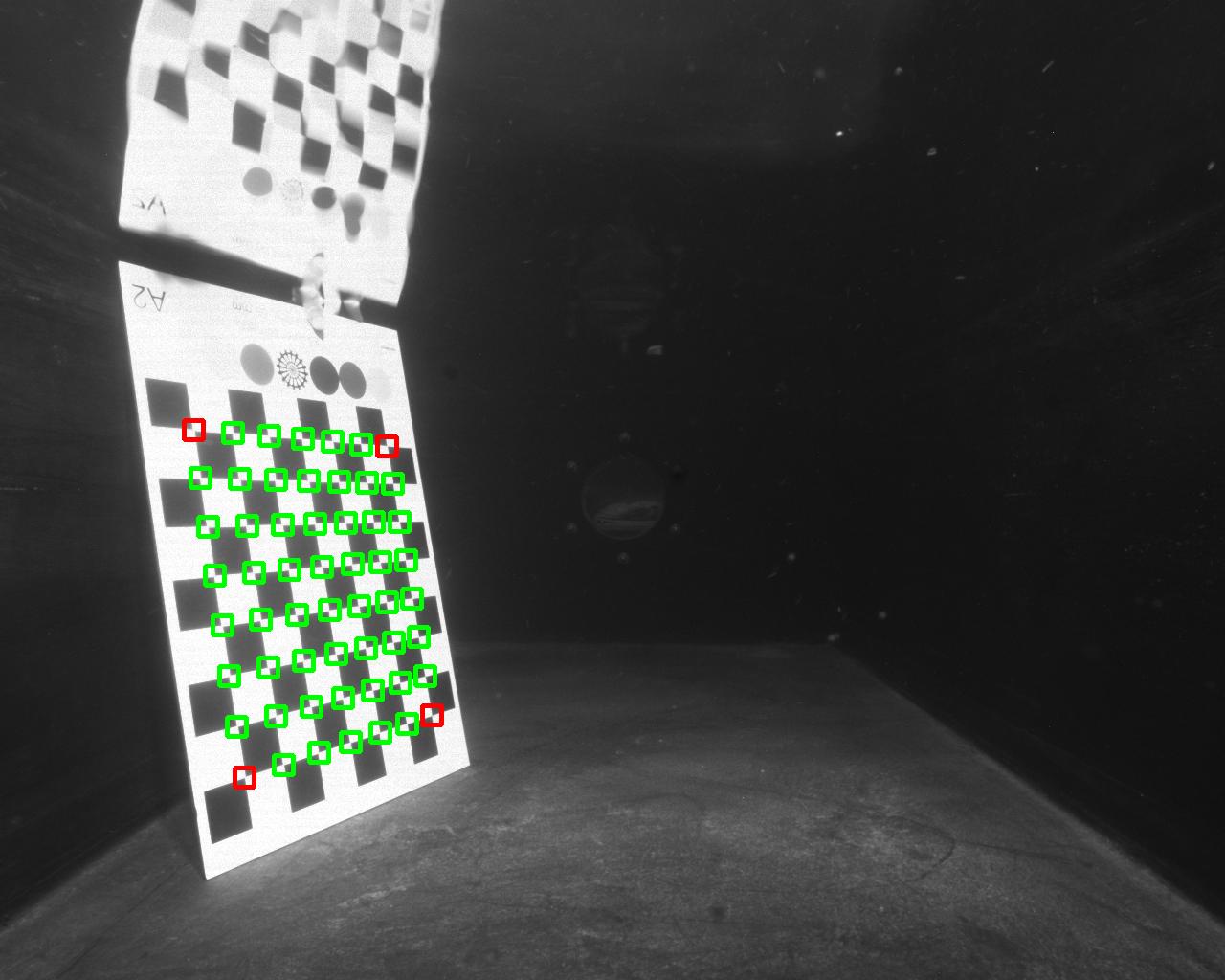}
			\includegraphics[width=0.24\textwidth]{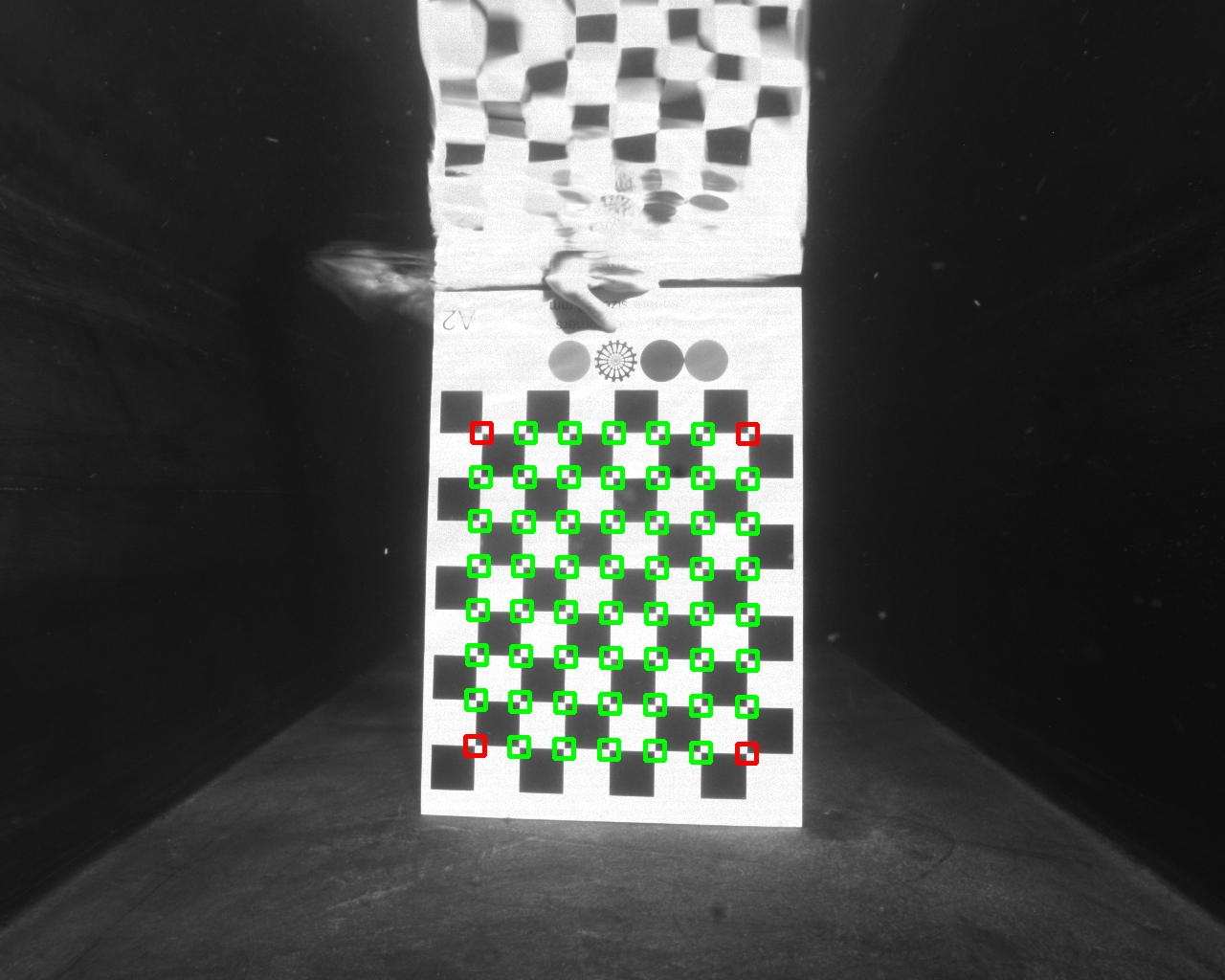}
			\includegraphics[width=0.24\textwidth]{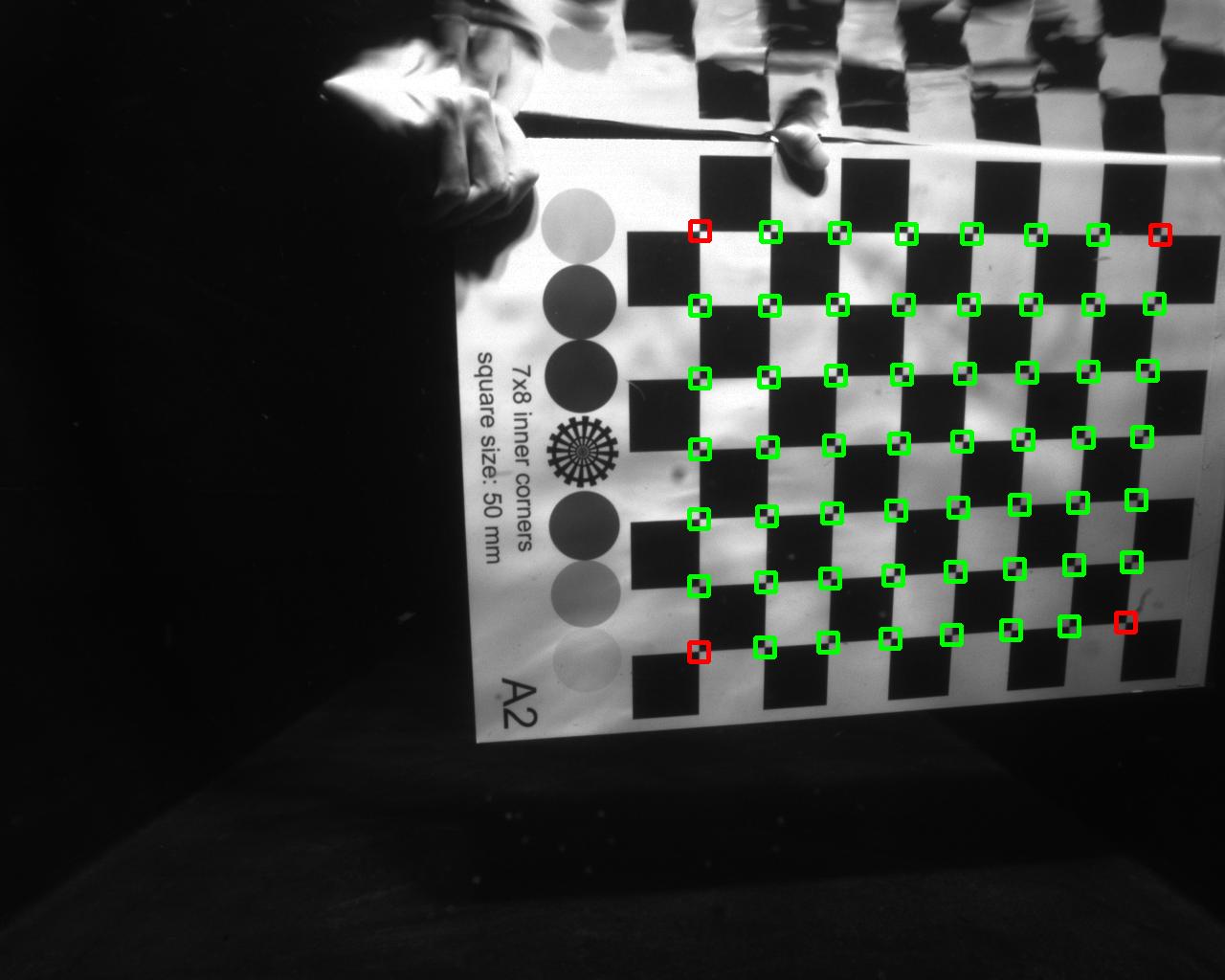}
			\label{fig:eval_real_test_set_b}
		}
	\end{center}
	\caption{Real-world evaluation results. (a), the results of the underwater decentering calibration in a water tank. (b), the validation results on the test set. The outermost 4 corners are marked as red squares, the inner reprojected corners are marked as green squares. }
	\label{fig:eval_real_test_set}
\end{figure*}

The first method is similar to the validation method suggested in \cite{peng2019calibration}, which is faced with a similar challenge. We separated our dataset into a calibration set and a validation set. The calibration set was used to perform the evaluation as described above whereas the latter one was used for validating the estimated decentering vector. For validation, the 4 outermost 3D-2D pairs of chessboard corners were utilized to compute the relative transformation of the camera with respect to the chessboard while keeping other parameters constant, then the remaining chessboard corners were projected and the average reprojection error was computed. The reason we used the outermost 4 corners for pose estimation is that the camera poses should be measured independently. There are 24 images for validation, thus there are $24 \times 52 = 1248$ points in total, and the average reprojection error was 0.611 pixels. Some example images are selected to be visualized in Fig. \ref{fig:eval_real_test_set_b}. Since localization with only 4 pairs of correspondences is less accurate than using all pairs, thus, below 1 pixel validation error can already be considered as satisfactory, from which we conclude that the decentring was estimated accurately.

\begin{figure}
	\begin{center}
		\subfloat[Photos of the setup]
		{
			\begin{minipage}{0.7\textwidth}
				\includegraphics[width=0.49\columnwidth]{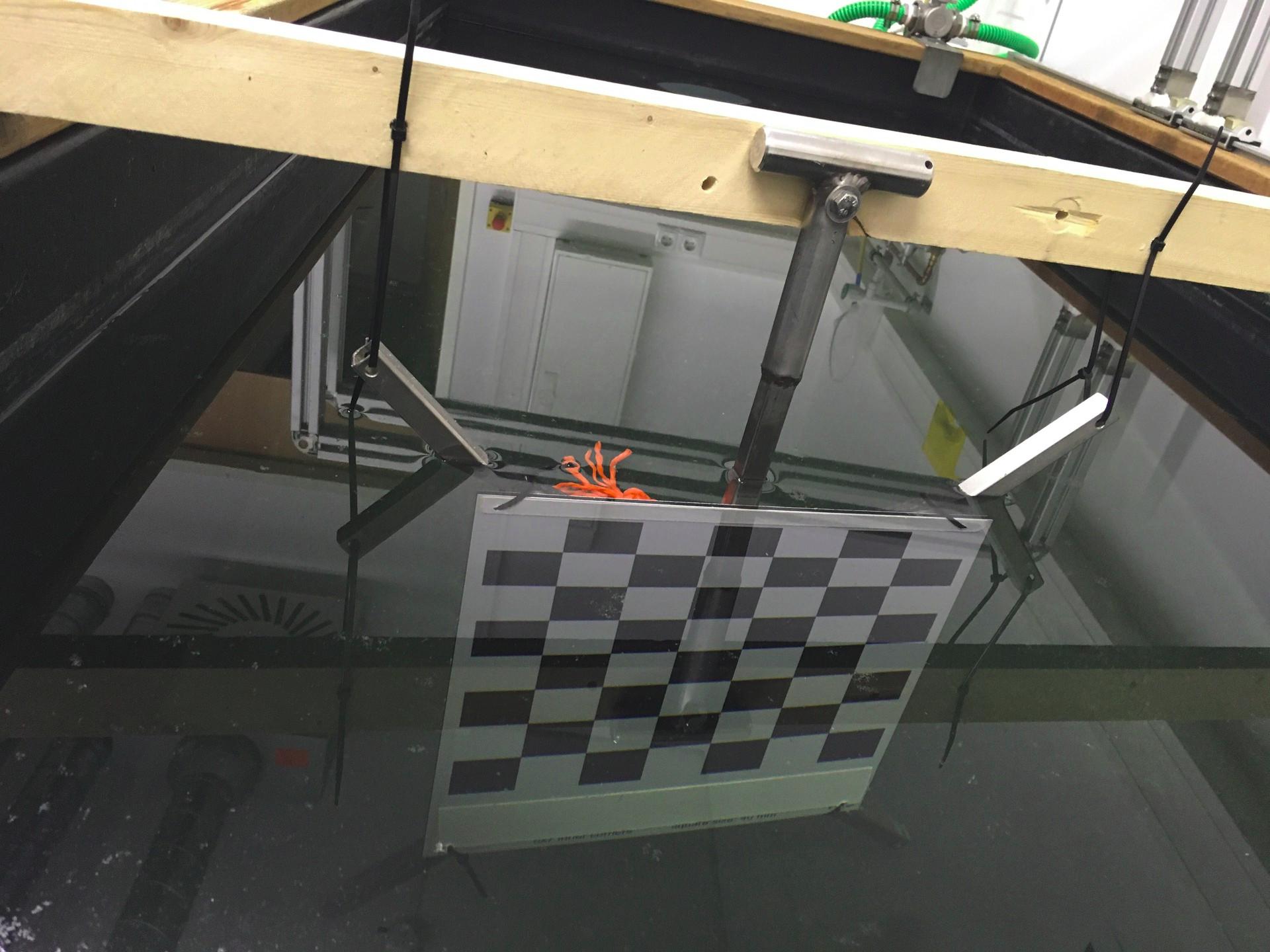}
				\includegraphics[width=0.49\columnwidth]{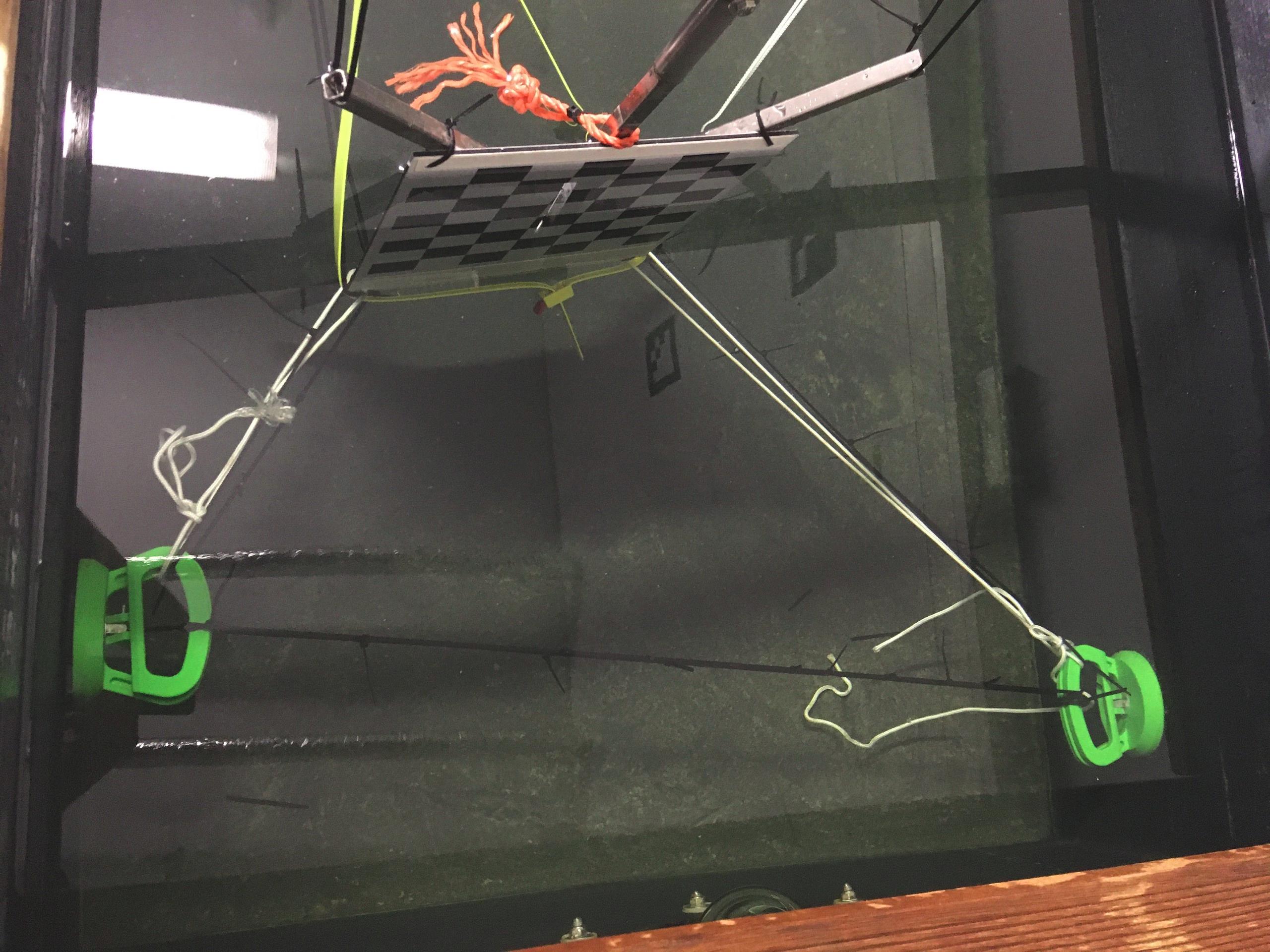}
			\end{minipage}
			\label{fig:airwater_validate_setup_a}
		}
	\end{center}
	\begin{center}
		\subfloat[Chessboard images captured in the air and underwater]
		{	\begin{minipage}{0.7\textwidth}
				\includegraphics[width=0.49\columnwidth]{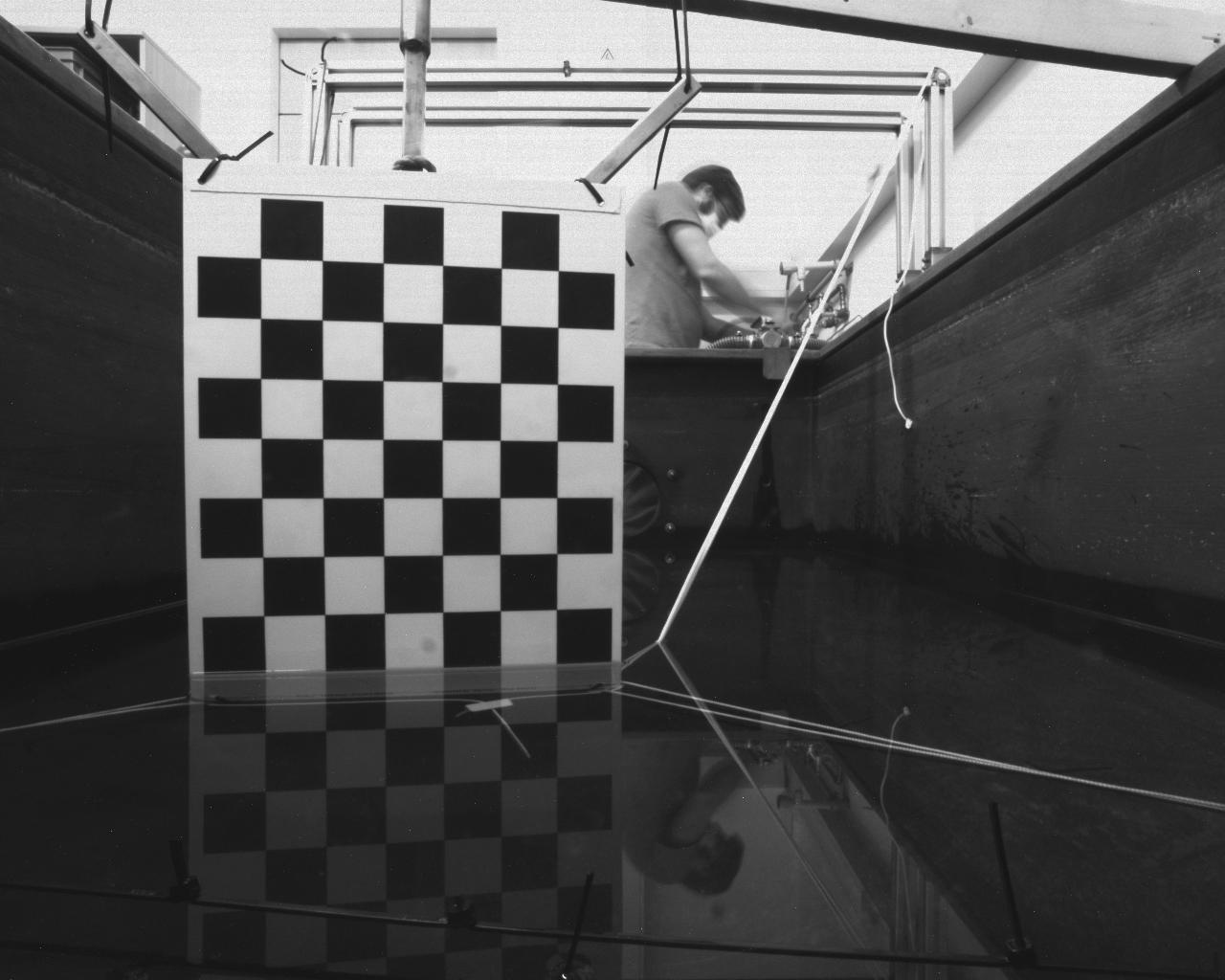}
				\includegraphics[width=0.49\columnwidth]{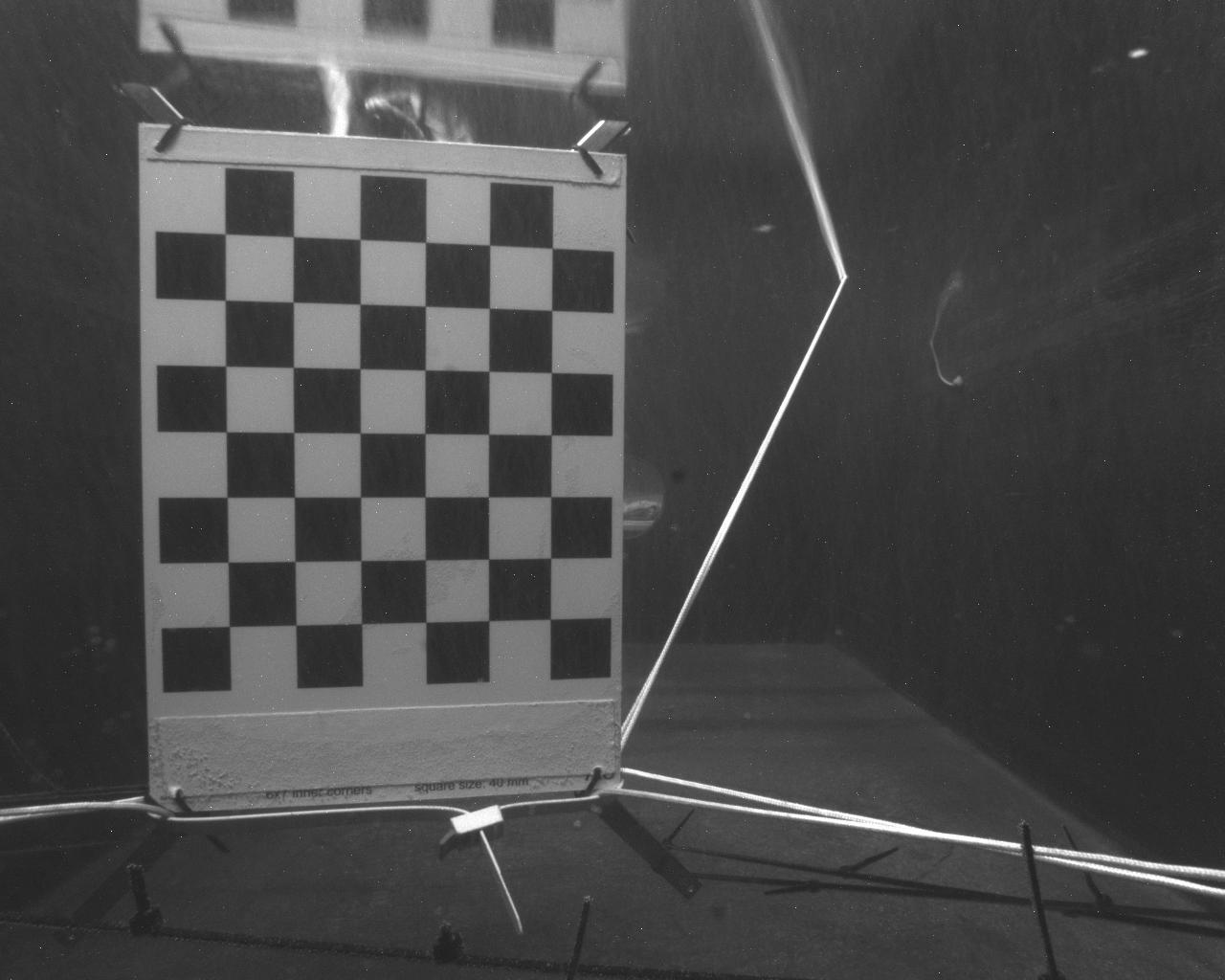}
			\end{minipage}
		}
	\end{center}
	\caption{Experimental setup for in-air/underwater image pair validation.}
	\label{fig:airwater_validate_setup}
\end{figure}

However, the relative transformation estimation could still partially remedy an estimation error in the decentering vector, therefore, we devised a second, complementary, validation experiment. We took a pair of photos of the chessboard underwater and in air, respectively, while keeping the chessboard position and orientation fixed. The experimental setup and the captured images are shown in Fig. \ref{fig:airwater_validate_setup}. Given the calibrated dome port camera system, we first used the in-air photo for estimating the chessboard pose, which can be denoted as $\q P_{\mathrm{air}}$. Next, we used the underwater photo to estimate the chessboard pose again, which we denote as $\q P_{\mathrm{unw}}$. Theoretically, the two estimated poses should be exactly the same as we assumed that the chessboard position and orientation had not changed while taking the image pair.

Consequently, we can measure the relative pose error between the two poses to validate the calibration results. The measured pose error in the translation component was $0.0107m$. It should be mentioned that when computing the in-air chessboard pose, we considered the refraction pattern when light rays travel through the air-glass-air interfaces. Afterwards, we did the reciprocal reprojection validation, where we projected the 3D chessboard corners onto the underwater image using the in-air estimated pose $\mq P_\mathrm{air}$ and then onto the in-air image using the underwater estimated pose $\mq P_\mathrm{unw}$. 

The average corner displacement induced by the refraction effect between the in-air and underwater image is 25.24 pixels. When considering refraction, as obtained from the disjoint set of calibration images, the measured distance is reduced to 0.59 and 0.62 pixels, respectively, which can be explained by unavoidable non-rigidity of the setup when filling the tank with water and some other potential sources of uncertainty as we will outline in the discussion. The results are displayed in Fig. \ref{fig:airwater_validate_results}.  Overall, the reduction of error and the previous validation show that the methods proposed are valid and can be applied in practice.

\begin{figure}
	\begin{center}
		\includegraphics[width=0.49\linewidth]{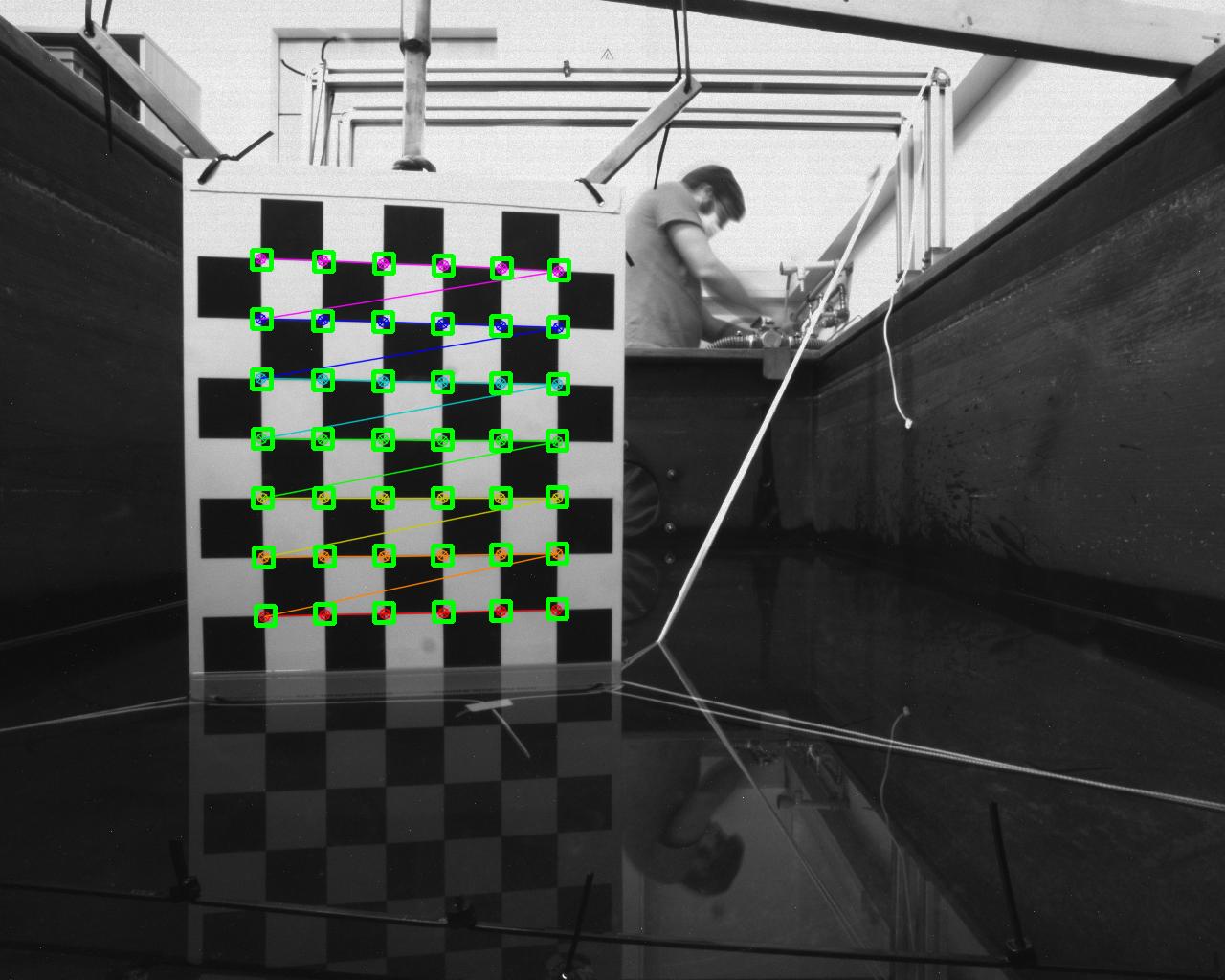}\vspace{0.8ex}
		\includegraphics[width=0.49\linewidth]{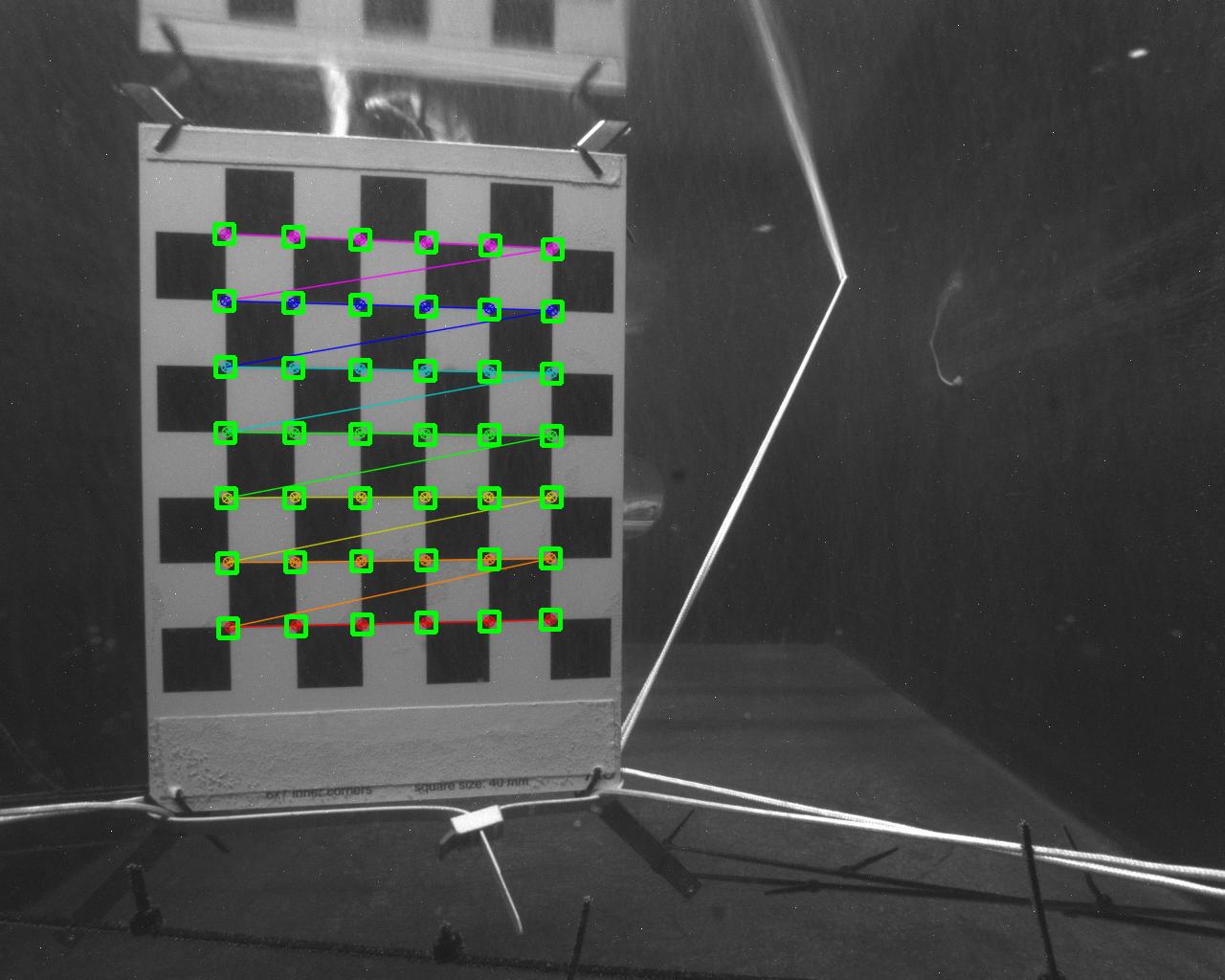}
	\end{center}
	\caption{Reciprocal reprojection validation results on the in-air/underwater image pair. }
	\label{fig:airwater_validate_results}
\end{figure}

\subsubsection*{Practical Calibration of an AUV Camera}
Finally, and this was also the motivation of this work in the first place, the insights and the developed techniques have been applied to a newly developed AUV camera for ocean research. This machine vision camera has a wide angle lens inside of a 50mm dome port and a pressure housing. We have first calibrated the camera in air, before we applied the mechanical adjustment method as proposed in \cite{she2019adjustment} until no refraction effects were observable. Afterwards, the entire system was submerged in a tank and chessboard images were recorded.
The images were undistorted according to the in-air calibration, then we apply the steps as reported in the previous sections. We obtain  a calibration residual of 0.299 pixels for a decentering estimate of $\q v_{\mathrm{off}} = (-0.36, 0.14, 0.27) mm$, which means less than half a millimeter decentering in total.
The camera system, sample calibration images with reprojected corners, as well as the AUV are displayed in Fig. \ref{fig:antoncam}. 

\section{Discussion}

\begin{figure}
	\begin{center}
		\subfloat[Strong refraction]{
			\includegraphics[width=0.49\linewidth]{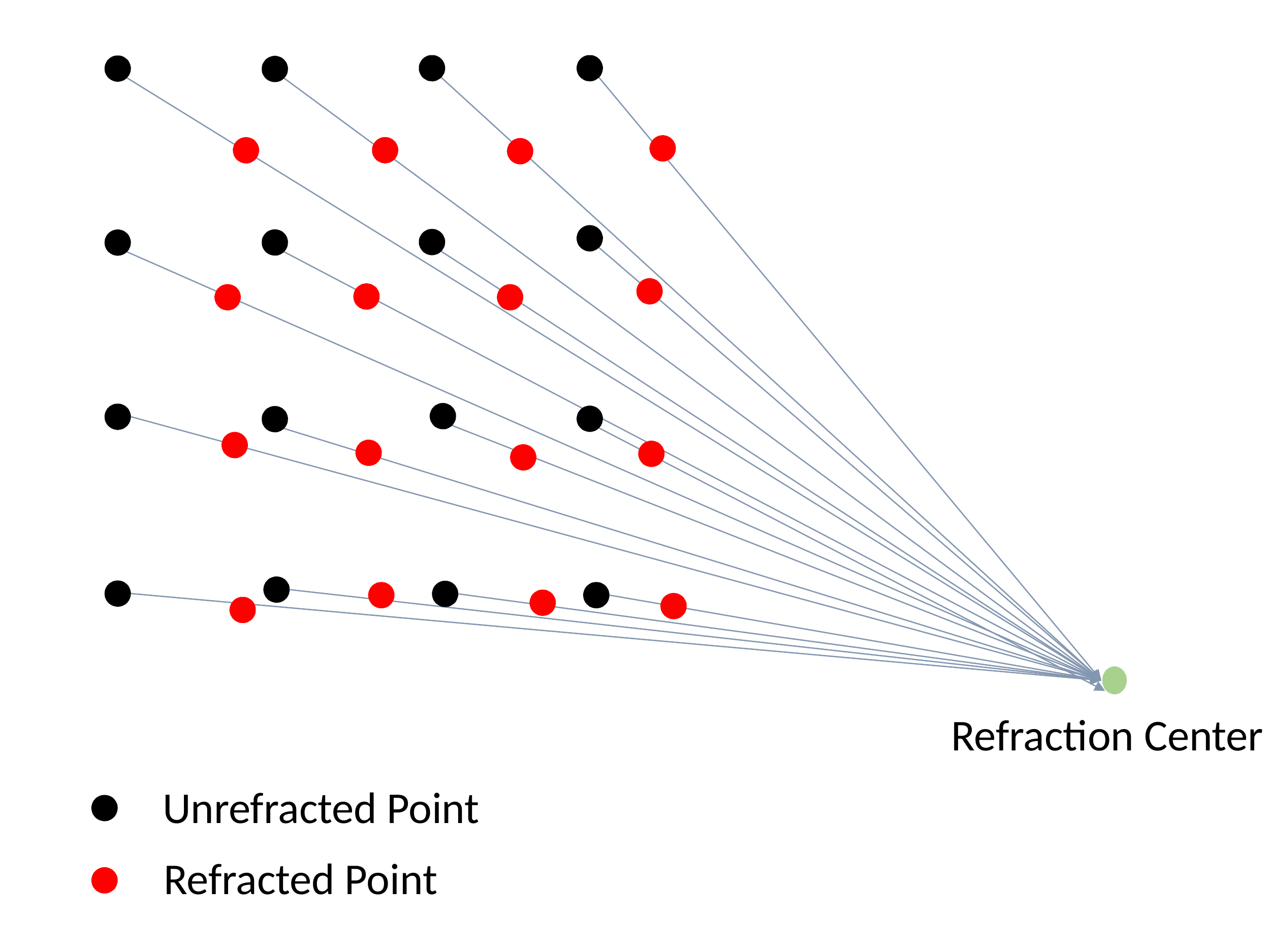}
			\includegraphics[width=0.49\linewidth]{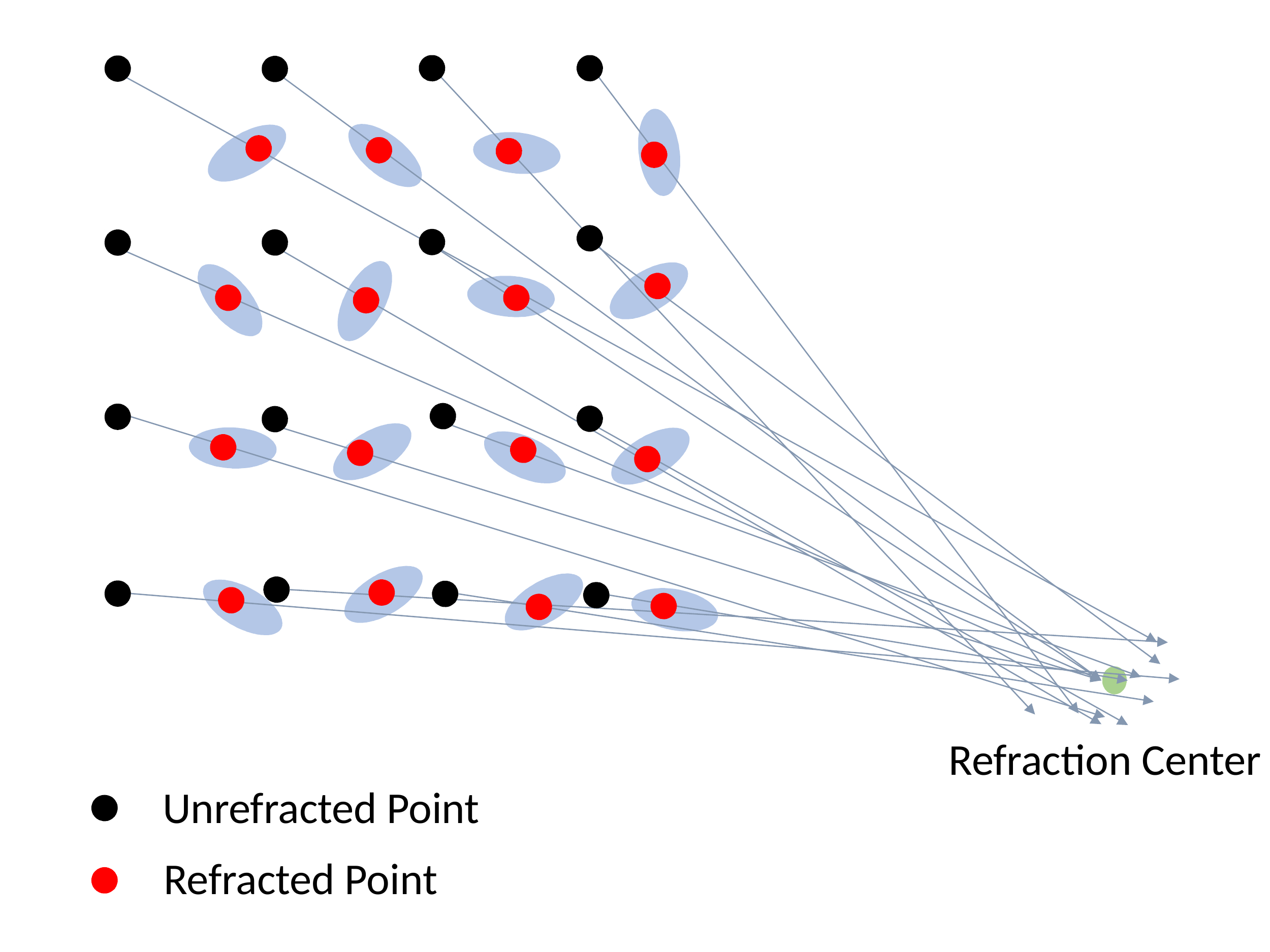}	
		}
	\end{center}
	\begin{center}
		\subfloat[Weak refraction]{
			\includegraphics[width=0.49\linewidth]{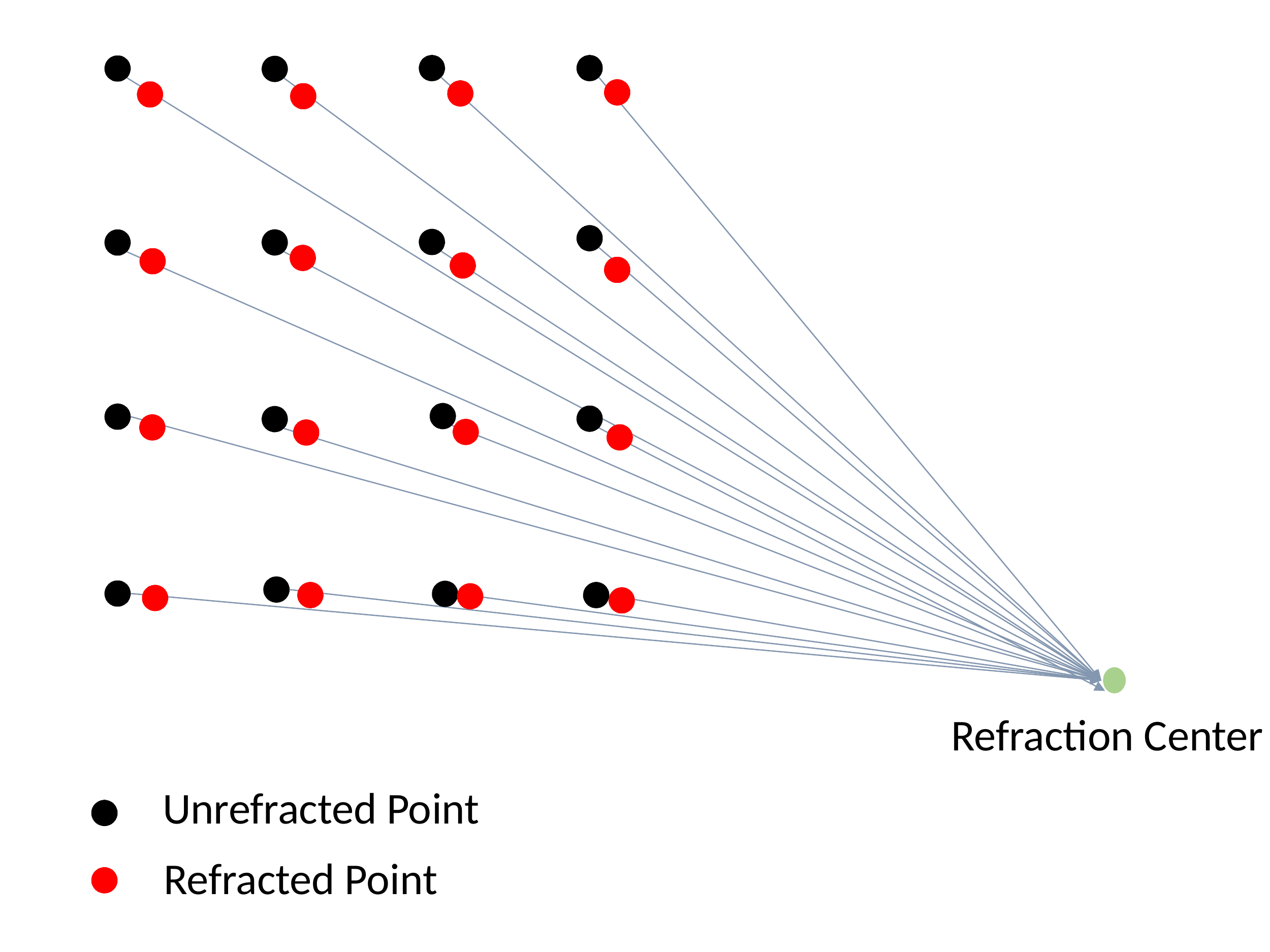}
			\includegraphics[width=0.49\linewidth]{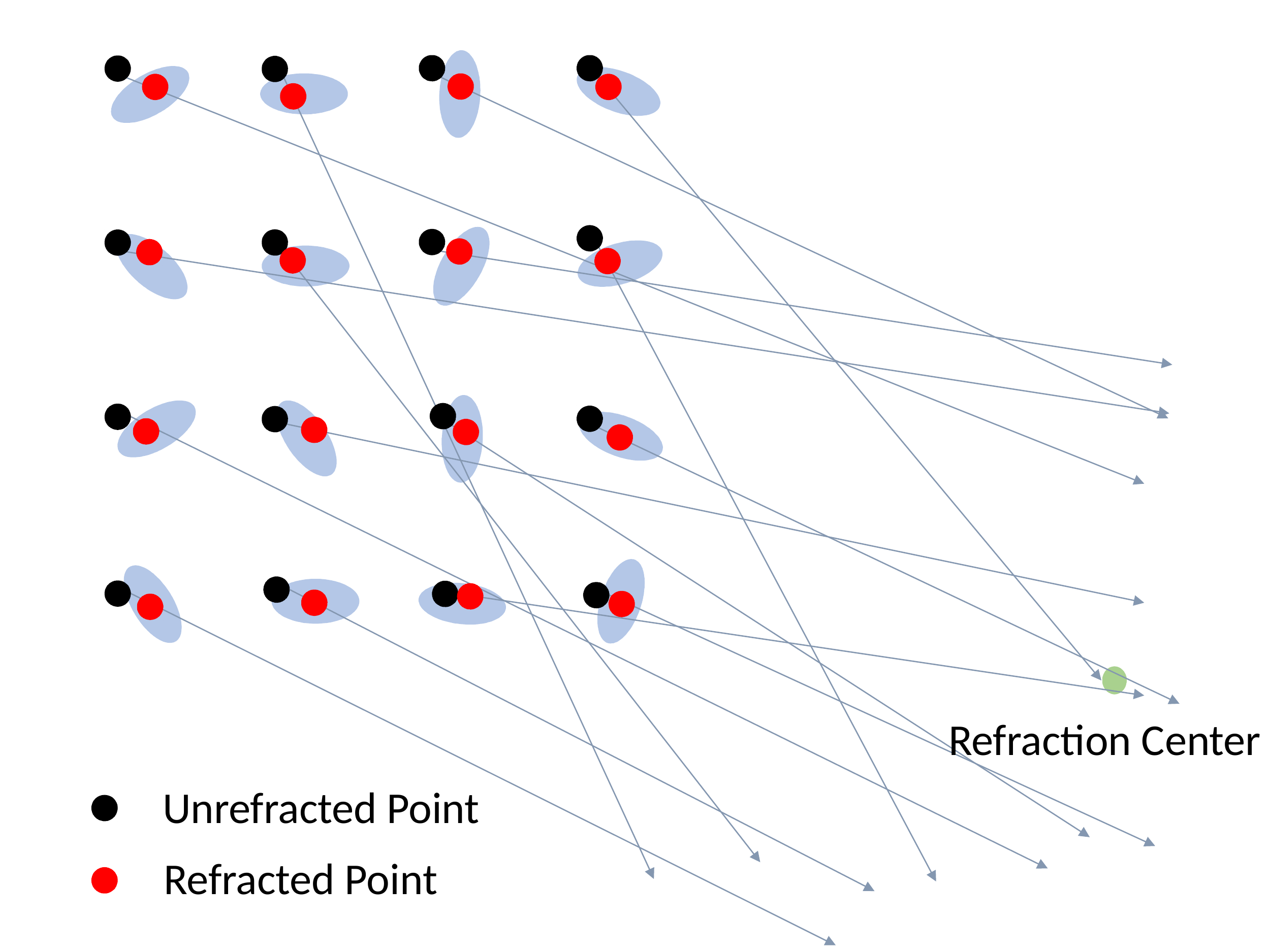}
			\label{fig:displace_noise_b}
		}
	\end{center}
	\label{fig:displace_noise}
	\caption{(a), When there is strong refraction, noise does not break the refractive colinear constraint, and the \textit{Homography Mapping Error} is high. (b), When the refraction is weak, noise can severely impair the refractive colinear constraint, which affects the calibration of the decentering, and the refraction center. In this case, the \textit{Homography Mapping Error} is low. The gray ellipses represent the uncertainties of the chessboard corner detection. For illustration purposes, the uncertainties are exaggerated.}
\end{figure}

As the proposed geometrical insights and the decentering calibration approach rely on refraction effects caused by a decentered camera photographing through a spherical interface, in this section we would like to report some lessons learned.
\paragraph{Homography Mapping Error}
In the classical approach of using a chessboard for calibration, which we also follow in this manuscript, we have to estimate a chessboard pose for each calibration image, on top of the actually desired refraction parameters.
If we can directly compute a homography between an underwater image and the chessboard pattern, and the residual error for the homography is below the corner detector noise, this means that the center of refraction is effectively not observable in this calibration image, i.e. refraction effects are \emph{drowned} in noise (see Fig. \ref{fig:displace_noise_b}).
Typically, this particular scenario happens when the decentering is very small or the chessboard corners in the image are located close to the refraction center (low refraction effects, see section \ref{sec:geometry}).
Also, in case the refracted image point $\q x_r$ and the un-refracted image point $\q x_a$ are related by a 2D similarity transformation
\begin{equation}
\q x_r \simeq \mq H_s \; \q x_a ,
\end{equation}
then the overall relation between refracted image coordinates and 3D chessboard points becomes
\begin{equation}
\q x_r \simeq \mq H_s \mq H \; \q x_c \simeq \mq H^{\prime} \; \q x_c .
\end{equation}
This means that the overall homography $\mq H^{\prime}$ can absorb all refraction effects and the equation system arising from eq.\ref{eq:festimation} is underconstrained (ambiguous).

To detect and avoid these situations, we define the \textit{Homography Mapping Error} as the residual of a homography mapping from the 3D points to the image points. 
We recommend users to monitor the Homography Mapping Error for each calibration image, and if it is low (compared to the noise level), to place the chessboard away from the refraction axis where we can see stronger effects (larger signal-to-noise-ratio). Additionally, the \textit{Homography Mapping Error} can also be used to pre-select or to weight images for decentering calibration.

\paragraph{Non-Single-View-Point Camera}
The geometry in this paper is analyzed based on the assumption that the camera has a single center of projection. For non-single-view-point lenses such as fisheye lenses, there is no perfect pinhole but a caustic where all light rays pass through. Centering such a lens with the dome port is in principle difficult as we can only bring one point to the dome center. 

\paragraph{Uncertainty of Intrinsic Calibration}
It is well known that the principal point of a camera is difficult to observe in air \citep{pollefeys1999self,hartley2002sensitivity}.
Conversely, this also means of course that it does not have a big impact on 3D mapping in air.
Underwater however, the situation is different: Since the principal point is the intersection point of the optical axis with the image plane, the ray direction and refraction for every pixel will be impacted once the principal point changes, creating a different refraction pattern. Chessboard-based experiments to calibrate both the camera intrinsics and the decentering jointly failed, although we can achieve a lower calibration residual, the estimated parameters were off from the ground truth. Those parameters are correlated and probably more powerful (non-flat) calibration targets are required to better constrain the principal point. As in this contribution we discuss refraction effects, calibration of intrinsics is beyond the scope of this paper, but high-precision principal point calibration using underwater refraction could be an interesting option for future research.

\paragraph{Advantage over In-Air/Underwater Pair Calibration}
In order to obtain independent measurements of the decentering we have performed the method proposed in \cite{she2019adjustment} that is based on a pair of images, one taken in water, the other one in air, but with the same pose.
While this method works in principle, the accuracy is limited by physical constraints.

When trying to bring the air and water pair into as good as possible agreement, we realized that it was extremely difficult to keep the chessboard steady when capturing the in-air/underwater image pair during the experiment. When the water is injected into the pool, waves are generated at the water surface, making the checkerboard unstable. In addition, we also found that the water pool deforms slightly when carrying more water due to the increasing pressure. Consequently, we had to spend substantial effort to keep the chessboard steady. As can be seen from Fig. \ref{fig:airwater_validate_setup_a}, we first firmly attached the chessboard to a metal frame and then attached them together on a wooden bar on top of the pool. Next, nylon ropes were used to connect the chessboard with the sidewall of the pool. Nevertheless, we still believe that there is half a pixel error due to non-rigidity of the setup. We therefore conclude that the in-air/underwater pair method is limited in the achievable accuracy. The pure underwater calibration suggested in this paper is more practical, as it can be cumbersome in some cases to bring the chessboard and camera (e.g. attached on a robot) from water into air without changing relative pose, while in the proposed method an underwater camera just has to be submerged, as shown in Fig. \ref{fig:antoncam} for a camera of an AUV. Other cameras as shown in Fig. \ref{fig:domeport} can also be calibrated using this method even in the ocean.

\begin{figure}
	\begin{center}
		\includegraphics[height=3.78cm]{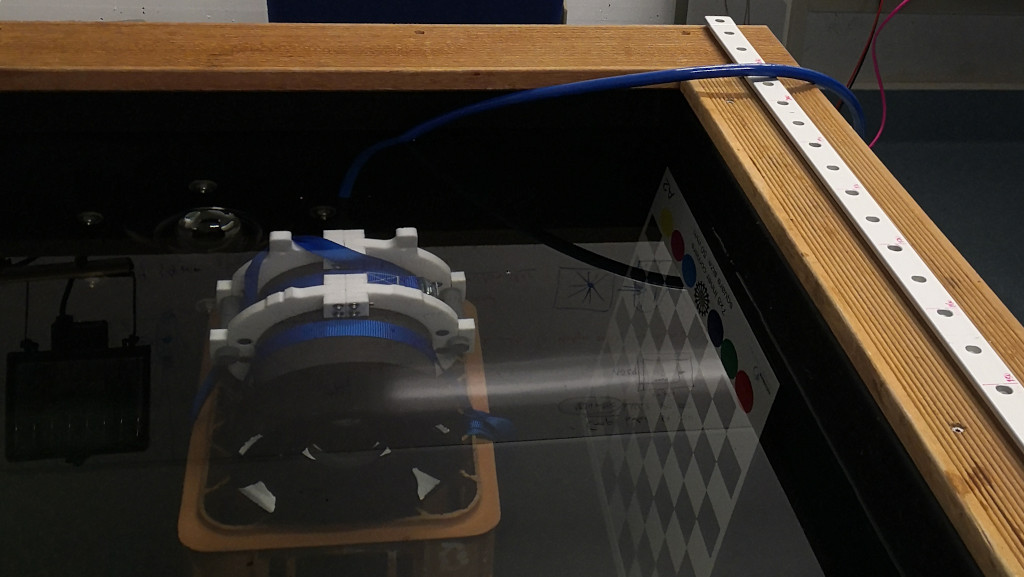}\vspace{0.1ex}
		\includegraphics[height=3.78cm]{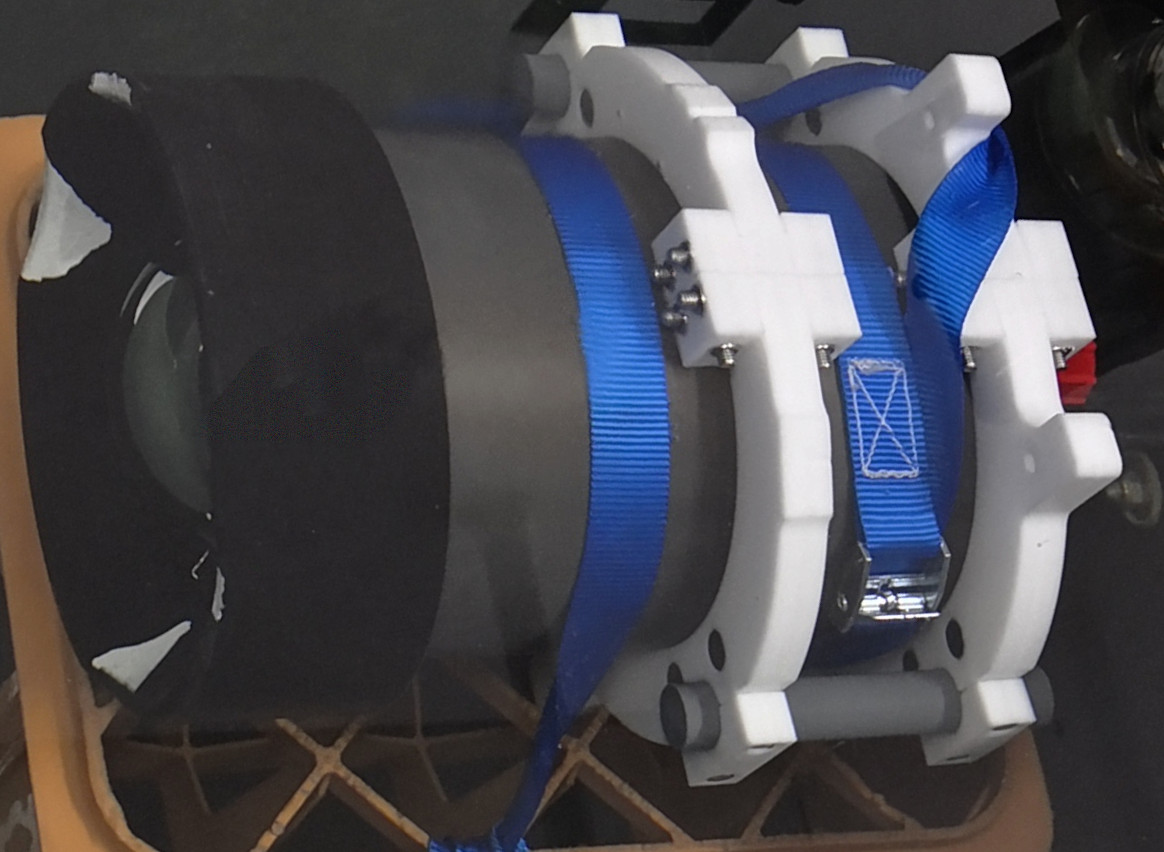}\vspace{0.1ex}
		\includegraphics[width=0.49\linewidth]{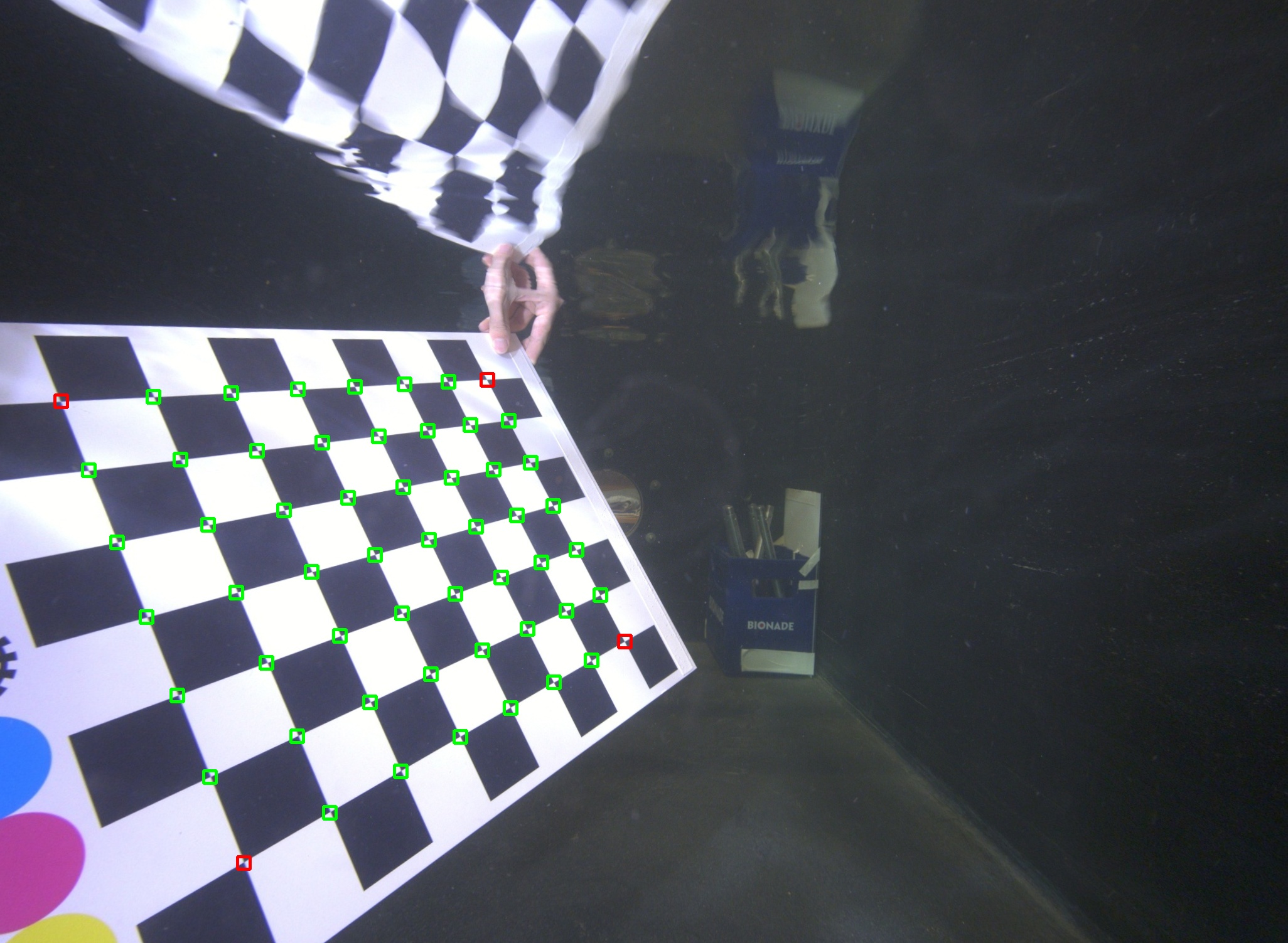}\vspace{0.1ex}
		\includegraphics[width=0.49\linewidth]{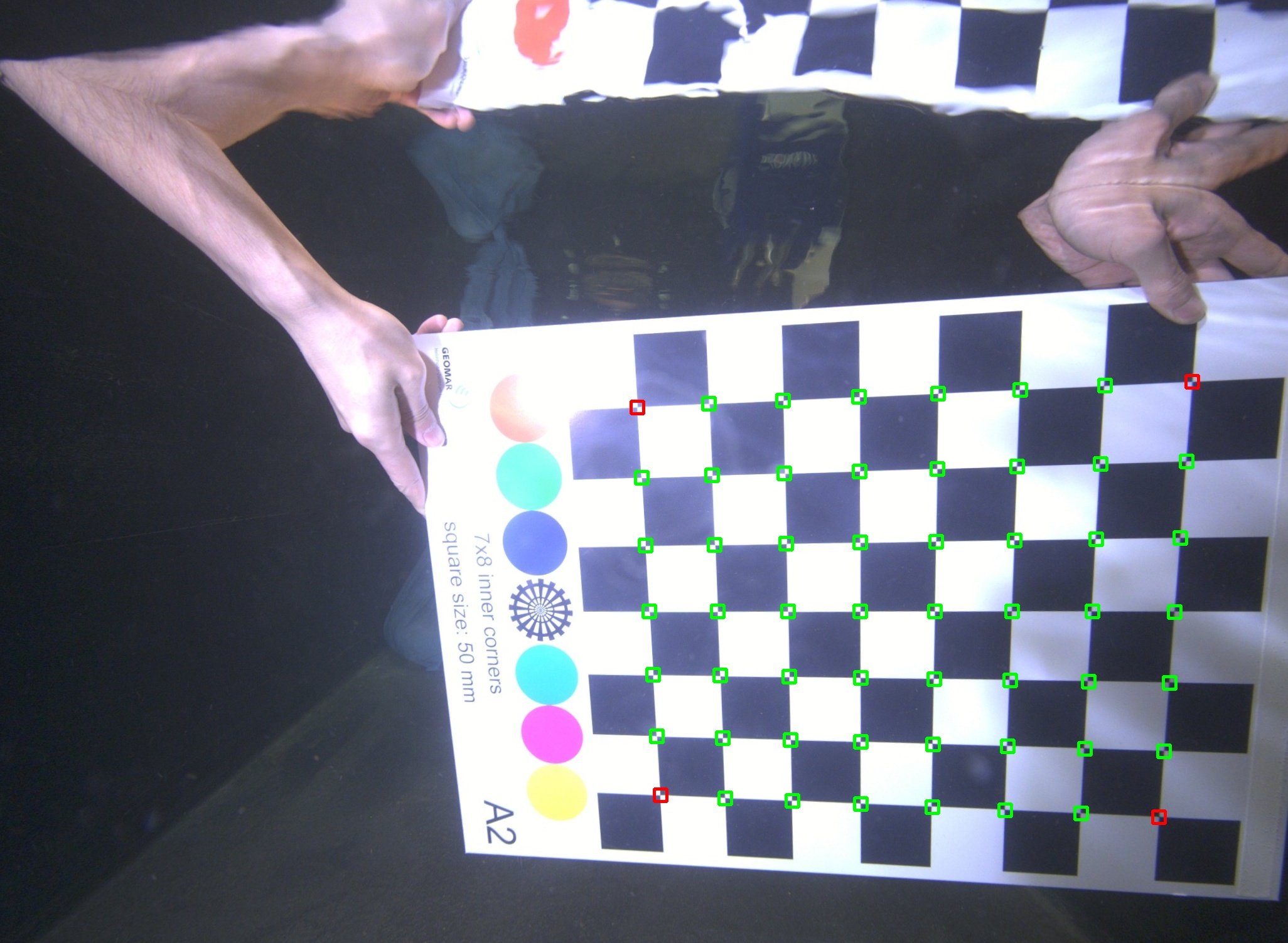}\vspace{0.18ex}
		\includegraphics[width=.99\textwidth]{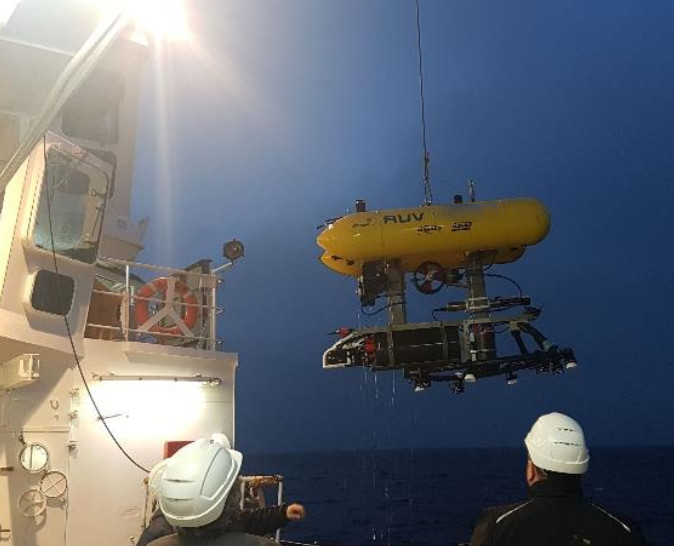}
	\end{center}
	\caption{Top: Camera detached from AUV, calibrated in a tank. Center: Sample images from refraction calibration. Bottom: AUV with camera shown after successfully completing visual mapping mission.\label{fig:antoncam}}
\end{figure}

\section{Conclusion}
In this paper, we have presented new insights into refraction effects caused by a decentered camera behind a spherical window. Somewhat similar to the flat refractive geometry case, the overall system acts as an axial camera, but here the axis intersects the sphere in two different poles that determine barrel- or pincushion-like refraction effects. Refraction happens along a line that connects the refraction center with the unrefracted projection, reducing the 2D search for projection to a 1D search. 

It was then shown how to directly estimate the refraction center from a single underwater chessboard image, and that it is possible to infer the decentering from underwater chessboard images only.
We then presented a novel practical approach to underwater calibration of dome port camera systems.

The approach can be used without a special pool with windows at the water level and does not need an in-air/underwater image pair, facilitating calibration also for bulky systems like submersibles.
The results obtained by our method are not only relevant for adjusting, but the remaining decentering offset can be considered when using the camera in practice, e.g. similar to refractive structure from motion with flat ports as proposed by Jordt et al. \cite{jordt_2016_refractive}.
In future works, we strive to integrate decentered dome port cameras into the state-of-art structure from motion pipeline to enable 3D reconstruction using dome port cameras, to facilitate a more detailed 3D error analysis, and also compare uncertainties and performance of different camera models.

\section*{Acknowledgment}
The authors would like to thank Dr. Amit Agrawal for offering generous help when deriving the analytical forward projection for the thin dome model. 
The authors also would like to thank Mareike Kampmeier for providing the AUV photos.
This publication has been funded by the German Research Foundation (Deutsche Forschungsgemeinschaft, DFG) Projektnummer 396311425, through the Emmy Noether Programme. 
Last but not least, the authors are also grateful for support from the Chinese Scholarship Council (CSC) for Mengkun She and Yifan Song.

\section*{References}

\bibliography{reference.bib}

\end{document}